\documentclass[letterpaper]{article} %
\usepackage{aaai23}  %
\usepackage{times}  %
\usepackage{helvet}  %
\usepackage{courier}  %
\usepackage[hyphens]{url}  %
\usepackage{graphicx} %
\usepackage{natbib}  %
\usepackage{caption} %
\usepackage{algorithm}
\usepackage{algorithmic}

\usepackage{newfloat}
\usepackage{listings}
\DeclareCaptionStyle{ruled}{labelfont=normalfont,labelsep=colon,strut=off} %
\floatstyle{ruled}
\newfloat{listing}{tb}{lst}{}
\floatname{listing}{Listing}
\usepackage{mathtools}   %
\usepackage{amsthm}
\usepackage{rotating}
\usepackage{chngcntr}   %
\usepackage{enumitem,nicefrac}
\usepackage{stmaryrd}		%
\usepackage[ngerman,english]{babel}     %
\usepackage{color}
\usepackage{cases}	%
\usepackage[titletoc,title]{appendix}

\usepackage{layouts}

\usepackage{amssymb} %
\usepackage{tikz}
\usetikzlibrary{shapes.geometric}   %
\usetikzlibrary{arrows.meta}
\usepackage{microtype}  %
\usepackage{tabularx}
\usepackage{wrapfig}
\usepackage{bm}

\usepackage{graphicx}
\usepackage{subcaption}
\usepackage[utf8]{inputenc}

\usepackage{multicol}
\usepackage{multirow}

\usepackage{caption}
\usepackage[textsize=tiny]{todonotes}

\newcommand{\R}{{\mathbb{R}}}

\newcommand{\N}{{\mathbb{N}}}

\renewcommand{\P}{{\mathbb{P}}}

\newcommand{\Ra}{\, \Rightarrow\,}

\newcommand{\eps}{{\varepsilon}}

\def\XXint#1#2#3{{\setbox0=\hbox{$#1{#2#3}{\int}$}
\vcenter{\hbox{$#2#3$}}\kern-.5\wd0}}

\newcommand{\id}{1\hspace{-0,9ex}1}

\newtheorem{Def}{Definition}[section]

\newtheorem{Le}[Def]{Lemma}
\newtheorem{Thm}[Def]{Theorem}

\counterwithout{figure}{section}
\counterwithout{figure}{section}
\title{AC-Band: A Combinatorial Bandit-Based Approach to Algorithm Configuration}
\author{
    Jasmin Brandt \textsuperscript{\rm 1},
    Elias Schede \textsuperscript{\rm 2},
    Viktor Bengs \textsuperscript{\rm 3},
    Bj\"orn Haddenhorst \textsuperscript{\rm 1},\\
    Eyke H\"ullermeier \textsuperscript{\rm 3,4},
    Kevin Tierney \textsuperscript{\rm 2}
}
\affiliations{
    \textsuperscript{\rm 1}Department of Computer Science, Paderborn University, Germany\\
    \textsuperscript{\rm 2}Decision and Operation Technologies Group, Bielefeld University, Germany\\
    \textsuperscript{\rm 3}Institute of Informatics, LMU Munich, Germany\\
    \textsuperscript{\rm 4}Munich Center for Machine Learning (MCML), Germany\\

}

\usepackage{bibentry}

\begin{document}

\maketitle

\begin{abstract}
    We study the algorithm configuration (AC) problem, in which one seeks to find an optimal parameter configuration of a given target algorithm in an automated way.
	Recently, there has been significant progress in designing AC approaches that satisfy strong theoretical guarantees.
	However, a significant gap still remains between the practical performance of these approaches and state-of-the-art heuristic methods.
	To this end, we introduce AC-Band, a general approach for the AC problem based on multi-armed bandits that provides theoretical guarantees while exhibiting strong practical performance.
    We show that AC-Band requires significantly less computation time than other AC approaches providing theoretical guarantees while still yielding high-quality configurations.
\end{abstract}

\section{Introduction} \label{sec:introduction}
Algorithm configuration (AC) is concerned with the task of automated search for a high-quality configuration (e.g., in terms of solution quality or runtime) of a given parameterized target algorithm. 
Parameterized algorithms are found in many different applications, including, for example, optimization (e.g., satisfiability (SAT)~\cite{audemard2018glucose} or mixed-integer programming (MIP) solvers \cite{cplexv20}), simulation, and machine learning.
Finding good configurations is a significant challenge for algorithm designers, as well as for users of such algorithms who may want to adjust the algorithm to perform well on data specific to their use case.
Finding good configurations through trial and error by hand is a daunting task, hence automated AC methods have been developed on the basis of heuristics, such as ParamILS \cite{Hutter2007AutoAC,Hutter2009ParamILS}, GGA \cite{Ansotegui2009GGA,Anstegui2015ModelBasedGA}, irace \cite{Birattari2002irace,lopez2016irace} or SMAC \cite{Hutter2013BayesOpt,Hutter2011SMAC}.

While heuristic configurators have had great success at finding good configurations on a wide variety of target algorithms, they do not come with theoretical guarantees. To this end, the pioneering work of \citet{kleinberg2017efficiency} proposes the first algorithm configurator with provable, theoretical guarantees on the (near-)optimality of the configuration returned. Further improvements and adjustments to those guarantees have followed~\cite{weisz2018leapsandbounds,weisz2019capsandruns,kleinberg2019procrastinating,weisz2020impatientcapsandruns}.
All of these works essentially consider the runtime as the performance objective and provide, in particular, an upper bound on the total runtime of the respective algorithm configurator that is (nearly) optimal in a worst-case sense.

Despite their appealing theoretical properties and the steady progress on their empirical performance, these approaches still cannot compete with heuristic approaches in terms of practical performance.
The main issue of these theoretical approaches is that they are conservative in the process of discarding configurations from the pool of potential candidates, as pointed out in recent work~\cite{weisz2020impatientcapsandruns}.
This is indeed a key characteristic difference compared to the heuristic approaches, which discard configurations quite quickly after being used only on a couple of problem instances.
From a practical point of view, this makes sense, as the risk of discarding all good configurations is generally lower than wasting lots of time looking at bad configurations.

In an attempt to further bridge the gap between heuristic configurators and theoretical approaches, we propose AC-Band, a general algorithm configurator inspired by the popular Hyperband \cite{Li2016Hyperband} approach based on multi-armed bandits \cite{lattimore2020bandit}. 
Hyperband is an algorithm for the hyperparameter optimization problem (HPO) \cite{yang2020hyperparameter,bischl2021hyperparameter}, which 
is essentially a subproblem of the general AC problem focusing on configuring solution quality of algorithms, with a particular focus on machine learning methods.
While using HPO approaches for AC looks attractive at first, it is rather uncommon in practice due to the subtle differences between the two problems.
These differences include potentially using runtime as a configuration metric and the existence of multiple \emph{problem instances}, which are different settings or scenarios that an optimization method must solve.

Our suggested approach reconciles the basic idea behind the mechanism of Hyperband with the key characteristics of the AC problem and incorporates the advantageous property of discarding configurations quickly.
This is achieved by first replacing the underlying bandit algorithm of Hyperband, Successive Halving (SH) \cite{karnin2013almost}, by a more general variant, Combinatorial Successive Elimination (CSE) \cite{brandt2022finding}, and then carefully adapting the parameters of the iterative CSE calls over time.
In contrast to SH, %
as well as to all other multi-armed bandit algorithms for pure exploration with finite budget, CSE allows (i) to steer the aggressiveness of discarding arms (configurations in our terminology) from the set of consideration, (ii) to pull more than one arm simultaneously (run multiple configurations in parallel), and (iii) to work with observations either of quantitative or qualitative nature.
As mentioned above, the first property seems to be of major importance in AC problems, but the other two properties will turn out to be particularly helpful as well.
Indeed, (ii) obviously allows parallelization of the search process, while the generality regarding the nature of the observations in (iii) transfers quite naturally to the suggested method.%

The interplay of the second and third properties allows us to instantiate AC-Band to obtain appealing practical performance compared to existing theoretical approaches regarding the total runtime.
On the theoretical side, we derive (under mild assumptions on the underlying AC problem) that AC-Band is guaranteed to return a nearly optimal configuration with high probability if used on sufficiently many problem instances.
Our theoretical result is quite general in the sense that the notion of optimality is not restricted to the runtime of the configurations, but also encompasses other target metrics such as solution quality or memory usage.

\section{Related Work}
\textbf{Theoretical advances in AC.} The field of AC has grown to include many different methods and settings; we refer to \citet{ACSurvey} for a full overview, especially with regard to the heuristic methods previously mentioned. 
Inspired by \citet{kleinberg2017efficiency}, who introduced \textit{Structured Procrastination} together with a non-trivial worst-case runtime bound, more and more algorithms with even better theoretical guarantees with respect to the runtime have been proposed. \textit{LeapsAndBounds} \cite{weisz2018leapsandbounds} tries to guess an upper bound on the optimal runtime by doubling the last failed guess, whereas \textit{Structured Procrastination with confidence} \cite{kleinberg2019procrastinating} works by delaying solving hard problem instances until later. Rather, it first runs the configurations with the smallest lower confidence bound of its mean runtime on instances that are easy to solve. Another recent method, \textit{CapsAndRuns} \cite{weisz2019capsandruns}, first estimates a timeout for each configuration and afterwards performs a Bernstein race over the configurations. As a follow up method, \textit{ImpatientCapsAndRuns} \cite{weisz2020impatientcapsandruns} uses a more aggressive elimination strategy by filtering unlikely optimal configurations in a preprocessing. 
Further theoretical progress has been made regarding the analysis of the estimation error in AC settings \cite{Balcan2019,balcan2020refined}, the distribution of the computational budget \cite{liu2020performance} and the understanding of heuristic methods \cite{Hall2019Cutoff,Hall2020}.

\textbf{HPO.} As a subset of AC, HPO involves setting the \textit{hyperparameters} of algorithms, in particular machine learning approaches. The term hyperparameter differentiates parameters that change the behavior of the algorithm being configured from parameters that are induced or learned from data and are thus not set by a user.
In contrast, AC focuses on configuring algorithms that solve instances of a dataset independently.
We refer to \citet{Bischl2021} for a full overview of HPO.

\textbf{Bandit methods for AC.} 
Classically, methods for multi-armed bandits (MAB) \cite{Lai1985AllocationRules,Bubeck2012RegretAnalysis,lattimore2020bandit} are designed to find a good balance between exploration-exploitation of specific choice alternatives (e.g., configurations or hyperparameters).
The pure exploration setting, however, has attracted much research interest as well \cite{bubeck2009pure,karnin2013almost,aziz2018pure}, especially for HPO \cite{Jamieson2015SuccHalv,Li2016Hyperband}.
However, up to now bandit algorithms making single choice alternative decisions have been leveraged, although the parallel execution of configurations (or hyperparameters) in the AC (or HPO) setting seems to be predetermined for the combinatorial bandit variant
\cite{CesaBianchi2012CMABs,Chen2013CMABs,Jourdan2021CMABs}. %
In light of this, the recently proposed CSE algorithm \cite{brandt2022finding} seems promising as a generalization of the popular SH approach.

	\section{Problem Formulation} \label{sec_problem_formulation}
	We adopt the formulation of the problem as by \citet{ACSurvey}.
	Let $\mathcal{I}$ be the space of problem instances and $\mathcal{P}$ an unknown  probability distribution over $\mathcal{I}$. 
	Suppose $\mathcal{A}$ is an algorithm that can be run on any problem instance $i\in \mathcal{I},$ and has   
	different parameters $p_j$ coming from a known domain $\Theta_j$ for each $j\in \{1,\dots, m\}.$ 
	We call $\mathcal{A}$  the \emph{target algorithm } and the Cartesian product of its parameter domains $\Theta = \Theta_1 \times \dots \times \Theta_m$ the \emph{configuration} or \emph{search space} consisting of all feasible parameter \emph{configurations}. 
	For a configuration $\theta \in \Theta$, we denote by $\mathcal{A}_{\theta}$ an instantiation of the target algorithm $\mathcal{A}$ with configuration $\theta$. 
	Running the target algorithm $\mathcal{A}$ with configuration $\theta$ on a specific problem instance $i \in \mathcal{I}$ results in costs specified by an unknown, and possibly stochastic, function $c: \mathcal{I} \times \Theta \rightarrow \mathbb{R},$ i.e., $c(i,\theta)$ represents the costs of using $\mathcal{A}_{\theta}$ for problem instance $i$.
	Here, the costs can correspond to the runtime of $\mathcal{A}_{\theta}$  for $i$, but also to other relevant target metrics such as the solution quality or the memory usage.
	\paragraph{Algorithm Configurator.}
	The goal in algorithm configuration is, roughly speaking, to find a configuration that is optimal, or at least nearly optimal, with respect to the costs in a certain sense, which we specify below.
	The search for such configurations is achieved by designing an algorithm configurator $\mathcal{AC}$ that (i) selects specific configurations in $\Theta$ and (ii) runs them on some (perhaps randomly) chosen  problem instances in $\mathcal{I}$.
	To this end, the algorithm configurator uses a statistic $s:\bigcup_{t\in \N} \R^t \to \R$ that maps the observed cost of a configuration $\theta$ used for a set of problem instances $i_1,\ldots,i_t$ to a representative numerical value $s(c(i_1,\theta), \dots, c(i_t,\theta)),$ that guides the search behavior of $\mathcal{AC}.$
	For example, $s$ could be the arithmetic mean, i.e., $s(c(i_1,\theta), \dots, c(i_t,\theta)) = t^{-1} \sum_{s=1}^t c(i_s,\theta)$.
	
	In this work, we are interested in algorithm configurators that can run several different configurations, up to a certain size $k$, in parallel on a selected problem instance.
	The algorithm configurator is given a fixed computational budget $B$, which represents the maximum number of such parallel runs and is set externally. 
	For this purpose, let $\Theta_{[2,k]} = \{\tilde{\Theta} \subset \Theta ~|~ 2 \leq |\tilde{\Theta}| \leq k\}$ be the set of all possible subsets of parameter configurations that have at least size $2$ and at most size $k$. Furthermore, let $\Theta_{[2,k]}(\theta) = \{\tilde{\Theta} \in \Theta_{[2,k]} ~|~ \theta \in \tilde{\Theta}\}$ be the set of all possible subsets of parameter configurations containing the configuration $\theta \in \Theta.$ 
	Note, that the observed cost $c(i,\theta)$ for running $\theta$ along with other configurations on an instance $i\in\mathcal{I}$ could depend on the respectively chosen configuration set $\tilde{\Theta} \in \Theta_{[2,k]}(\theta).$
	For example, the algorithm configurator could stop the parallel solution process as soon as one of the configurations provides a solution, and set a default cost (penalty) for the configurations that did not complete.
	Hence, we write $c_{\tilde{\Theta}}$ for the cost function in the following to take this contingency into account.
	Finally, we introduce $s_{\theta|\tilde{\Theta}}(t) = s(c_{\tilde{\Theta}}(i_{1},\theta), \dots, c_{\tilde{\Theta}}(i_{t},\theta))$ which is the statistic of $\theta$ after running it in parallel with the configurations in $\tilde{\Theta}\backslash\{\theta\}$ on $t$ problem instances $i_1,\ldots,i_t.$ 
	
	\paragraph{$\epsilon$-optimal Configurations.}
    Since the observed costs are potentially dependent on the chosen set of configurations to evaluate, we first introduce the following assumption on the limit behavior of the statistics.
	\begin{align*}
		(A1)~:~ \forall \tilde{\Theta} \in \Theta_{[2,k]} \ \forall \theta \in \tilde{\Theta} ~:~ S_{\theta|\tilde{\Theta}} = \lim_{t \rightarrow \infty} s_{\theta|\tilde{\Theta}}(t) \text{ exists.}
	\end{align*}
	In words, for each possible set of configurations the (possibly dependent) statistic of each configuration involved converges to some limit value if run on infinitely many problem instances.
	Recalling the example of $s$ being the arithmetic mean, this is arguably a mild assumption and implicitly assumed by most approaches for AC problems, due to the considered i.i.d. setting. Since our assumption is more general, it would also allow considering non-stationary scenarios of AC.

	The natural notion of an optimal configuration $\theta^{\ast}$ is a configuration that has the  largest (configuration-set dependent) limit value over all configurations.
	Indeed, if we would replace $\Theta_{[2,k]}$ by the singleton sets of all configurations in $\Theta,$ this would correspond to the commonly used definition of the optimal configuration (see \ \cite{ACSurvey}), as the limit value would be then $\mathbb{E}[c(i,\theta)]$\footnote{The expectation is w.r.t.\ $\mathcal{P}$ and the possible randomness due to $\mathcal{A}_{\theta}$ and/or the cost generation.}.  
	However, in our case this notion of optimality has two decisive drawbacks, as first of all such a $\theta^{\ast}$ may not exist. Moreover, even if it exists, the search for it might be hopeless as the configuration space is infinite (or very large).
	The latter issue arises in the ``usual'' AC problem scenario as well, and is resolved by relaxing the objective to finding a ``good enough'' configuration.
	We adapt this notion of near optimality by resorting to the definition of an $\epsilon$-best arm from the preference-based bandit literature \cite{bengs2021preference}. For some fixed relaxation parameter $\varepsilon>0,$ we call a configuration $\theta$ an  \emph{$\epsilon$-best configuration} iff
	\begin{align} \label{def_epsilon_best}
		\forall \tilde{\Theta} \in \Theta_{[2,k]}(\theta)~: ~ S_{\theta|\tilde{\Theta}} \geq S_{(1)|\tilde{\Theta}} - \epsilon,
	\end{align}
    where $S_{(1)|\tilde{\Theta}} \geq\ldots \geq S_{(|\tilde{\Theta}|)|\tilde{\Theta}}$ is the ordering of $\{S_{\theta|\tilde{\Theta}}\}_{\theta\in \tilde{\Theta}}.$
	
	\noindent
	Although we have relaxed the notion of optimality, finding $\epsilon$-best configurations is still often like searching for a needle in a haystack.
	Hence, we need to ensure that there is a sufficiently high probability that an $\epsilon$-best configuration is included in a large random sample set of $\Theta$: 
	\begin{align*}
			(A2):\mbox{the proportion of $\epsilon$-best configurations is $\alpha\in(0,1).$}
	\end{align*}
	Note this assumption is once again inspired by the bandit literature dealing with infinitely many arms \cite{NEURIPS2021_bd33f02c}.
	By fixing the probability for the non-occurrence of an $\epsilon$-best configuration to some $\delta \in (0,1),$ Assumption $(A2)$ ensures that a uniformly at random sampled set of configurations with size  $N_{\alpha, \delta} = \lceil \log_{1-\alpha}(\delta)\rceil$ contains at least one $\epsilon$-best configuration with probability at least $1-\delta$.
	
	Of course, an efficient algorithm configurator $\mathcal{AC}$ that aims to find an $\epsilon$-best configuration $\theta^{\ast}$ cannot verify the condition $S_{\theta^{\ast}|\tilde{\Theta}} \geq S_{(1)|\tilde{\Theta}} - \epsilon$ for every possible query set $\tilde{\Theta} \in \Theta_{[2,k]}(\theta^{\ast})$, in particular when the number of configurations, and thus the cardinality of $\Theta_{[2,k]}(\theta^{\ast})$, is infinite. Instead, $\mathcal{AC}$ can only guarantee the above condition for a finite number of query sets and therefore it will always find a proxy for an $\epsilon$-best configuration that does not have to be a true $\epsilon$-best configuration. To guarantee that $\mathcal{AC}$ finds a true $\epsilon$-best configuration with high probability, we introduce the following assumption.
	\begin{align*}
		&(A3):\forall M \in \N, \forall \theta\in \Theta : \\
		&\P_{} \Big( (\forall i\in [M]: S_{\theta|\tilde{\Theta}_i} \geq S_{(1)|\tilde{\Theta}_i} - \epsilon)  \\
		&\ \Ra  (\forall \tilde{\Theta} \in \Theta_{[2,k]}(\theta) : S_{\theta|\tilde{\Theta}} \geq S_{(1)|\tilde{\Theta}}-\epsilon )  \Big) \geq 1- \psi(M),
    \end{align*}
	where $\tilde{\Theta}_1, \dots ,\tilde{\Theta}_M \sim \mathrm{Uniform}(\Theta_{[2,k]}(\theta))$ and $\psi:\N \to [0,1]$ is a strictly monotone decreasing function.
	In words, the probability that a configuration $\theta$ is a global $\epsilon$-best configuration increases with the number of (randomly chosen) configuration sets on which it is a local $\epsilon$-best configuration, i.e., the characteristic condition in \eqref{def_epsilon_best} is fulfilled.
	
	Note that $(A3)$ is a high-level assumption on the difficulty of the underlying AC problem. The ``easier'' the problem, the steeper the form of $\psi$ and vice versa. As we do not impose further assumptions on $\psi$ other than monotonicity, our theoretical results below are valid for a variety of AC problems.
	\section{AC-Band Algorithm}
    
    The AC-Band algorithm consists of iterative calls to CSE~\cite{brandt2022finding} that allow it to successively reduce the size of a candidate set of configurations. We first describe how CSE works, then elaborate on its use in the AC-Band approach.
	\paragraph{Combinatorial Successive Elimination.}
	CSE (Algorithm \ref{alg:Framework}) is a combinatorial bandit algorithm that, given a finite set of arms (configurations), finds an optimal arm\footnote{Its existence is simply assumed.}  defined similarly to \eqref{def_epsilon_best} for $\epsilon=0$ using only a limited amount of feedback observations (budget). 
	To this end, CSE proceeds on a round-by-round basis, in each of which  (i) the non-eliminated arms are partitioned into groups of a given size $k$ (line $3$)  if possible (lines 4--8, 18--19) and (ii) feedback for each group is queried under a round-specific budget, which, once exhausted, leads to the elimination of a certain fraction of arms in each group (Algorithm \ref{alg:ae} called in lines 12 and 19). 	
	Here, the feedback observed for an arm is mapped to a numerical value using a statistic $s$ that indicates the (group-dependent) utility of the arm, and is used as the basis for elimination in each round.
	The fraction of eliminated arms is steered via a function $f_{\rho}:[k] \rightarrow [k]$ with $f_{\rho}(x) = \lfloor \nicefrac{x}{2^\rho} \rfloor,$ where the predetermined parameter $\rho \in (0, \log_2(k)]$ controls the aggressiveness of the elimination. A large value of $\rho$ corresponds to a very aggressive elimination, retaining only the arm(s) with (the) highest statistics, while a small $\rho$ eliminates only the arm(s) with the lowest statistics.
	The overall available budget is first split equally for each round, and then for all partitions in each round (line 2).

			\begin{algorithm}[H]
				\caption{Combinatorial Successive Elimination (CSE)} \label{alg:Framework}
				\textbf{Input:} set of configurations $\tilde{\Theta}$ with $|\tilde{\Theta}| = n$, subset size $k\le n$, budget $B,$ 
				$\rho \in (0, \log_2(k)]$, problem instances $I$ with $|I|=B$ which can be partitioned into \\ 
				$R^{\rho,k,n} = \min_x\{ g^{\circ x}(n)\leq k\} + \min_x\{f_{\rho}^{\circ x}(k)\leq 1 \}$ 
				problem instances $I_r$ with $|I_r| = P_{r}^{\rho,k,n} \cdot b_r$ where $\left\{P^{\rho,k,n}_{r}\right\}_r = \{ \lfloor \nicefrac{n}{k} \left(\nicefrac{f_{\rho}(k)}{k} \right)^{r-1} \rfloor \}_r,$  and $g(x) = f_{\rho}\left(k\right) \cdot \left\lfloor\nicefrac{x}{k} \right\rfloor + x\mod k$ \\
				\textbf{Initialization:}
				$ \tilde{\Theta}_1 \leftarrow \tilde{\Theta}$,
				$ r \leftarrow 1 $ \\
				\vspace*{-.3cm}
				\begin{algorithmic}[1]
					\WHILE {$|\tilde{\Theta}_r| \ge k$}
					\STATE    $b_{r} \leftarrow \lfloor \nicefrac{B}{(P_{r}^{\rho,k,n} \cdot R^{\rho,k,n})}\rfloor$ , $J \leftarrow P_{r}^{\rho,k,n}$
					\STATE $\tilde{\Theta}_{r,1}, \tilde{\Theta}_{r,2},\dots,\tilde{\Theta}_{r,J} \leftarrow \mathrm{Partition}(\tilde{\Theta}_r, k)$
					\IF {$|\tilde{\Theta}_{r,J}| < k $}
					\STATE $\mathcal{R} \leftarrow \tilde{\Theta}_{r,J}$, $J \leftarrow J-1$
					\ELSE 
					\STATE $\mathcal{R} \leftarrow \emptyset$
					\ENDIF
					\STATE $I_{r,1},\dots,I_{r,J} \leftarrow \mathrm{Partition}(I_r, b_r)$
					\STATE $\tilde{\Theta}_{r+1} \leftarrow \mathcal{R}$
					\FOR {$j \in [J]$}
					\STATE $\mathcal{R} \leftarrow \mathrm{ArmElimination}(\tilde{\Theta}_{r,j}, b_r, f_{\rho}(|\tilde{\Theta}_{r,j}|),I_{r,j})$
					\STATE $ \tilde{\Theta}_{r+1} \leftarrow \tilde{\Theta}_{r+1} \cup \mathcal{R}$
					\ENDFOR
					\STATE $r \leftarrow r+1$
					\ENDWHILE
					\STATE $\tilde{\Theta}_{r+1} \leftarrow \emptyset$
					\WHILE {$|\tilde{\Theta}_r| > 1$}
					\STATE $\tilde{\Theta}_{r+1} \leftarrow $ $\mathrm{ArmElimination}(\tilde{\Theta}_{r},b_r,f_{\rho}(|\tilde{\Theta}_{r}|), I_{r})$
					\STATE $r \leftarrow r+1$
					\ENDWHILE
				\end{algorithmic} 
				\hspace*{3pt} \textbf{Output:} The remaining item in $\tilde{\Theta}_r$
			\end{algorithm}
			\vspace*{-0.1cm}

			\vspace*{-0.5cm}
			\begin{algorithm}[H]
				\caption{ArmElimination$(\tilde{\Theta}',b,l,I')$} 
				\label{alg:ae}
				\begin{algorithmic}[1]
					\STATE Use $\tilde{\Theta}'$ for $b$ times on problem instances $I'$
					\STATE For all $\theta \in \tilde{\Theta}'$, update $s_{\theta|\tilde{\Theta}'}(b)$
					\STATE Choose an ordering $\theta_{1},\dots,\theta_{|\tilde{\Theta}'|}$ of $(s_{\theta|\tilde{\Theta}'}(b))_{\theta\in \tilde{\Theta}'}$
					\STATE \textbf{Output: } $\{\theta_{1},\dots,\theta_{l}\}$
				\end{algorithmic} 
			\end{algorithm}

	In light of the AC problem we are facing, \emph{querying feedback for a group of arms} corresponds to running a subset of configurations in parallel on a problem instance, which results in observations in the form of costs.
	Moreover, we do not reuse a single problem instance for any parallel run so that the budget is in fact equal to the number of problem instances used.
	Accordingly, the overall budget of CSE corresponds to the number of problem instances used in total, which are split into disjoint problem instance sets of size $b_r,$ i.e., the round-specific budget (line 9).
	
	Since CSE initially assumes a finite set of arms and a fixed parameter $\rho$ guiding the overall elimination aggressiveness, we face two trade-offs.
	The first is regarding the interplay between the number of initial arms and the round-wise budget:
	\begin{itemize}[noitemsep,topsep=2pt,leftmargin=8.5mm]
		\item [(T1):] If the initial number of configurations for CSE is small (large), the more (fewer) runs can be carried out on different instances, leading to potentially more reliable (unreliable) statistics, but only on a few (many) subsets of configurations.
	\end{itemize}
	The second trade-off arises through the interplay between the round-wise budget and the elimination aggressiveness: 
	\begin{itemize}[noitemsep,topsep=2pt,leftmargin=8.5mm]
		\item [(T2):] If the elimination behavior of CSE is aggressive (conservative), then more (fewer) runs can be carried out on different instances, leading to potentially more reliable (unreliable) statistics, but only on a few (many) subsets of configurations.
	\end{itemize}
	The challenge now is to reconcile these two trade-offs and, above all, to take into account the specifics of AC problems.

			\begin{algorithm}[H]
				\caption{AC-Band} \label{alg:ACband}
				\textbf{Input:} target algorithm $\mathcal{A}$, configuration space $\Theta$, problem instance space $\mathcal{I}$, Budget $B$, subset size $k$, suboptimality  $\epsilon$ of "good enough" configuration, proportion of $\epsilon$-best configurations $\alpha$, failure probability $\delta$, $n_0 > \lceil \nicefrac{\ln(\delta)}{\ln(1-\alpha)} \rceil \in \mathbb{N}$\\
				\textbf{Initialization:} $E \leftarrow \left\lceil \log_2\left(\nicefrac{n_0}{n_0 - N_{\alpha, \delta}} \right) \right\rceil,$ $q \leftarrow 1+\nicefrac{k-1}{E}$\\
				$C_1 \leftarrow \log_{q}\left(2\right)$, 
				$C_2 \leftarrow 1 + \log_{q}\left(n_0 +  \frac{4 n_0}{n_0-N_{\alpha,\delta}} \right)$, \\
				$C_3 \leftarrow \lceil \log_{q}(k) \rceil$
				\begin{algorithmic}[1]
					\STATE sample $\theta_0 \in \Theta$
					\FOR { $e \in [E]$}
					\STATE $n_e = \left\lceil \nicefrac{n_0}{2^{e}} \right\rceil + 1$, $\rho_{e} =  \log_2\left(\nicefrac{e+k-1}{e}\right)$, 
					\STATE sample  $\theta_{e,1}, \dots, \theta_{e,{n_e-1}} \in \Theta$
					\STATE sample  $I_e \subset \mathcal{I} \backslash \cup_{e'=1}^{e-1} I_{e'}$ with $|I_e| = \nicefrac{B}{c_e},$ \\ 
					where $c_e = \frac{(C_1 E - (2^{E}-1)(2C_1-C_2-C_3))2^{e}}{2^{E}(-eC_1+C_2+C_3)}$
					\STATE $\theta_e = \mathrm{CSE}( \{\theta_{e-1}, \theta_{e,1}, \dots, \theta_{e,{n_e-1}}\}, k,$ $|I_e|,\rho_{e},I_e)$
					\ENDFOR
				\end{algorithmic}
				\hspace*{3pt} \textbf{Output:} $\theta_{E}$
			\end{algorithm}

	    \begin{figure}
\centering
\scalebox{0.6}{
\begin{tikzpicture}
		\draw[blue, dotted, thick , rounded corners] (0.4,1.5) rectangle ++(4,4.75);
		\node[blue] at (2.4,1.75) {$e=1$};
		\draw[blue,  dotted, thick, rounded corners] (4.5,1.5) rectangle ++(4,4.75);
		\node[blue] at (6.5,1.75) {$e=2$};
		\node[rectangle, draw, rounded corners] (A) at (0,4) {$\theta_{0}$};
		\node[rectangle, draw, rounded corners] (B) at (2,4) {$\begin{matrix} \theta_{0} \\ \theta_{1,1} \\ \theta_{1,2} \\ \vdots \\ \vdots \\ \theta_{1,n_{1}-1} \end{matrix}$};
		\node[rectangle, draw, rounded corners] (C) at (4,4) {$\theta_{1}$};		
		\node[rectangle, draw, rounded corners] (D) at (6,4) {$\begin{matrix} \theta_{1} \\ \theta_{2,1} \\  \vdots \\ \theta_{2,n_{2}-1} \end{matrix}$};
		\node[rectangle, draw, rounded corners] (E) at (8,4) {$\theta_{2}$};	
		\node (F) at (10,4) {{\Large $\cdots$}};	
		\node[ellipse, draw, black,minimum width = 2cm, minimum height = 1.2cm ] (X) at (10,1) {};
		\node[black] at (10,1) {$\Theta$};
		\node[ellipse, draw, black,minimum width = 2cm, minimum height = 1.2cm ] (Y) at (10,7) {};
		\node[black] at (10,7) {$\mathcal{I}$};
		\draw[->,thick] (A) -- (B);
		\draw[->,very thick] (B) -- node[below] {{\small CSE}} (C);
		\draw[->,thick] (C) -- (D);
		\draw[->,thick, very thick] (D) -- node[below] {{\small CSE}} (E);
		\draw[->,thick] (E) -- (F);
		\draw[-,thick, dashed] (Y) -- node[above] {{\small sample new problem instances $I_{e}$ }} (3.25,7);
		\draw[-,thick, dashed] (X) -- node[below] {{\small sample $n_{e}-1$ configurations}} (0.75,1);
		\draw[->,thick,dashed] (0.75,1) -- (0.75,4);
		\draw[->,thick,dashed] (4.75,1) -- (4.75,4);
		\draw[->,thick,dashed] (3.25,7) -- (3.25,4);
		\draw[->,thick,dashed] (7.25,7) -- (7.25,4);
\end{tikzpicture} 
}
\caption{Illustration of AC-Band's solving process.}
\label{fig_AC_Band}
\end{figure}
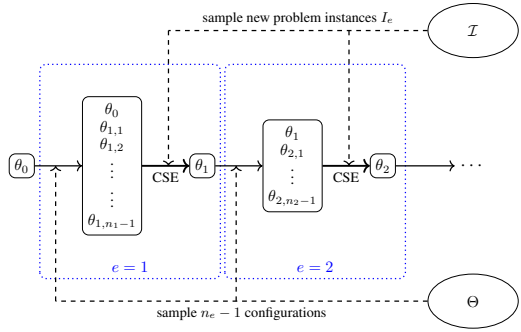

	\paragraph{AC-Band.}
	The design of AC-Band (Algorithm \ref{alg:ACband}) seeks to find a good balance for both tradeoffs (T1) and (T2) by calling CSE iteratively.
	Initially, CSE is invoked with larger sets of configurations, and an aggressive elimination strategy is applied. 
	Over time, the size of the candidate sets is successively reduced, and the aggressiveness of the elimination strategy is also gradually decreased. 
	Roughly speaking, the idea is to have high exploration in the beginning, and thus more risk that good configurations are discarded, and become more and more careful towards the end.
	
	More specifically, AC-Band proceeds in epochs $e\in\{1,\ldots,E\}$, in each of which CSE is called on a specific set of problem instances using a specified degree of elimination aggressiveness $\rho_e$ and a set of configurations of size $n_e,$ with both $\rho_e$ and $n_e$ decreasing w.r.t.\ $e$ (line 3). 
	At the end of an epoch, i.e., when CSE terminates, a single high quality configuration among the considered set of configurations is returned (line 6).
	The set of configurations used in an epoch consists of the high quality configuration of the previous epoch, which is sampled randomly for the first epoch (line 1), and $n_e-1$ randomly sampled configurations, which have not been considered before (line 4). 

	This epoch-wise procedure of AC-Band is depicted in Figure~\ref{fig_AC_Band}.
		Although AC-Band is similar in design to Hyperband in that it tries to find a good balance between specific trade-offs by successively invoking a bandit algorithm (CSE vs.\ SH), AC-Band differs in the way it defines the quality of the search objects\footnote{Hyperband uses hyperparameters of machine learning models as search objects and AC-Band algorithm configurations.}.  
	Unlike Hyperband, we do not run configurations individually on problem instances and consider the cost of each configuration on its own, rather configurations are run in parallel and we consider potential interactions. 
	Accordingly, we do not have one global quality value that we can compare for all configurations seen, but several at once.

	The overall number of considered problem instances is $B$ (the evaluation budget), which is a parameter that we analyze below.
	Besides the ``usual'' parameters of an algorithm configurator, i.e., the target algorithm $\mathcal{A},$ its configuration space $\Theta$ and the problem instance space $\mathcal{I},$ AC-Band requires:
	\begin{itemize}[noitemsep,topsep=0pt,leftmargin=3.5mm]
		\item the maximum number of configurations that can be run in parallel $k,$
		\item some relevant summary statistic $s:\bigcup_{t\in \N} \R^t \to \R$ for the observed costs (see Section \ref{sec_problem_formulation}),
		\item the theoretically motivated guarantee parameters $\epsilon>0,$  $\alpha \in (0,1),$ and $\delta \in(0,1)$ (see Section \ref{sec_problem_formulation})
        \item a reference size  $n_0$ for the set of epoch-wise sampled configurations (must be  $\geq \lceil \frac{\ln(\delta)}{\ln(1-\alpha)} \rceil$ for technical reasons).
	\end{itemize}
	With these parameters specified, AC-Band determines the overall number of epochs $E$ and the sufficient number of problem instances $B/c_e$ (line 5) for an epoch-wise CSE to return a high quality configuration (see Section \ref{sec_theoretical_guarantuess}).
	Moreover, the overall number of considered configurations is guaranteed to be at least $N_{\alpha, \delta} = \lceil \frac{\ln(\delta)}{\ln(1-\alpha)} \rceil,$ which in light of the random sampling of the configurations ensures that at least one $\epsilon$-best configuration is sampled with a probability of at least $1-\delta$.

\section{Theoretical Guarantees} \label{sec_theoretical_guarantuess}
In Appendix \ref{sec:ACBand}, we prove the following theoretical guarantee for AC-Band regarding the sufficient evaluation budget (or number of problem instances) to find an $\epsilon$-best configuration with high probability w.r.t.\ $\mathcal{P}$ as well as the randomness invoked by AC-Band.
For the proof, we need to extend the theoretical guarantees for CSE to the setting of finding $\epsilon$-best configurations (see Appendix \ref{sec:CSE}).
\begin{Thm} \label{Thm:ACBandSufficientBudget}
    Let $R^e$ be the number of rounds of CSE in epoch $e\in\{1,\ldots,E\}$, let $C_1$, $C_2$ and $C_3$ be as defined in Alg.\ \ref{alg:ACband} and further let $\mathbb{A}_{r}(\theta^{\ast})$\footnote{$\mathbb{A}_{r}(\theta) = \emptyset$ if $\theta$ is not contained anymore in the set of active configurations in round $r$.} be the partition in round $r$ of CSE containing $\theta^{\ast}$ and
    \begin{align*}
        &\bar{\gamma}^{-1} = \max_{e \in [E],r\in [R^{e}]} \big(1+ \bar{\gamma}_{\mathbb{A}_{r}(\theta^{\ast})}^{-1}\big( \\
        &\qquad \qquad \qquad \max\big\{ \frac{\epsilon}{2},\max_{r \in [R^{e}]}\frac{\Delta_{(f_{\rho}(|\mathbb{A}_{r}(\theta^{\ast})|)+1)|\mathbb{A}_{r}(\theta^{\ast})}}{2}\big\}\big) \big), \\
        &\Delta_{(i)|\tilde{\Theta}}=S_{(1)|\tilde{\Theta}}-S_{(i)|\tilde{\Theta}}, ~~~\\
        &\bar{\gamma}^{-1}_{\mathbb{A}_{r}(\theta^{\ast})}(t)= \min_{\theta \in \mathbb{A}_{r}(\theta^{\ast})} \inf\limits_{t' \in \mathbb{N} }\{ |s_{\theta|\mathbb{A}_{r}(\theta^{\ast})}(t') - S_{\theta | \mathbb{A}_{r}(\theta^{\ast})} | \leq t \}.  
    \end{align*}
    Under Assumptions $(A1)$--$(A3)$, Algorithm \ref{alg:ACband} used with a subset size $k$, an $\epsilon$-best configuration proportion of $\alpha$, and a failure probability of $\delta$ finds an $\epsilon$-best configuration $\theta^{\ast}$ with probability at least $\min\{ 1-\delta,  1 - \psi((R^{E})^{-1}) \}$
    if 
    \begin{align*}
       B \geq \bar{\gamma}^{-1} \cdot \frac{n_0}{k} \cdot \frac{C_1 E - (2^{E}-1)(2C_1-C_2-C_3)}{2^{E}}.
    \end{align*}
\end{Thm}
Roughly speaking, AC-Band finds a near-optimal configuration with a probability depending on the allowed failure probability $\delta$ and the probability that a locally optimal configuration is also a globally optimal one (see Assumption (A3)) if the budget is large enough. The sufficient budget, in turn, is essentially driven by $\bar{\gamma}^{-1},$ which depends on the difficulty of the underlying AC problem by two characteristics: the maximal inverse convergence speed $\bar{\gamma}^{-1}_{\mathbb{A}_r(\theta^*)}$ of the used statistic $s,$ and the maximal (halved) suboptimality gap $\Delta$ of the limits of the statistic between the best configuration and the best one that will be discarded from the query set $\theta^*$ is contained.
The remaining terms of the sufficient budget can be computed explicitly once the theoretical guarantee parameters $\alpha$ and $\delta$, as well as the subset size $k$, are fixed. The sufficient budget in dependence of the mentioned parameters is discussed and plotted in Appendix \ref{sec:SuffBudgetDiscussion}.

\color{black}

Note that the theoretical guarantee in Theorem \ref{Thm:ACBandSufficientBudget} is not directly comparable to the ones by the theoretical AC approaches  \cite{kleinberg2019procrastinating,weisz2018leapsandbounds,weisz2019capsandruns,weisz2020impatientcapsandruns}.
This is due to the major differences of our approach and the later ones on how we approach the AC problem.
Indeed, we do not restrict ourselves to runtime as the target metric (or the costs), and we also take possible dependencies in the parallel runs into account.
As a consequence, the notion of near optimality of a configuration in the other works is tailored more towards runtimes in an absolute sense, i.e., without considering interaction effects, while ours is more general and in particular focusing on such interaction effects.
Thus, the theoretical guarantee in Theorem \ref{Thm:ACBandSufficientBudget} in the form of a sufficient evaluation budget to obtain a nearly optimal configuration is sensible, as we do not commit to a specific target metric.

\section{Experiments}\label{sec:experiments}
We examine AC-Band on several standard datasets for evaluating theoretical approaches for AC. We note that while these datasets refer exclusively to runtimes, AC-Band is applicable to other target metrics (see Section \ref{sec_problem_formulation}).
In our experiments, we address the following two research questions: 
\textbf{Q1}: Is AC-Band able to find high quality configurations in less CPU time than state-of-the-art AC methods with guarantees? \textbf{Q2}: How does AC-Band scale with $k$?

\textbf{Datasets.} We use one SAT and two MIP datasets that have been used before to assess theoretical works on AC~\cite{weisz2020impatientcapsandruns}.
Due to space constraints, we only consider one of the MIP datasets here, while the appendix also discusses the other. The SAT dataset contains precomputed runtimes of configurations of the MiniSat SAT solver~\cite{een2003extensible} obtained by solving instances that are generated using CNFuzzDD %
\cite{weisz2018leapsandbounds}. The dataset contains runtimes for $948$ configurations on $20118$ instances. The MIP datasets curated by \citet{weisz2020impatientcapsandruns} are generated using an Empirical Performance Model (EPM)\cite{hutter2014algorithm}. In particular, the EPM is trained on the CPLEX solver~\cite{cplexv20} separately using instances from a combinatorial auction (Regions200 \cite{leyton2000towards}) and wildlife conservation (RCW \cite{ahmadizadeh2010empirical}) dataset. The resulting model is used to predict runtimes for $2000$ configurations and $50000$ and $35640$ new instances, respectively. Since all runtimes are precomputed (a timeout of $900$ seconds is used), the evaluation of configuration-instance pairs only required table look-ups in our experiments.

\begin{figure*}[ht!]

    \centering
    \includegraphics[width=0.9\textwidth]{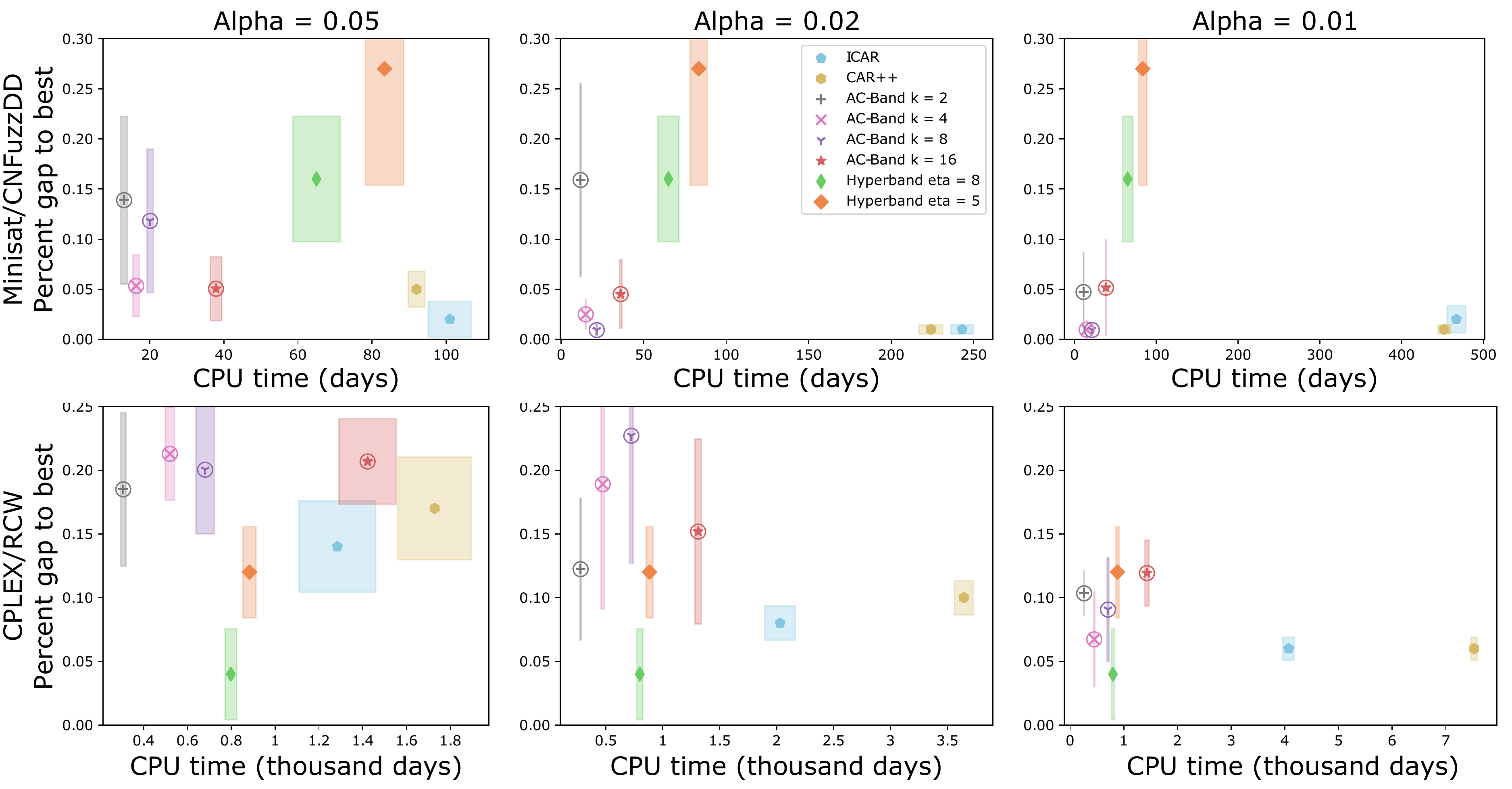}
  
  \caption{Mean CPU time and percent gap to best over $5$ seeds for $\delta=0.05$ and different $\alpha$ (columns) for AC-Band, ICAR, CAR++ and Hyperband on CNFuzzDD (top) and RCW (bottom). Circles indicate variants of AC-Band. Rectangles represent the standard error over the seeds. The number of configurations tried for CAR++: $\{97, 245, 492\}$, ICAR: $\{134, 351, 724\},$ AC-Band: $\{60, 153, 303\}$, Hyperband($\eta=5$): $\{842\},$ Hyperband($\eta=8$): $\{ 618\}$.} \label{ACBand results d05}
\end{figure*}

\textbf{Evaluation.} To compare methods, we consider two metrics: (i) the accumulated CPU time needed by a method to find a configuration, and (ii) the percent gap to the best configuration. This second metric measures in percent how much more time the configuration returned by the method needs to solve all instances compared to the best overall configuration for the dataset. 
Smaller values indicate that the configuration found is closer to the best configuration in terms of runtime and the best configuration has a value of $0$.
This allows for comparing the quality between methods, as well as to determine how ``far'' a configuration is from the optimal one. In practical applications of AC, wall-clock time is often a bottleneck, and speeding up the process of finding a suitable configuration is the main focus. For these speedups, practitioners are (usually) willing to sacrifice configuration quality to a certain extent.
The other theoretical works 
use the $R^{\delta}$ metric (note that $\delta$ has a different meaning in this work) to evaluate the quality of a returned configuration. This metric is a variation of the mean runtime, where the mean runtime of a configuration is only computed over the $(1-\delta)$ portion of instances with the lowest runtime. In real-world settings, we do not have the luxury of ignoring a part of the instance set, thus we do not view this metric as suitable for evaluating our approach.
For the sake of completeness, we nevertheless report the $R^{\delta}$  values in Appendix \ref{sec_further_results}.

    \textbf{Initialization of AC-Band.} Due to the generality of our approach, a summary statistic $s$ that measures the quality of a configuration needs to be determined. In our case, the $k$ configurations in a subset of CSE can be evaluated in parallel for an instance given that $k$ CPUs are available. When running $k$ configurations in parallel, time can be saved by stopping all remaining configuration runs as soon as the first configuration finishes on the given instance. Through this capping, we obtain right-censored feedback where a runtime is only observed for the ``finisher''. A statistic that is able to deal with this censored feedback is needed to avoid using an imputation technique that could potentially add bias to the feedback. In line with \citet{brandt2022finding}, we interpret the obtained feedback as a preference:  the finishing configuration is preferred over the remaining, unfinished configurations. Once we have obtained these preferences for multiple instances, we can use a preference-based statistic such as the relative frequency to establish an ordering over the configurations in $k$. In particular, we count how many times a configuration finishes first over a set of instances\footnote{Code: https://github.com/DOTBielefeld/ACBand}.

\textbf{Competitors.} AC-Band is compared against ICAR,  CAR++~\cite{weisz2020impatientcapsandruns} and Hyperband~\cite{Li2016Hyperband}. At the moment, ICAR is the best performing AC method that comes with theoretical guarantees. We use the implementation provided by \citet{weisz2020impatientcapsandruns} with the same hyperparameter settings. Since AC-Band is inspired by Hyperband, we also adapt Hyperband for the AC setting for a comparison. Specifically, we set the parameter $R$ of Hyperband such that it uses the same total budget (number of instances) as AC-Band. 
In addition, we use the average runtime over instances as the validation loss and consequently return the configuration with smallest average runtime.
Finally, we set $s_{max} = \lceil (\log_{\eta}(n_{max}))\rceil$, adjust the  calculation of $r_i$ slightly to account for instances that were already seen, and try different values of $\eta.$ 

\textbf{Choice of $\delta.$}
Varying $\delta$ can lead to significantly different performance of AC-Band and other techniques.
Due to space constraints, we only show results for two datasets for $\delta = 0.05$ since this setting has also been used in previous work~\cite{weisz2020impatientcapsandruns}. We note, however, that other settings are just as valid, and therefore also provide results for  $\delta = 0.01$ and additional datasets in Appendix \ref{sec_further_results}. In fact, using $\delta = 0.01$ can result in finding better configurations, albeit it is up to the user of the AC approach to decide which $\delta$ best suits their needs.

\paragraph{Q1}
Figure \ref{ACBand results d05} shows the CPU time used by each method and the percent gap to best configuration. With a small value of $k$, AC-Band lies on the Pareto front of percent gap versus CPU time for both datasets (for the third, Regions200, as well). This allows us to answer \textbf{Q1} in the affirmative. In particular, with $k=2$ AC-Band is $72\%$ percent faster than ICAR and $73\%$ faster than Hyperband for $\delta = 0.05$ over all $\alpha$ and datasets, while providing configurations that are only $7\%$ and $6\%$ worse in terms of the gap to the best configuration. For most real-world applications, this is an acceptable trade-off in time versus quality. For all datasets, the percent gap to best decreases when AC-Band samples more configurations (smaller $\alpha$). This increase in exploration does not lead to a significant increase in CPU time for a fixed $k$, since AC-Band still has the same budget, i.e., additional configurations are evaluated on fewer instances. 

Hyperband samples more configurations in total than AC-Band, which helps it to achieve a better percent gap to best on datasets where the majority of configurations have a high mean runtime. On these datasets, only a few good configurations are present. This is the case for the RCW dataset (and the Regions200 dataset in the appendix) where only a few instances are needed to determine that a configuration should be discarded. On datasets where the runtime of configurations is more evenly distributed, such as the CNFuzzDD dataset, using too few instances may lead to discarding promising configurations early, giving AC-Band an edge by evaluating less configurations more thorough. Lastly, since Hyperband does not use  capping, its CPU time deteriorates.

\paragraph{Q2}
Our experiments clearly indicate that lower values of $k$ are preferable. With $k=2$, more CSE rounds are performed and thus the number of configurations decreases slower than with a higher $k$. With a larger $k$, 
the opposite occurs, and significant amounts of CPU time are expended with little information gain.
However, note that higher $k$'s have a lower wall-clock time, so a user would receive answers sooner.

\section{Conclusion}
In this paper we introduced AC-Band, a versatile approach for AC problems that comes with theoretical guarantees even for a range of target metrics in addition to runtime.
We showed that AC-Band returns a nearly optimal configuration w.r.t.\ the target metric with high probability if used with a sufficient number of problem instances and the underlying AC problem satisfies some mild assumptions.
In our experimental study, we considered an instantiation of AC-Band based on preference feedback, which generally leads to faster average CPU times than other theoretical approaches, while still returning a suitable configuration in the end. 

Our results open up several possibilities for future work.
First, it would be interesting to analyze AC-Band specifically for the case in which runtime is the relevant target metric and investigate whether a similar worst-case overall runtime guarantee can be derived as for the other theoretical approaches in this vein.
Next, a theoretical as well as empirical analysis regarding the interplay between the explicit instantiation of AC-Band w.r.t.\ the underlying statistic $s$ and the characteristics of the underlying AC would be desirable. 
In other words, on what  types of AC problems does a specific instantiation of AC perform well or poorly? 
Also, our mild Assumption (A1) would even allow some leeway in the configuration or problem instance sampling strategy of the algorithm configurator, which is currently simply uniformly at random for AC-Band.
Finally, real-world AC applications generally have only a handful of instances available, thus it would be advantageous to have strong theoretical guarantees even for scenarios without thousands of instances.

\section*{Acknowledgements}

This research was supported by the research training group Dataninja (Trustworthy AI for Seamless Problem Solving: Next Generation Intelligence Joins Robust Data Analysis) funded by the German federal state of North Rhine-Westphalia and by the German Research Foundation (DFG) within the project ``Online Preference Learning with Bandit Algorithms'' (project no. 317046553).

\bibliography{Camera-ready}

\clearpage

\appendix
\onecolumn
\section{List of Symbols}
The following table contains a list of symbols that are frequently used in the main paper as well as in the following supplementary material. \\ \medskip
\small
\begin{tabularx}{\columnwidth}{lX}
    \hline
	\multicolumn{2}{c}{\textbf{Basics}} \\
	\hline
	$\id\{ \cdot  \}$ & Indicator function \\
	$\N$  & Set of natural numbers (without 0), i.e., $\N = \{1,2,3,\dots\}$ \\
	$\R$ & Set of real numbers\\
    $\mathcal{I}$ & Space of problem instances\\
    $\mathcal{P}$ & Probability distribution over instance space $\mathcal{I}$\\
    $\mathcal{A}$ & Target algorithm that can be run on problem instance $i\in\mathcal{I}$\\
    $\Theta$ & Configuration search space consisting all feasible parameter configurations\\
    $\mathcal{A}_{\theta}$ & Instantiation of $\mathcal{A}$ with configuration $\theta$\\
    $c(i,\theta)$ & $c:\mathcal{I} \times \Theta \rightarrow \mathbb{R}$ costs of using $\mathcal{A}_{\theta}$ on problem instance $i\in \mathcal{I}$\\
    $\mathcal{AC}$ & Algorithm configurator  \\
    $k$ & Maximal possible subset size \\
    $B$ & Budget for the learner (algorithm configurator) \\ 
    $\Theta_{[2,k]}$ & All subsets of $\Theta$ of size at least 2 and at most $k$: $\{\tilde{\Theta} \subseteq \Theta ~|~ 2 \leq |\tilde{\Theta}| \leq k \}$\\
    $\Theta_{[2,k]}(\theta)$ & All subsets in $\Theta_{[2,k]}$ which contain configuration $\theta$: $\{\tilde{\Theta} \in \Theta_{[2,k]} ~|~ \theta \in \tilde{\Theta}\}$\\
    $c_{\tilde{\Theta}}$ & Cost of running $\theta$ on problem instance $i\in\mathcal{I}$ in parallel with configurations $\theta' \in \tilde{\Theta} \in \Theta_{[2,k]} \backslash\{\theta\} $\\
    \hline
	\multicolumn{2}{c}{\textbf{Modelling related}} \\
	\hline
	$s$ & Statistic mapping costs to a numerical value: $s: \cup_{t\in \mathbb{N}} \mathbb{R}^t \rightarrow \mathbb{R}$, $\left(c(i_1, \theta) , \ldots , c(i_t, \theta) \right) \mapsto s(c(i_1, \theta), \dots, c(i_t,\theta))$ \\
    $s_{\theta|\tilde{\Theta}}(t)$ & $s_{\theta|\tilde{\Theta}}(t) = s(c_{\tilde{\Theta}}(i_1, \theta), \dots, c_{\tilde{\Theta}}(i_t, \theta))$ statistic of $\theta \in \Theta$ after running $\tilde{\Theta}$ in parallel\\
    $S_{\theta|\tilde{\Theta}}$ & Limit of the statistics for configuration $\theta$ in query set $\tilde{\Theta},$ i.e., $\lim_{t\rightarrow \infty} s_{\theta|\tilde{\Theta}}(t)$ (see Assumption (A1))\\
    $\epsilon$ &  Near-optimality parameter, $\epsilon \in (0,1)$\\
    $\alpha$ & Proportion of $\epsilon$-best configurations in $\Theta$, $\alpha \in (0,1)$ (see Assumption (A2)) \\
    $\delta$ & Fixed error probability for identifying an $\epsilon$-best configuration, $\delta \in (0,1)$\\
    $N_{\alpha,\delta}$ & $N_{\alpha, \delta} = \lceil \log_{1-\alpha}(\delta)\rceil$ number of configurations that have to be sampled to ensure that at least one $\epsilon$-best configuration is contained with probability at least $1-\delta$\\
    $\gamma_{\theta|\tilde{\Theta}}$ &  
    Pointwise minimal function such that $|s_{\theta|\tilde{\Theta}}(t) - S_{\theta|\tilde{\Theta}}| \leq \gamma_{\theta|\tilde{\Theta}}(t)$ for all $t$
    \\ 
    $\gamma_{\theta|\tilde{\Theta}}^{-1}(t)$ & $\min\{t' \in \mathbb{N} ~|~ |s_{\theta|\tilde{\Theta}}(t') - S_{\theta|\tilde{\Theta}}| \leq t \}$ (exists due to Assumption (A1))\\
    $\bar{\gamma}_{\tilde{\Theta}}^{-1}(t)$ & Minimal $\gamma_{\theta|\tilde{\Theta}}^{-1}(t) = \min\{t' \in \mathbb{N} ~|~ |s_{\theta|\tilde{\Theta}}(t') - S_{\theta|\tilde{\Theta}}| \leq t \}$ over all $\theta\in \tilde{\Theta}$ (exists due to Assumption (A1))\\
    $S_{(l)|\tilde{\Theta}}$,  $\Delta_{(l)|\tilde{\Theta}}$   & $l$-th order statistic of $\{S_{\theta|\tilde{\Theta}}\}_{\theta\in \tilde{\Theta}}$ for $l\in\{1,2,\ldots,|\tilde{\Theta}|\}$ and its gap $\Delta_{(l)|\tilde{\Theta}} = S_{\theta^{\ast}|\tilde{\Theta}} - S_{(l)|\tilde{\Theta}}$ \\
    \hline
	\multicolumn{2}{c}{\textbf{Algorithm related}} \\
	\hline
	\textsc{CSE} & The generic \emph{combinatorial successive elimination} algorithm (Algorithm \ref{alg:Framework}) \\
    $f_{\rho}$ & Function from $[k]$ to $[k]$, $f_{\rho}(x) = \lfloor \nicefrac{x}{2^{\rho}} \rfloor$ for a $\rho \in (0, \log_2(k)]$ specifying the nature of the configuration elimination strategy \\
    $n_0$ & Variable to specify the initial sample size, $n_0 \in (N_{\alpha,\delta}, 2N_{\alpha,\delta}]$\\
    $n_e$ & $n_e = \lceil \nicefrac{n_0}{2^{e}} \rceil + 1$ number of considered configurations in epoch $e \in [E]$\\
    $E$ & Number of epochs $E = \lceil \log_2(\nicefrac{n_0}{n_0-N_{\alpha, \delta}}) \rceil$\\
    $C_1$ & Internal constant variable for AC-Band, $C_1 = \log_{1 + \frac{k-1}{E}}\left(2\right)$\\
    $C_2$ & Internal constant variable for AC-Band, $C_2 = 1 + \log_{1 + \frac{k-1}{E}}\left(n_0 + 4 \frac{n_0}{n_0-N_{\alpha,\delta}} \right)$\\
    $C_3$ & Internal constant variable for AC-Band, $C_3 = \left\lceil \log_{1+\frac{k-1}{E}}(k) \right\rceil$ \\
    $c_e$ & Internal variable for AC-Band in epoch $e\in[E]$: $c_e = \frac{(C_1 E - (2^{E}-1)(2C_1-C_2-C_3))2^{e}}{2^{E}(-eC_1+C_2+C_3)}$\\
    $B_e$ & Budget for call of CSE in epoch $e\in[E]$: $B_e = \nicefrac{B}{c_e}$ for an overall budget of $B$ for AC-Band\\
    $R^{\rho_e,k,n_e}$ & Number of rounds in CSE in epoch $e \in [E]$ \\
    $P^{\rho_e,k,n_e}_{r}$ & Number of partitions in CSE in epoch $e \in [E]$ and round $r \in [R^{\rho_e,k,n_e}]$ \\
    $\mathbb{A}_{r}(\theta)$ & The partition in round $r$ of CSE containing $\theta$ (emptyset otherwise)\\
    $b_r$ & Budget used in round $r$ of CSE for a partition \\
\end{tabularx}
\normalsize
\clearpage

\newpage
\section{Extension of CSE for Finding $\epsilon$-best Arms} \label{sec:CSE}
\noindent
\subsection{Adjustments of Algorithm Parameters}
\noindent
We modify the definition of the function f in \cite{brandt2022finding} and thus define $f_{\rho}(x) = \lfloor \frac{x}{2^{\rho}} \rfloor$ for a $\rho \in (0, \log_2(k)]$ and obtain for
\begin{itemize}
    \item $\rho=1$ : Combinatorial Succesive Halving (CSH) with $f_{\rho}(k) = \lfloor \frac{k}{2} \rfloor$ 
    \item $\rho = \log_2(k)$ : Combinatorial Succesive Winner Stays (CSWS) with $f_{\rho}(k) = 1$
    \item $\rho \rightarrow 0$ : Combinatorial Succesive Rejects (CSR) with $f_{\rho}(k) = k-1$.
\end{itemize}
Note that for a fixed subset size $k$ and  for a fixed $\rho \in (0,\log_2(k)]$, one can derive the number of rounds and number of partitions per round as follows. The number of rounds in the first while loop of Algorithm \ref{alg:Framework} can be computed as
\begin{align*}
R^{\rho,k,n}_{1} = \{\min x ~:~ g^{(\circ x)}(n)\leq k\} ~\text{for}~ g(x) = f_{\rho}\left(k\right) \cdot \left\lfloor\frac{x}{k} \right\rfloor + x\mod k,
\end{align*}
where $g^{(\circ x)}$ denotes the $x$-times composition of $g.$ 
Furthermore, it holds that $R_1^{\rho,k,n} \leq \lceil \log_{\frac{k}{f_\rho(k)}}(n) \rceil$, which we use for the sake of ease to estimate $R_1^{\rho,k,n}$ in the following in our theoretical analyses.
In the second while-loop of Algorithm \ref{alg:Framework}, we have only $k$ arms left, thus we can calculate the number of rounds as
\begin{align*}
    R^{\rho,k}_{2} = \left\{\min x ~:~f_{\rho}^{(\circ x)}(k)\leq 1\right\}. 
\end{align*}
For all $k \in \mathbb{N}$, $R_2^{\rho, k} \leq \lceil \log_{\frac{k}{f_{\rho}(k)}}{(k)} \rceil$.
The overall number of rounds is therefore $R^{\rho,k,n} = R^{\rho,k,n}_{1} + R^{\rho,k}_{2} \leq \lceil \log_{\frac{k}{f_{\rho}(k)}}{(n)} \rceil + \lceil \log_{\frac{k}{f_{\rho}(k)}}{(k)} \rceil$. 
The number of partitions per round $r \in \{1, \dots, R^{f_{\rho},k,n}\}$  is given by
\begin{align*}
    P^{\rho,k,n}_{r}  =  \left\lfloor \frac{n}{k} \left(\frac{f_{\rho}(k)}{k} \right)^{r-1} \right\rfloor   .
\end{align*}
Thus, we can automatically compute the number of rounds and partitions in Algorithm \ref{alg:Framework} if the parameter $\rho$ for the discard function $f_{\rho}$, the subset size $k$, and the number of arms $n$ are given. In contrast to \cite{brandt2022finding}, we do not need to estimate and specify $R$ and $\{P_r\}$ by hand before we can run Algorithm \ref{alg:Framework}.
\subsection{Theoretical Guarantees}
The first step for extending Algorithm \ref{alg:Framework} to the context of AC is to extend the theoretical guarantees to the context of finding an $\epsilon$-best arm of a finite set of arms (configurations in our terminology) and to derive a sufficient budget for CSE that is necessary to find such an $\epsilon$-best arm, provided that an $\epsilon$-best arm is defined as follows.
\begin{Def}
    Assume we have $n$ arms (configurations) $\theta_1,\ldots,\theta_n.$ Let $\theta^{\ast}$ be such that
    \begin{align*}
        \forall \theta^{\ast} \in \Theta_{[2,k]}(\theta^{\ast}) ~ S_{\theta^{\ast}|\tilde{\Theta}} \geq S_{(1)|\tilde{\Theta}} - \epsilon.
    \end{align*}
    We call $\theta^{\ast}$ an $\epsilon$-best arm.
\end{Def}
\noindent
    Note that we only have a finite set of $n$ arms (configurations) $\theta_1,\ldots,\theta_n,$ which we identify simply by their indices $1,\ldots,n$ in the following.
    Further assume in the following that an $\epsilon$-best arm exists. We identify this arm by $i^{\ast},$ and write $\Delta_{i|\tilde{\Theta}} = S_{(1)|\tilde{\Theta}}-S_{i|\tilde{\Theta}}$. In addition, let $\mathbb{A}_r$ be the set of active arms in round $r \in [R^{\rho,k,n}]$, which will then be partitioned into $\mathbb{A}_{r,1}, \dots, \mathbb{A}_{r, P_r^{\rho,k,n}}$.
\begin{Thm}[Sufficient budget for CSE for finding an $\epsilon$-best arm] \label{Thm_generalSufficient_eps_max}
    Using Algorithm \ref{alg:Framework} with $n$ arms, a discard function $f_{\rho}$, and a subset size $k$ returns an arm $i^{\ast},$ which is an $\epsilon$-best arm if the budget $B$ is larger than 
	\begin{align*}
	    z&\left(\rho,k,n,\epsilon \right)
    	\coloneqq R^{\rho,k,n} \max_{r\in [R^{\rho,k,n}]} P^{\rho,k,n}_{r} \cdot \max_{r\in [R^{\rho,k,n}]} \left(1+ \bar{\gamma}_{\mathbb{A}_{r}(i^{\ast})}^{-1}\left(\max\left\{ \frac{\epsilon}{2},\max_{r \in [R^{\rho,k,n}]}\frac{\Delta_{(f_{\rho}(|\mathbb{A}_{r}(i^{\ast})|)+1)|\mathbb{A}_{r}(i^{\ast})}}{2}\right\}\right) \right),	\end{align*}	
    	where $\mathbb{A}_{r}(i^{\ast})$ denotes the partition in round $r$ of CSE containing $i^{\ast}$ (or $\theta^{\ast}$).
\end{Thm}
\begin{proof}[Proof of Theorem \ref{Thm_generalSufficient_eps_max}]   

    \underline{Step 1:} Algorithm \ref{alg:Framework} never requires a number of problem instances that exceeds the budget $B$:
    \begin{align*}
        \sum_{r=1}^{R^{\rho,k,n}} P_{r}^{\rho,k,n} \cdot b_{r} 
        &= \sum_{r=1}^{R^{\rho,k,n}} P_{r}^{\rho,k,n} \cdot \left\lfloor \frac{B}{P_{r}^{\rho,k,n}\cdot R^{\rho,k,n}} \right\rfloor 
        \leq \sum_{r=1}^{R^{\rho,k,n}} \frac{B}{R^{\rho,k,n}} = B.
    \end{align*}
    \underline{Step 2:} Assume in the following $B \geq z(\rho,k,n,\epsilon)$, then we have for each round $r \in [R^{\rho,k,n}]$ that
    \begin{align*}
        b_r &\geq \frac{B}{P_{r}^{\rho,k,n} \cdot R^{\rho,k,n}} -1 \\
        &\geq \max_{r \in [R^{\rho,k,n}]} \left( 1+ \bar{\gamma}^{-1}_{\mathbb{A}_r(i^{\ast})}\left(\max\left\{ \frac{\epsilon}{2}, \max_{r\in [R^{\rho,k,n}]} \frac{\Delta_{(f_{\rho}(|\mathbb{A}_{r}(i^{\ast})|)+1)|\mathbb{A}_{r}(i^{\ast})}}{2}\right\}\right) \right) - 1 \\
        &= \max_{r \in [R^{\rho,k,n}]}\bar{\gamma}^{-1}_{\mathbb{A}_r(i^{\ast})}\left(\max\left\{ \frac{\epsilon}{2}, \max_{r\in [R^{\rho,k,n}]} \frac{\Delta_{(f_{\rho}(|\mathbb{A}_{r}(i^{\ast})|)+1)|\mathbb{A}_{r}(i^{\ast})}}{2}\right\}\right).
    \end{align*} 
    We can assume in the following w.l.o.g.\ $i^{\ast} = 1$ and $\mathbb{A}_{r}(1) = \mathbb{A}_{r1}$ by relabeling the arms (configurations) and query sets (sets of configurations). 
    Now, we first show, that $s_{1|\mathbb{A}_{r1}}(t) - s_{i|\mathbb{A}_{r1}}(t) \geq 0$ for all rounds $r\in [R^{\rho,k,n}]$, all arms $i \in \mathbb{A}_{r1}$ and all $t\geq \tau_{i} := \max_{r\in [R^{\rho,k,n}]} \bar{\gamma}^{-1}_{\mathbb{A}_{r1}}\left(\max_{r\in [R^{\rho,k,n}]} \frac{\Delta_{i|\mathbb{A}_{r1}}}{2} \right)$. 
    Define  $\gamma_{\theta|\tilde{\Theta}}$ as the 
    pointwise minimal function such that $|s_{\theta|\tilde{\Theta}}(t) - S_{\theta|\tilde{\Theta}}| \leq \gamma_{\theta|\tilde{\Theta}}(t)$ for all $t$ and $\gamma_{\tilde{\Theta}} = \max_{\theta \in \tilde{\Theta}} \gamma_{\theta|\tilde{\Theta}}.$
    Thus, $\tau_i \geq \bar{\gamma}^{-1}_{\mathbb{A}_{r1}}\left(\frac{\Delta_{i|\mathbb{A}_{r1}}}{2} \right)$ for all $r \in [R^{\rho,k,n}]$, and according to the definition of $\gamma_{\tilde{\Theta}}$ we have
    \begin{align*}
        |s_{i|\mathbb{A}_{r1}}(t) - S_{i|\mathbb{A}_{r1}}| 
        \leq \gamma_{\mathbb{A}_{r1}}(t) 
        \leq \frac{\Delta_{i|\mathbb{A}_{r1}}}{2} 
        ~~\text{ for } t\geq \tau_i.
    \end{align*}
    Thus, for all $t\geq \tau_i$,
    \begin{align*}
        s_{1|\mathbb{A}_{r1}}(t) - s_{i|\mathbb{A}_{r1}}(t) 
        &= s_{1|\mathbb{A}_{r1}}(t) - S_{1|\mathbb{A}_{r1}} + S_{1|\mathbb{A}_{r1}} - S_{i|\mathbb{A}_{r1}} + S_{i|\mathbb{A}_{r1}} - s_{i|\mathbb{A}_{r1}}(t)\\
        &= s_{1|\mathbb{A}_{r1}}(t) - S_{1|\mathbb{A}_{r1}} - (s_{i|\mathbb{A}_{r1}}(t) - S_{i|\mathbb{A}_{r1}}) + S_{1|\mathbb{A}_{r1}} - S_{i|\mathbb{A}_{r1}} \\
        &\geq -2\gamma_{\mathbb{A}_{r1}}(t) + S_{1|\mathbb{A}_{r1}} - S_{i|\mathbb{A}_{r1}} \\
        &\geq -2 \frac{S_{1|\mathbb{A}_{r1}} - S_{i|\mathbb{A}_{r1}}}{2} + S_{1|\mathbb{A}_{r1}} - S_{i|\mathbb{A}_{r1}}  = 0.
    \end{align*}
    In this scenario, arm $i$ will be eliminated before arm $1$, since the $f_{\rho}(|\mathbb{A}_{r}|)$ arms with the lowest statistic $s$ are discarded in each round $r \in [R^{\rho,k,n}]$.
    
    Now consider a round $r \in [R^{\rho,k,n}]$ and assume that by reordering the arms, w.l.o.g., $S_{1|\mathbb{A}_{r1}} \geq S_{2|\mathbb{A}_{r1}} \geq \dots \geq S_{|\mathbb{A}_{r1}||\mathbb{A}_{r1}}$. Since $S_{i|\mathbb{A}_{r1}}$ is non-increasing in $i$, the $\tau_i$ are non-increasing in $i$: $\tau_2 \geq \tau_3 \geq \dots$. Thus,
    \begin{equation} \label{eq:tau}
        t \geq \tau_i \Rightarrow s_{1|\mathbb{A}_{r1}}(t) \geq s_{i|\mathbb{A}_{r1}}(t).
    \end{equation}
    Assume now that $1 \in \mathbb{A}_{r} \wedge 1 \notin \mathbb{A}_{r+1}$
    \begin{align*}
        &\Rightarrow \sum_{i \in \mathbb{A}_{r1}} \id \{ s_{i|\mathbb{A}_{r1}}(b_r) > s_{1|\mathbb{A}_{r1}}(b_r) \} > f_{\rho}(|\mathbb{A}_{r1}|) \\
        &\Rightarrow \sum_{i \in \mathbb{A}_{r1}} \id \{ b_r < \tau_i \} > f_{\rho}(|\mathbb{A}_{r1}|) \\
        &\Rightarrow b_r < \tau_{f_{\rho}(|\mathbb{A}_{r1}|) + 1}.
    \end{align*}
    This implies
    \begin{equation} \label{eq:brbigenough}
        1 \in \mathbb{A}_{r} \wedge b_r \geq \tau_{f_{\rho}(|\mathbb{A}_{r1}|) + 1} \Rightarrow 1 \in \mathbb{A}_{r+1}.
    \end{equation}
    Recall that $b_r \geq \max_{r \in [R^{\rho,k,n}]}\bar{\gamma}^{-1}_{\mathbb{A}_r(i^{\ast})}\left(\max\left\{ \frac{\epsilon}{2}, \max_{r\in [R^{\rho,k,n}]} \frac{\Delta_{(f_{\rho}(|\mathbb{A}_{r}(i^{\ast})|)+1)|\mathbb{A}_{r}(i^{\ast})}}{2}\right\}\right)$ and $\tau_{f(|\mathbb{A}_{r1}|)+1} = \max_{r\in [R^{\rho,k,n}]} \bar{\gamma}^{-1}_{\mathbb{A}_{r1}}\left(\max_{r\in [R^{\rho,k,n}]} \frac{\Delta_{f_{\rho}(|\mathbb{A}_{r1}|)+1|\mathbb{A}_{r1}}}{2} \right)$.\\
    \begin{itemize} [leftmargin=1.3cm]
        \item[\underline{Case 1:}] $\max_{r\in [R^{\rho,k,n}]} \frac{\Delta_{f_{\rho}(|\mathbb{A}_{r1}|)+1|\mathbb{A}_{r1}}}{2} \geq \frac{\epsilon}{2}$ and $1 \in \mathbb{A}_r$.\\
        We have $b_r \geq \max_{r\in [R^{\rho,k,n}]} \bar{\gamma}^{-1}_{\mathbb{A}_{r1}}\left(\max_{r\in[R^{\rho,k,n}]} \frac{\Delta_{f_{\rho}(|\mathbb{A}_{r1}|)+1|\mathbb{A}_{r1}}}{2}\right) = \tau_{f_{\rho}(|\mathbb{A}_{r1}|) + 1}$. By  \eqref{eq:brbigenough} we have that $1\in \mathbb{A}_{r+1}$.
        \item[\underline{Case 2:}] $\max_{r\in [R^{\rho,k,n}]} \frac{\Delta_{f_{\rho}(|\mathbb{A}_{r1}|)+1|\mathbb{A}_{r1}}}{2} < \frac{\epsilon}{2}$ and $1 \in \mathbb{A}_r$.\\
        We have $b_r \geq \max_{r\in [R^{\rho,k,n}]} \bar{\gamma}^{-1}_{\mathbb{A}_{r1}}\left(\frac{\epsilon}{2}\right)$ and\\
        $\max_{r\in [R^{\rho,k,n}]} \bar{\gamma}^{-1}_{\mathbb{A}_{r1}}\left(\frac{\epsilon}{2}\right) < \max_{r\in [R^{\rho,k,n}]} \bar{\gamma}^{-1}_{\mathbb{A}_{r1}}\left(\max_{r\in[R^{\rho,k,n}]} \frac{\Delta_{f_{\rho}(|\mathbb{A}_{r1}|)+1|\mathbb{A}_{r1}}}{2}\right) = \tau_{f_{\rho}(|\mathbb{A}_{r1}|) +1}$.\\
        If $1 \in \mathbb{A}_{r+1}$, Algorithm \ref{alg:Framework} continues in round $r+1$ with case 1 or case 2. Now assume $1 \notin \mathbb{A}_{r+1}$ and let
        \begin{equation*}
            p \coloneqq \min \left\{i \in \mathbb{A}_{r1} ~:~ \max_{r\in[R^{\rho,k,n}]}\frac{S_{1|\mathbb{A}_{r1}} - S_{i|\mathbb{A}_{r1}}}{2} \geq \frac{\epsilon}{2} \right\}.
        \end{equation*}
        By the assumption of the case 2, we have $p > f_{\rho}(|\mathbb{A}_{r1}|) +1$ and 
        \begin{align*}
            b_r \geq \max_{r \in [R^{\rho,k,n}]} \bar{\gamma}^{-1}_{\mathbb{A}_{r1}}\left(\frac{\epsilon}{2}\right) 
            \geq \max_{r\in [R^{\rho,k,n}]} \bar{\gamma}^{-1}_{\mathbb{A}_{r1}}\left(\max_{r\in[R^{\rho,k,n}]} \frac{\Delta_{i|\mathbb{A}_{r1}}}{2}\right) 
            = \tau_{i} \text{ for all } i\geq p.
        \end{align*}
        Thus, by \eqref{eq:tau} we have $s_{1|\mathbb{A}_{r1}}(b_r) \geq s_{i|\mathbb{A}_{r1}}(b_r)$ for $i\geq p$, so all arms $i\geq p$ are eliminated before arm $1$. \\
        Moreover, we have 
        \begin{align*}
            \max_{i\in \mathbb{A}_{r+1}} S_{i|\mathbb{A}_{r1}} 
            \geq \max_{i<p} S_{i|\mathbb{A}_{r1}} 
            \geq S_{1|\mathbb{A}_{r1}} - \epsilon
        \end{align*}
        since 
        \begin{align*}
            \frac{S_{1|\mathbb{A}_{r1}} - S_{i|\mathbb{A}_{r1}}}{2}
            < \max_{r\in[R^{\rho,k,n}]}\frac{S_{1|\mathbb{A}_{r1}} - S_{i|\mathbb{A}_{r1}}}{2}
            < \frac{\epsilon}{2}.
        \end{align*}
        In addition, by definition of $p$ it holds even for all $r \in [R^{\rho,k,n}]$ that $S_{i|\mathbb{A}_{r1}} > S_{1|\mathbb{A}_{r1}} -\epsilon$.
        Thus, although CSE might discard $i^{\ast},$ it is ensured that an $\eps$-best arm is still active. Consequently, we relabel the arms such that the still active $\eps$-best arm is now denoted by $i^{\ast}.$ 
        \item[\underline{Case 3:}] $1\notin \mathbb{A}_{r}$.\\
        Since $1 \in \mathbb{A}_0$, there was a round $\tilde{r} < r$ such that $1\in\mathbb{A}_{\tilde{r}}$ and $1 \notin \mathbb{A}_{\tilde{r}+1}$. For this round $\tilde{r}$ only case 2 was possible, otherwise $1 \in \mathbb{A}_{\tilde{r}+1}$. Since case 2 was true and $1\notin \mathbb{A}_{\tilde{r}+1}$, we have 
        \begin{align*}
            \max_{i \in \mathbb{A}_{\tilde{r}+1}} S_{i|\mathbb{A}_{\tilde{r}}}
            \geq S_{1|\mathbb{A}_{\tilde{r}}} - \epsilon
        \end{align*}
        and also for all other rounds $r \in [R^{\rho,k,n}]$ that $S_{i|\mathbb{A}_{r1}} > S_{1|\mathbb{A}_{r1}} - \epsilon$.
    \end{itemize}
\end{proof}
\section{Guarantees for AC-Band} \label{sec:ACBand}
\subsection{Number of Epochs}
Let $N_{\alpha, \delta} = \lceil \frac{\ln(\delta)}{\ln(1-\alpha)}\rceil$ be the number of sampled configurations necessary to ensure that an $\epsilon$-best configuration is contained in the samples with probability at least $1-\delta$ provided the proportion of $\epsilon$-best configurations in the configuration space is $\alpha$ (see Assumption (A2)). In each epoch $e$, we consider in total $n_e = \lceil \frac{n_0}{2^{e}}\rceil+1$ configurations and keep the winner of the last epoch, such that we have to sample at least $\lceil \frac{n_0}{2^{e}} \rceil$ new configurations in epoch $e$. To be precise, we sample in one run of AC-Band,  the following number of configurations in total:
\begin{align*}
    \sum_{e=1}^{E} \left\lceil \frac{n_0}{2^{e}} \right\rceil &\geq  \sum_{e=1}^{E} \frac{n_0}{2^{e}} \\
    &= n_0  \sum_{e=1}^{E} \left( \frac{1}{2} \right)^{e}\\
    &= n_0  \left(\sum_{e=0}^{E} \left( \frac{1}{2} \right)^{e} - 1 \right)\\
    &= 2n_0\left(1- \frac{1}{2^{E+1}}\right) -n_0 \\
    &= n_0 - \frac{n_0}{2^{E}}. 
\end{align*}
This value must be greater than $N_{\alpha, \delta}$ to guarantee that we have an $\epsilon$-best configuration contained in the sampled configurations with probability at least $1-\delta$. Rearranging the above inequality leads to
\begin{align*}
    2^{E} &\geq \frac{n_0}{n_0-N_{\alpha, \delta}}\\
    \Rightarrow ~ E &\geq \log_2\left(\frac{n_0}{n_0-N_{\alpha,\delta}} \right),
\end{align*}
where we consider that $n_0 > N_{\alpha, \delta}$. Since we want our algorithm to finish as fast as possible, $E = \left\lceil \log_2\left(\frac{n_0}{n_0-N_{\alpha,\delta}}\right) \right\rceil$ is a suitable choice.
In addition, we need to guarantee that the number of epochs is well defined, thus we must ensure that
\begin{align*}
    E &\geq \log_2\left(\frac{n_0}{n_0-N_{\alpha,\delta}}\right) \overset{!}{\geq} 1 \\
    &\Leftrightarrow ~~ \frac{n_0}{n_0- N_{\alpha,\delta}} \geq 2 \\
    &\Leftrightarrow ~~ n_0 \leq 2 N_{\alpha, \delta}.
\end{align*}
Putting everything together, we get the condition that $n_0 \in \Big( N_{\alpha, \delta}, 2 N_{\alpha, \delta} \Big]$.
\subsection{Sufficient Budget}
\begin{Le} \label{Le_geom_series}
    For $N\in \mathbb{N}$ and any $a,b,c\in \R$ it holds that 
    \begin{align*}
        \sum_{i=1}^{N} \frac{-ia + b + c}{2^{i}} = \frac{aN-(2^{N}-1)(2a-b-c)}{2^{N}}.
    \end{align*}
\end{Le}
\begin{proof}
    \begin{align*}
        \sum_{i=1}^{N} \frac{-ia + b + c}{2^{i}} &= -a\cdot\sum_{i=1}^{N} \frac{i}{2^{i}} + (b+c)\cdot \sum_{i=1}^{N} \frac{1}{2^{i}} \\
        &= -a\cdot\sum_{i=0}^{N} \frac{i}{2^{i}} + (b+c)\cdot \left(\sum_{i=0}^{N} \frac{1}{2^{i}} -1 \right)\\ 
        &= -a \frac{\frac{N}{2^{N+2}} - \frac{N+1}{2^{N+1}}+\frac{1}{2}}{\left(\frac{1}{2}-1\right)^2} + (b+c)\left(\frac{1-\frac{1}{2^{N+1}}}{1-\frac{1}{2}} - 1\right) \\
        &= \frac{-aN}{2^N}+\frac{a(N+1)}{2^{N-1}}-2a + (b+c)\left(1-\frac{1}{2^N}\right)\\
        &= \frac{-aN+a(2N+2)-a2^{N+1}}{2^N}+\frac{(b+c)(2^N-1)}{2^N} \\
        &= \frac{aN-(2^N-1)(2a-b-c)}{2^N},
    \end{align*}
    where the closed-form sum formulas of the geometric series are used in the third equality. 
\end{proof}
\allowdisplaybreaks
\begin{proof}[Proof of Thm. \ref{Thm:ACBandSufficientBudget}]
If the total budget $B$ for AC-Band is such that the epoch-wise budget for CSE is at least $z(\rho_e, k, n_e, \epsilon)$ for each epoch $e$ (see Theorem \ref{Thm_generalSufficient_eps_max}), AC-Band will return a configuration that is locally $\epsilon$-best. 
With the help of Assumption (A3), we can then infer the claim.

Thus, the sufficient budget to guarantee that AC-Band finds a local $\epsilon$-best configuration $i^\ast$ can be computed as
\begin{align*}
    &\sum_{e=1}^{E} z(\rho_e, k, n_e, \epsilon)\\
    &= \sum_{e=1}^{E} R^{\rho_e,k,n_e} \max_{r\in [R^{\rho_e,k,n_e}]} P^{\rho_e,k,n_e}_{r} \cdot \max_{r\in [R^{\rho_e,k,n_e}]} \left(1+ \bar{\gamma}_{\mathbb{A}_{r}(i^{\ast})}^{-1}\left(\max\left\{ \frac{\epsilon}{2},\max_{r \in [R^{\rho_e,k,n_e}]}\frac{\Delta_{(f_{\rho}(|\mathbb{A}_{r}(i^{\ast})|)+1)|\mathbb{A}_{r}(i^{\ast})}}{2}\right\}\right) \right) \\
    &= \sum_{e=1}^{E} (R_1^{\rho_e,k,n_e} +  R_2^{\rho_e,k}) \left\lfloor \frac{n_e}{k} \right\rfloor \cdot \max_{r\in [R^{\rho_e,k,n_e}]} \left(1+ \bar{\gamma}_{\mathbb{A}_{r}(i^{\ast})}^{-1}\left(\max\left\{ \frac{\epsilon}{2},\max_{r \in [R^{\rho_e,k,n_e}]}\frac{\Delta_{(f_{\rho}(|\mathbb{A}_{r}(i^{\ast})|)+1)|\mathbb{A}_{r}(i^{\ast})}}{2}\right\}\right) \right) \\
    &=: (*)
\end{align*}
Note that for $\rho_e = \log_2 \left(\frac{e+k-1}{e} \right)$, we have $f_{\rho_e}(x) = \left\lfloor \frac{x}{2^\rho} \right\rfloor = \left\lfloor \frac{xe}{e+k-1} \right\rfloor$. We can estimate $R_1^{\rho_e,k,n_e}$ now by
\begin{align*}
    R_1^{\rho_e,k,n_2} &\leq \left\lceil \log_{\frac{k}{f_{\rho_e}(k)}}\left( n_e\right) \right\rceil \\
    &= \left\lceil \log_{ 
     \frac{k}{\left\lfloor \frac{ke}{e+k-1}\right\rfloor}
     }\left( \left\lceil \frac{n_0}{2^{e}} \right\rceil + 1 \right) \right\rceil \\
    &\leq \left\lceil \log_{ 
     \frac{e+k-1}{e} 
     }\left( \left\lceil \frac{n_0}{2^{e}} \right\rceil + 1 \right) \right\rceil \\
    &= \left\lceil \frac{\log\left(\left\lceil \frac{n_0}{2^{e}} \right\rceil + 1 \right)}{\log\left( 1 + \frac{k-1}{e} \right)} \right\rceil 
    \\
    &\leq~ \left\lceil \frac{\log\left(\frac{n_0}{2^{e}} + 2 \right)}{\log\left( 1 + \frac{k-1}{E} \right)} \right\rceil \\
    &= \left\lceil \frac{\log\left(n_0 + 2^{e+1} \right) - \log(2^{e})}{\log\left( 1 + \frac{k-1}{E} \right)} \right\rceil\\ 
    &\leq \left\lceil \frac{\log\left(n_0 + 2^{E+1} \right) - e\log(2)}{\log\left( 1 + \frac{k-1}{E} \right)} \right\rceil\\
    &\leq  \left\lceil \frac{\log\left(n_0 + 2^{\log_2\left(\frac{n_0}{n_0-N_{\alpha,\delta}}\right)+2} \right) - e\log(2)}{\log\left( 1 + \frac{k-1}{E} \right)} \right\rceil\\
    &= \left\lceil \frac{\log\left(n_0 + 
    4\frac{n_0}{n_0-N_{\alpha,\delta}}\right)}{\log\left( 1 + \frac{k-1}{E} \right)} - e \cdot \frac{\log(2)}{\log\left( 1 + \frac{k-1}{E} \right)} \right\rceil\\
    &\leq \underbrace{1 + \log_{1 + \frac{k-1}{E}}\left(n_0 + 4 \frac{n_0}{n_0-N_{\alpha,\delta}} \right)}_{=:C_2} - e \cdot \underbrace{\log_{1 + \frac{k-1}{E}}\left(2\right)}_{=:C_1}.
\end{align*}
We can proceed analogously to get an upper bound for $R_2^{\rho_e,k}$:
\begin{align*}
    R_2^{\rho_e,k} &\leq \left\lceil \log_{\frac{k}{f_{\rho_e}(k)}}\left(k\right) \right\rceil 
    ~=~ \left\lceil \log_{ 
     \frac{k}{\left\lfloor \frac{ke}{e+k-1}\right\rfloor}
     }(k) \right\rceil \\
    &\leq \left\lceil \log_{\frac{e+k-1}{e}}(k) \right\rceil 
    ~=~ \left\lceil \log_{1+\frac{k-1}{e}}(k) \right\rceil \\
    &\leq \left\lceil \log_{1+\frac{k-1}{E}}(k) \right\rceil
    ~=:~ C_3 \leq C_2.
\end{align*}
In the next step we can put the above estimations together.
\begin{align*}
    (*) &\leq \sum_{e=1}^{E} \left\lfloor \left(\frac{n_0}{2^{e}} + 2\right)\frac{1}{k} \right\rfloor \cdot \left( -eC_1 + C_2 + C_3\right) \\
    &~~~~~ \cdot \max_{r\in [R^{\rho_e,k,n_e}]} \left(1+ \bar{\gamma}_{\mathbb{A}_{r}(i^{\ast})}^{-1}\left(\max\left\{ \frac{\epsilon}{2},\max_{r \in [R^{\rho_e,k,n_e}]}\frac{\Delta_{(f_{\rho}(|\mathbb{A}_{r}(i^{\ast})|)+1)|\mathbb{A}_{r}(i^{\ast})}}{2}\right\}\right) \right) \\
    &\leq \underbrace{\max_{e \in [E]} \max_{r\in [R^{\rho_e,k,n_e}]} \left(1+ \bar{\gamma}_{\mathbb{A}_{r}(i^{\ast})}^{-1}\left(\max\left\{ \frac{\epsilon}{2},\max_{r \in [R^{\rho_e,k,n_e}]}\frac{\Delta_{(f_{\rho}(|\mathbb{A}_{r}(i^{\ast})|)+1)|\mathbb{A}_{r}(i^{\ast})}}{2}\right\}\right) \right)}_{=:\bar{\gamma}^{-1}}\\
    &~~~~~ \cdot \left(\frac{n_0}{k} \sum_{e=1}^{E} \frac{-eC_1 + C_2 + C_3}{2^{e}} 
    - \frac{2C_1}{k} \sum_{e=1}^{E} e + \frac{2E(C_2 + C_3)}{k}
    \right) \\
    &= \bar{\gamma}^{-1} \cdot \Bigg( \frac{n_0}{k} \cdot \frac{C_1 E - (2^{E}-1)(2C_1-C_2-C_3)}{2^{E}} 
    + \frac{2E(C_2 + C_3) - C_{1}E(E+1)}{k}    \Bigg)
\end{align*}
according to Lemma \ref{Le_geom_series} and the Gaussian sum formula.\\
A cruder bound for the sufficient budget is given by
\begin{align} \label{eq_more_rough_suff_budget}
    \bar{\gamma}^{-1} \cdot \frac{n_0 + 2^{E+1}}{k} \cdot \frac{C_1 E - (2^{E}-1)(2C_1-C_2-C_3)}{2^{E}}.
\end{align}
Recall that the number of parallel runs is decreasing with the epoch $e,$ so that in the worst case the configuration returned by AC-Band is sampled only in the last epoch $E.$
Consequently it will be run in parallel only with configurations, which are considered in the last epoch $E,$ and guaranteed to be a local $\epsilon$-best configuration (see Theorem \ref{Thm_generalSufficient_eps_max}).    
By Assumption (A3), the local $\epsilon$-best property corresponds to a global $\epsilon$-best property with probability of at least
\begin{align*}
    &1- \frac{1}{\#\text{query sets containing $i^{\ast}$ in epoch } E}
    = 1 - \left(R^{\rho_{E}, k, n_{E}}\right)^{-1}
\end{align*}
for this worst case.
In addition we have to take into account that an $\epsilon$-best configuration is contained only with probability at least $1-\delta$, such that we get an overall probability of at least
\begin{align*}
    \min\{ 1- \delta, 1 - \left(R^{\rho_{E}, k, n_{E}}\right)^{-1} \}
\end{align*}
that AC-Band returns an $\epsilon$-best configuration.
\end{proof}
Note that the total budget of AC band is divided among the epoch-wise calls of CSE by means of the quotient $c_e$: 
\begin{align*}
    c_e = \frac{(C_1 E - (2^{E}-1)(2C_1-C_2-C_3))2^{e}}{2^{E}(-eC_1+C_2+C_3)}.
\end{align*}
This quotient is obtained by bounding the sufficient budget for CSE similar as in the proof of Theorem \ref{Thm:ACBandSufficientBudget} to obtain \eqref{eq_more_rough_suff_budget}:
\begin{align*}
    z_(\rho_e, k, n_e, \epsilon) &= R^{\rho_e,k,n_e} \max_{r \in [R^{\rho_e,k,n_e}]} P_r^{\rho_e,k,n_e} \\
    &~~~~~\cdot \max_{r\in [R^{\rho_e,k,n_e}]} \left(1+ \bar{\gamma}_{\mathbb{A}_{r}(i^{\ast})}^{-1}\left(\max\left\{ \frac{\epsilon}{2},\max_{r \in [R^{\rho_e,k,n_e}]}\frac{\Delta_{(f_{\rho}(|\mathbb{A}_{r}(i^{\ast})|)+1)|\mathbb{A}_{r}(i^{\ast})}}{2}\right\}\right) \right)\\
    &\leq \bar{\gamma}^{-1} (R_1^{\rho_e,k,n_e} + R_2^{\rho_e,k})  \left\lfloor \frac{n_e}{k} \right\rfloor  \\
    &\leq \bar{\gamma}^{-1}\cdot \frac{n_0 + 2^{E+1}}{k} \cdot \frac{-eC_1+C_2+C_3}{2^{e}}.
\end{align*}
AC-Band thus allocates its entire budget such that any call to CSE is guaranteed to return an appropriate configuration once the total budget is sufficiently large.

\subsection{Discussion of Sufficient Budget} \label{sec:SuffBudgetDiscussion}
The behavior of the sufficient budget with respect to $k,$ $\alpha$ and $\delta$ is illustrated in Figure \ref{fig:suffBudgets}.
Note that we ignore the $\gamma^{-1}$ terms in these plots, as these depend on the underlying AC problem and occur as a multiplicative constant in the sufficient budget.

\begin{figure}[h]
    \centering
    \includegraphics[width=0.98\linewidth]{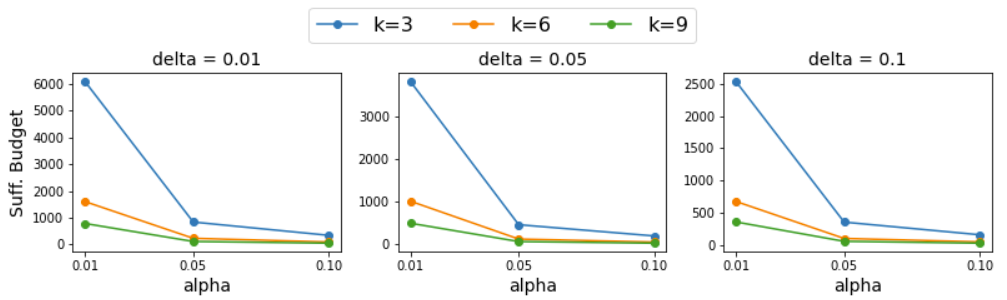}
    \caption{Sufficient budget for different values for k, $\alpha$ and $\delta$.}
    \label{fig:suffBudgets}
\end{figure}

\noindent The dependency of the sufficient budget on $k,\alpha$ and $\delta$ is as expected, since it decreases with increasing $k,\alpha$ and $\delta,$ respectively.

\section{Extended Experiments} \label{sec_further_results}

We provide further details regarding the experiments in Section \ref{sec:experiments}. In particular, the results from Figure \ref{ACBand results d05} are reported in more detail in Table \ref{table_d05_rcw} and augmented by the results for the Regions200 dataset, for which we provide a similar illustration of the results in Figure \ref{ACBand results d05 Regions200}  Moreover, we outline additional experimental results and provide additional metrics to evaluate the quality of the configurations found. The experimental setup used here is the same as in the main paper.
We report two additional metrics to provide additional insights: the percent gap to subset-best and the $R^{\delta}$ metric used in previous works \cite{kleinberg2019procrastinating, weisz2018leapsandbounds, weisz2019capsandruns,weisz2020impatientcapsandruns} within the following tables. The percent gap to subset-best is a variation of the percent gap to best metric where only the configurations that were sampled by the method during a run are considered. In this way, we can see how good a method performs within the sample it selects. A value of $0$ means the configuration returned is the best in the subset. A $10\%$ cutoff is used for the $\delta$-capped runtime. We provide results for two additional experiments: (i) varying the $\delta$ for AC-Band and (I)CAR(++) and, (ii) increasing the configuration sampling budget of AC-band.

\begin{figure*}[ht]

    \centering
    \includegraphics[width=0.9\textwidth]{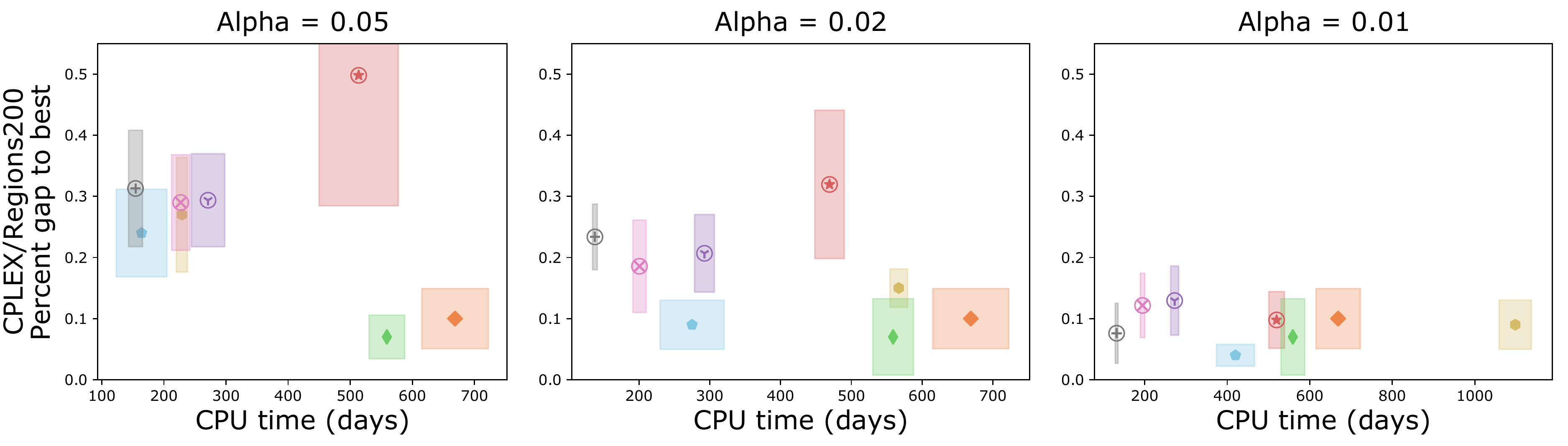}
  
  \caption{Mean CPU time and percent gap to best over $5$ seeds for $\delta=0.05$ and different $\alpha$ (columns) for AC-Band, ICAR, CAR++ and Hyperband on the Regions200 dataset. Circles indicate variants of AC-Band. Rectangles represent the standard error over the seeds. The number of configurations tried for CAR++: $\{97, 245, 492\}$, ICAR: $\{134, 351, 724\},$ AC-Band: $\{60, 153, 303\}$, Hyperband($\eta=5$): $\{842\},$ Hyperband($\eta=8$): $\{ 618\}$.} \label{ACBand results d05 Regions200}
\end{figure*}

\begin{table}[h]
\setcounter{table}{0}
\renewcommand{\thetable}{(\alph{table})}
\renewcommand{\tablename}{}
\tiny
\center

\begin{tabular}{@{\hspace{0\tabcolsep}}c@{\hspace{0.3\tabcolsep}}|@{\hspace{0.2\tabcolsep}}  r @{\hspace{0.6\tabcolsep}}  r @{\hspace{0.6\tabcolsep}} r @{\hspace{0.6\tabcolsep}} r @{\hspace{0.6\tabcolsep}} r @{\hspace{0.6\tabcolsep}} r@{\hspace{0.4\tabcolsep}}| @{\hspace{0.4\tabcolsep}} r @{\hspace{0.6\tabcolsep}}  r @{\hspace{0.6\tabcolsep}} r  @{\hspace{0.6\tabcolsep}} r @{\hspace{0.6\tabcolsep}} r @{\hspace{0.6\tabcolsep}} r@{\hspace{0.4\tabcolsep}}| @{\hspace{0.4\tabcolsep}} r @{\hspace{0.6\tabcolsep}} r @{\hspace{0.6\tabcolsep}} r  @{\hspace{0.6\tabcolsep}} r @{\hspace{0.6\tabcolsep}} r @{\hspace{0.6\tabcolsep}} r@{\hspace{0.4\tabcolsep}}| @{\hspace{0.4\tabcolsep}} r @{\hspace{0.6\tabcolsep}} r @{\hspace{0.6\tabcolsep}} r  @{\hspace{0.6\tabcolsep}} r @{\hspace{0.6\tabcolsep}} r @{\hspace{0.6\tabcolsep}} r @{\hspace{0.6\tabcolsep}}  r @{\hspace{0.6\tabcolsep}} r @{\hspace{0.6\tabcolsep}} r @{\hspace{0.6\tabcolsep}} r @{\hspace{0.6\tabcolsep}} r @{\hspace{0.6\tabcolsep}} r @{\hspace{0.6\tabcolsep}}}

 & \multicolumn{6}{c}{CPU Time (thousand days)}  & \multicolumn{6}{c}{Percent gap to best}  & \multicolumn{6}{c}{Percent gap to subset-best}  & \multicolumn{6}{c}{$R^{\delta}$}\\
\hline
$\alpha$ & \multicolumn{2}{@{\hspace{0.5\tabcolsep}} c @{\hspace{0.5\tabcolsep}}}{$0.05$} & \multicolumn{2}{@{\hspace{0.5\tabcolsep}} c @{\hspace{0.5\tabcolsep}}}{$0.02$} & \multicolumn{2}{@{\hspace{0.5\tabcolsep}} c @{\hspace{0.5\tabcolsep}}}{$0.01$}& \multicolumn{2}{@{\hspace{0.5\tabcolsep}} c @{\hspace{0.5\tabcolsep}}}{$0.05$} & \multicolumn{2}{@{\hspace{0.5\tabcolsep}} c @{\hspace{0.5\tabcolsep}}}{$0.02$} & \multicolumn{2}{@{\hspace{0.5\tabcolsep}} c @{\hspace{0.5\tabcolsep}}}{$0.01$}& \multicolumn{2}{@{\hspace{0.5\tabcolsep}} c @{\hspace{0.5\tabcolsep}}}{$0.05$} & \multicolumn{2}{@{\hspace{0.5\tabcolsep}} c @{\hspace{0.5\tabcolsep}}}{$0.02$} & \multicolumn{2}{@{\hspace{0.5\tabcolsep}} c @{\hspace{0.5\tabcolsep}}}{$0.01$}& \multicolumn{2}{@{\hspace{0.5\tabcolsep}} c @{\hspace{0.5\tabcolsep}}}{$0.05$} & \multicolumn{2}{@{\hspace{0.5\tabcolsep}} c @{\hspace{0.5\tabcolsep}}}{$0.02$} & \multicolumn{2}{@{\hspace{0.5\tabcolsep}} c @{\hspace{0.5\tabcolsep}}}{$0.01$}\\
\hline
 Method & $\mu$ & $\sigma$ & $\mu$ & $\sigma$ & $\mu$ & $\sigma$ & $\mu$ & $\sigma$ & $\mu$ & $\sigma$ & $\mu$ & $\sigma$ & $\mu$ & $\sigma$ & $\mu$ & $\sigma$ & $\mu$ & $\sigma$ & $\mu$ & $\sigma$ & $\mu$ & $\sigma$ & $\mu$ & $\sigma$\\
\hline
ICAR & 100.64 & 12.74 & 242.74 & 14.96 & 467.32 & 25.08  & 0.02 & 0.04 &0.01 & 0.01 &0.02 & 0.03 &0.01 & 0.02 &0.01 & 0.01 &0.02 & 0.03 & 5.0 & 0.1 & 4.9 & 0.1 & 4.9 & 0.1 \\
CAR++ & 92.30 & 5.48 & 224.21 & 16.38 & 452.06 & 18.08  &0.05 & 0.04 &0.01 & 0.01 &0.01 & 0.01 &0.02 & 0.03 &0.00 & 0.01 & 0.01 & 0.01 &  5.2 & 0.2 & 4.9 & 0.1 & 4.9 & 0.1 \\
CAR & 157.76 & 18.07 & 367.86 & 7.24 & 771.45 & 21.72  & 0.04 & 0.03 & 0.01 & 0.01 & 0.01 & 0.01 & 0.01 & 0.02 & 0.00 & 0.01 & 0.01 & 0.01 & 5.2 & 0.2 & 4.9 & 0.1 & 4.9 & 0.1 \\
\hline
$k = 2$ & 13.05 & 2.04 & 11.76 & 0.90 & 11.02 & 0.56 & 0.14 & 0.17 & 0.16 & 0.19 & 0.05 & 0.08 & 0.09 & 0.12 & 0.14 & 0.19 & 0.05 & 0.08 & 5.7 & 0.90 & 5.7 & 1.0 & 5.1 & 0.5 \\ 
$k = 4$  & 16.32 & 1.91 & 14.91 & 0.68 & 14.82 & 0.53 & 0.05 & 0.06 & 0.02 & 0.03 & 0.01 & 0.02 & 0.00 & 0.00 & 0.00 & 0.01 & 0.01 & 0.02 & 5.3 & 0.3 & 5.1 & 0.2 & 4.9 & 0.1 \\ 
$k = 8$  & 20.09 & 1.92 & 21.58 & 1.24 & 21.21 & 0.91  & 0.12 & 0.14 & 0.01 & 0.01 & 0.01 & 0.01 & 0.06 & 0.11 & 0.00 & 0.01 & 0.01 & 0.01 & 5.6 & 0.8 & 5.0 & 0.1 & 4.9 & 0.1 \\ 
$k = 16$ & 37.87 & 3.46 & 36.07 & 1.71 & 38.45 & 0.32 & 0.05 & 0.06 & 0.05 & 0.07 & 0.05 & 0.10 & 0.02 & 0.03 & 0.03 & 0.05 & 0.05 & 0.10 & 5.2 & 0.3 & 5.0 & 0.3 & 5.2 & 0.6 \\ 
\end{tabular}\caption{CNFuzzDD dataset}

\begin{tabular}{@{\hspace{0\tabcolsep}}c@{\hspace{0.3\tabcolsep}}|@{\hspace{0.2\tabcolsep}}  r @{\hspace{0.6\tabcolsep}}  r @{\hspace{0.6\tabcolsep}} r @{\hspace{0.6\tabcolsep}} r @{\hspace{0.6\tabcolsep}} r @{\hspace{0.6\tabcolsep}} r@{\hspace{0.4\tabcolsep}}| @{\hspace{0.4\tabcolsep}} r @{\hspace{0.6\tabcolsep}}  r @{\hspace{0.6\tabcolsep}} r  @{\hspace{0.6\tabcolsep}} r @{\hspace{0.6\tabcolsep}} r @{\hspace{0.6\tabcolsep}} r@{\hspace{0.4\tabcolsep}}| @{\hspace{0.4\tabcolsep}} r @{\hspace{0.6\tabcolsep}} r @{\hspace{0.6\tabcolsep}} r  @{\hspace{0.6\tabcolsep}} r @{\hspace{0.6\tabcolsep}} r @{\hspace{0.6\tabcolsep}} r@{\hspace{0.4\tabcolsep}}| @{\hspace{0.4\tabcolsep}} r @{\hspace{0.6\tabcolsep}} r @{\hspace{0.6\tabcolsep}} r  @{\hspace{0.6\tabcolsep}} r @{\hspace{0.6\tabcolsep}} r @{\hspace{0.6\tabcolsep}} r @{\hspace{0.6\tabcolsep}}  r @{\hspace{0.6\tabcolsep}} r @{\hspace{0.6\tabcolsep}} r @{\hspace{0.6\tabcolsep}} r @{\hspace{0.6\tabcolsep}} r @{\hspace{0.6\tabcolsep}} r @{\hspace{0.6\tabcolsep}}}

 & \multicolumn{6}{c}{CPU Time (days)}  & \multicolumn{6}{c}{Percent gap to best}  & \multicolumn{6}{c}{Percent gap to subset-best}  & \multicolumn{6}{c}{$R^{\delta}$}\\
\hline
$\alpha$ & \multicolumn{2}{@{\hspace{0.5\tabcolsep}} c @{\hspace{0.5\tabcolsep}}}{$0.05$} & \multicolumn{2}{@{\hspace{0.5\tabcolsep}} c @{\hspace{0.5\tabcolsep}}}{$0.02$} & \multicolumn{2}{@{\hspace{0.5\tabcolsep}} c @{\hspace{0.5\tabcolsep}}}{$0.01$}& \multicolumn{2}{@{\hspace{0.5\tabcolsep}} c @{\hspace{0.5\tabcolsep}}}{$0.05$} & \multicolumn{2}{@{\hspace{0.5\tabcolsep}} c @{\hspace{0.5\tabcolsep}}}{$0.02$} & \multicolumn{2}{@{\hspace{0.5\tabcolsep}} c @{\hspace{0.5\tabcolsep}}}{$0.01$}& \multicolumn{2}{@{\hspace{0.5\tabcolsep}} c @{\hspace{0.5\tabcolsep}}}{$0.05$} & \multicolumn{2}{@{\hspace{0.5\tabcolsep}} c @{\hspace{0.5\tabcolsep}}}{$0.02$} & \multicolumn{2}{@{\hspace{0.5\tabcolsep}} c @{\hspace{0.5\tabcolsep}}}{$0.01$}& \multicolumn{2}{@{\hspace{0.5\tabcolsep}} c @{\hspace{0.5\tabcolsep}}}{$0.05$} & \multicolumn{2}{@{\hspace{0.5\tabcolsep}} c @{\hspace{0.5\tabcolsep}}}{$0.02$} & \multicolumn{2}{@{\hspace{0.5\tabcolsep}} c @{\hspace{0.5\tabcolsep}}}{$0.01$}\\
\hline
 Method & $\mu$ & $\sigma$ & $\mu$ & $\sigma$ & $\mu$ & $\sigma$ & $\mu$ & $\sigma$ & $\mu$ & $\sigma$ & $\mu$ & $\sigma$ & $\mu$ & $\sigma$ & $\mu$ & $\sigma$ & $\mu$ & $\sigma$ & $\mu$ & $\sigma$ & $\mu$ & $\sigma$ & $\mu$ & $\sigma$\\
\hline
ICAR & 164.30 & 91.05 & 274.84 & 100.72 & 420.15 & 103.24 & 0.24 & 0.16 & 0.09 & 0.09 & 0.04 & 0.04 & 0.00 & 0.00 & 0.00 & 0.00 & 0.00 & 0.00 & 34.8 & 4.3 & 29.8 & 2.2 & 28.5 & 1.8 \\
CAR++ & 229.29 & 19.93 & 566.99 & 28.21 & 1097.90 & 88.41 & 0.27 & 0.17 & 0.15 & 0.07 & 0.09 & 0.09 & 0.00 & 0.00 & 0.00 & 0.00 & 0.00 & 0.00 & 35.3 & 4.3 & 32.0 & 2.2 & 29.8 & 1.8 \\
CAR & 523.69 & 53.34 & 1294.87 & 64.11 & 2549.22 & 199.00 & 0.27 & 0.17 & 0.16 & 0.09 & 0.09 & 0.09 & 0.00 & 0.00 & 0.10 & 0.30 & 0.00 & 0.00 & 35.3 & 4.5 & 31.9 & 1.6 & 29.8 & 2.2 \\
\hline
$k = 2$ & 154.27 & 25.29 & 137.79 & 7.59 & 132.00 & 7.78 & 0.31 & 0.21 & 0.23 & 0.12 & 0.08 & 0.11 & 0.00 & 0.00 & 0.01 & 0.02 & 0.00 & 0.00 & 36.6 & 5.9 & 34.6 & 5.0 & 30.2 & 1.6  \\ 
$k = 4$  & 227.07 & 33.18 & 200.72 & 20.76 & 194.88 & 12.07 & 0.29 & 0.17 & 0.19 & 0.17 & 0.12 & 0.12 & 0.00 & 0.00 & 0.00 & 0.00 & 0.01 & 0.02 & 36.2 & 4.4 & 34.3 & 5.4 & 31.8 & 2.7  \\ 
$k = 8$ & 271.02 & 60.26 & 292.56 & 31.26 & 272.56 & 21.64 & 0.29 & 0.17 & 0.21 & 0.14 & 0.13 & 0.13 & 0.00 & 0.00 & 0.01 & 0.02 & 0.01 & 0.02 & 36.2 & 6.3 & 33.5 & 3.5 & 31.1 & 2.2  \\
$k = 16$ & 513.57 & 142.68 & 469.19 & 46.49 & 519.70 & 42.30 & 0.50 & 0.48 & 0.32 & 0.27 & 0.10 & 0.10 & 0.00 & 0.00 & 0.03 & 0.06 & 0.00 & 0.00 & 40.0 & 9.0 & 36.8 & 6.9 & 30.6 & 1.7  \\ 

\end{tabular}\caption{Regions200 dataset}

\begin{tabular}{@{\hspace{0\tabcolsep}}c@{\hspace{0.3\tabcolsep}}|@{\hspace{0.2\tabcolsep}}  r @{\hspace{0.6\tabcolsep}}  r @{\hspace{0.6\tabcolsep}} r @{\hspace{0.6\tabcolsep}} r @{\hspace{0.6\tabcolsep}} r @{\hspace{0.6\tabcolsep}} r@{\hspace{0.4\tabcolsep}}| @{\hspace{0.4\tabcolsep}} r @{\hspace{0.6\tabcolsep}}  r @{\hspace{0.6\tabcolsep}} r  @{\hspace{0.6\tabcolsep}} r @{\hspace{0.6\tabcolsep}} r @{\hspace{0.6\tabcolsep}} r@{\hspace{0.4\tabcolsep}}| @{\hspace{0.4\tabcolsep}} r @{\hspace{0.6\tabcolsep}} r @{\hspace{0.6\tabcolsep}} r  @{\hspace{0.6\tabcolsep}} r @{\hspace{0.6\tabcolsep}} r @{\hspace{0.6\tabcolsep}} r@{\hspace{0.4\tabcolsep}}| @{\hspace{0.4\tabcolsep}} r @{\hspace{0.6\tabcolsep}} r @{\hspace{0.6\tabcolsep}} r  @{\hspace{0.6\tabcolsep}} r @{\hspace{0.6\tabcolsep}} r @{\hspace{0.6\tabcolsep}} r @{\hspace{0.6\tabcolsep}}  r @{\hspace{0.6\tabcolsep}} r @{\hspace{0.6\tabcolsep}} r @{\hspace{0.6\tabcolsep}} r @{\hspace{0.6\tabcolsep}} r @{\hspace{0.6\tabcolsep}} r @{\hspace{0.6\tabcolsep}}}

 & \multicolumn{6}{c}{CPU Time (thousand days)}  & \multicolumn{6}{c}{Percent gap to best}  & \multicolumn{6}{c}{Percent gap to subset-best}  & \multicolumn{6}{c}{$R^{\delta}$}\\
\hline
$\alpha$ & \multicolumn{2}{@{\hspace{0.5\tabcolsep}} c @{\hspace{0.5\tabcolsep}}}{$0.05$} & \multicolumn{2}{@{\hspace{0.5\tabcolsep}} c @{\hspace{0.5\tabcolsep}}}{$0.02$} & \multicolumn{2}{@{\hspace{0.5\tabcolsep}} c @{\hspace{0.5\tabcolsep}}}{$0.01$}& \multicolumn{2}{@{\hspace{0.5\tabcolsep}} c @{\hspace{0.5\tabcolsep}}}{$0.05$} & \multicolumn{2}{@{\hspace{0.5\tabcolsep}} c @{\hspace{0.5\tabcolsep}}}{$0.02$} & \multicolumn{2}{@{\hspace{0.5\tabcolsep}} c @{\hspace{0.5\tabcolsep}}}{$0.01$}& \multicolumn{2}{@{\hspace{0.5\tabcolsep}} c @{\hspace{0.5\tabcolsep}}}{$0.05$} & \multicolumn{2}{@{\hspace{0.5\tabcolsep}} c @{\hspace{0.5\tabcolsep}}}{$0.02$} & \multicolumn{2}{@{\hspace{0.5\tabcolsep}} c @{\hspace{0.5\tabcolsep}}}{$0.01$}& \multicolumn{2}{@{\hspace{0.5\tabcolsep}} c @{\hspace{0.5\tabcolsep}}}{$0.05$} & \multicolumn{2}{@{\hspace{0.5\tabcolsep}} c @{\hspace{0.5\tabcolsep}}}{$0.02$} & \multicolumn{2}{@{\hspace{0.5\tabcolsep}} c @{\hspace{0.5\tabcolsep}}}{$0.01$}\\
\hline
 Method & $\mu$ & $\sigma$ & $\mu$ & $\sigma$ & $\mu$ & $\sigma$ & $\mu$ & $\sigma$ & $\mu$ & $\sigma$ & $\mu$ & $\sigma$ & $\mu$ & $\sigma$ & $\mu$ & $\sigma$ & $\mu$ & $\sigma$ & $\mu$ & $\sigma$ & $\mu$ & $\sigma$ & $\mu$ & $\sigma$\\
\hline
ICAR & 1.28 & 0.39 & 2.03 & 0.30 & 4.07 & 0.24  & 0.14 & 0.08 & 0.08 & 0.03 & 0.06 & 0.02 & 0.00 & 0.01 & 0.00 & 0.00 & 0.00 & 0.00 & 156.1 & 11.9 & 146.5 & 4.1 & 143.3 & 4.9  \\
CAR++ & 1.73 & 0.37 & 3.64 & 0.18 & 7.52 & 0.13 & 0.17 & 0.09 & 0.10 & 0.03 & 0.06 & 0.02 & 0.00 & 0.01 & 0.00 & 0.00 & 0.00 & 0.00 & 162.1 & 11.9 & 149.1 & 4.1 & 143.3 & 4.9 \\
CAR  & 3.30 & 0.50 & 7.59 & 0.19 & 15.65 & 0.25 & 0.17 & 0.09 & 0.10 & 0.03 & 0.06 & 0.02 & 0.00 & 0.01 & 0.00 & 0.00 & 0.00 & 0.00 & 160.1 & 13.3 & 149.1 & 4.7 & 143.3 & 4.9 \\
\hline
$k = 2$ & 0.30 & 0.02 & 0.27 & 0.01 & 0.26 & 0.01 & 0.18 & 0.13 & 0.12 & 0.12 & 0.10 & 0.04 & 0.02 & 0.03 & 0.04 & 0.07 & 0.04 & 0.04 & 164.1 & 32.2 & 152.2 & 17.8 & 144.7 & 2.9  \\ 
$k = 4$  & 0.52 & 0.04 & 0.47 & 0.03 & 0.45 & 0.02 & 0.21 & 0.08 & 0.19 & 0.22 & 0.07 & 0.08 & 0.03 & 0.03 & 0.08 & 0.16 & 0.02 & 0.03 & 169.8 & 18.5 & 165.2 & 36.5 & 139.8 & 8.4  \\ 
$k = 8$ & 0.69 & 0.09 & 0.72 & 0.03 & 0.71 & 0.04 & 0.20 & 0.11 & 0.23 & 0.22 & 0.09 & 0.09 & 0.04 & 0.06 & 0.14 & 0.15 & 0.06 & 0.07 & 165.6 & 24.0 & 173.0 & 42.5 & 143.0 & 14.0  \\ 
$k = 16$ & 1.42 & 0.29 & 1.31 & 0.06 & 1.43 & 0.08 & 0.21 & 0.08 & 0.15 & 0.16 & 0.12 & 0.06 & 0.05 & 0.07 & 0.03 & 0.06 & 0.05 & 0.04 & 167.6 & 15.5 & 158.2 & 30.5 & 146.4 & 4.6  \\ 
\end{tabular}\caption{RCW dataset}

\renewcommand{\tablename}{Table}
\renewcommand{\thetable}{1}
\caption{Mean ($\mu$) and standard derivation ($\sigma$) for CPU time, percent gap to best, percent gap to subset-best and mean $\delta$-capped runtime over $5$ seeds for $\delta=\textbf{0.05}$ and different $\alpha$ (columns) for AC-Band, ICAR, CAR++, CAR on the CNFuzzDD (top), Regions200 (middle) and RCW (bottom) dataset. The number of configurations tried for CAR(++): $\{97, 245, 492\}$, ICAR: $\{134, 351, 724\},$ AC-Band: $\{60, 153, 303\}$.}\label{table_d05_rcw}
\end{table}

\paragraph{Varying $\delta$}
The user must decide $\delta$ based on their preferences, thus there is no ``correct'' value to set $\delta$ to in our experiments. Therefore, we also experiment with $\delta = 0.01$ in addition to the results in Section \ref{sec:experiments} where $\delta = 0.05$ is used. The results for a lower $\delta$ (see Figure \ref{fig_d01} or Table \ref{table_d01_rcw}) are consistent with the results reported for $\delta = 0.05$. In particular, AC-Band with $k = 2$ lies on the Pareto front of percent gap to best and CPU time, backing up the claim that a lower value of $k$ is preferable. With $\delta = 0.01$ and $k = 2$, AC-Band is $80\%$ percent faster than ICAR and $74\%$ faster than Hyperband over all $\alpha$ and all datasets, while providing configurations that are only $7\%$ and $6\%$ worse in terms of the gap to the best configuration. Furthermore, AC-Band exhibits a nearly linear speedup with the number of available cores, leading to a low wallclock time for increasing $k$. We note that a real parallel implementation would, of course, suffer from some extra overhead.
\begin{figure*}[h]
  \begin{subfigure}[t]{0.3\textwidth}
    \centering
    \includegraphics[width=\textwidth]{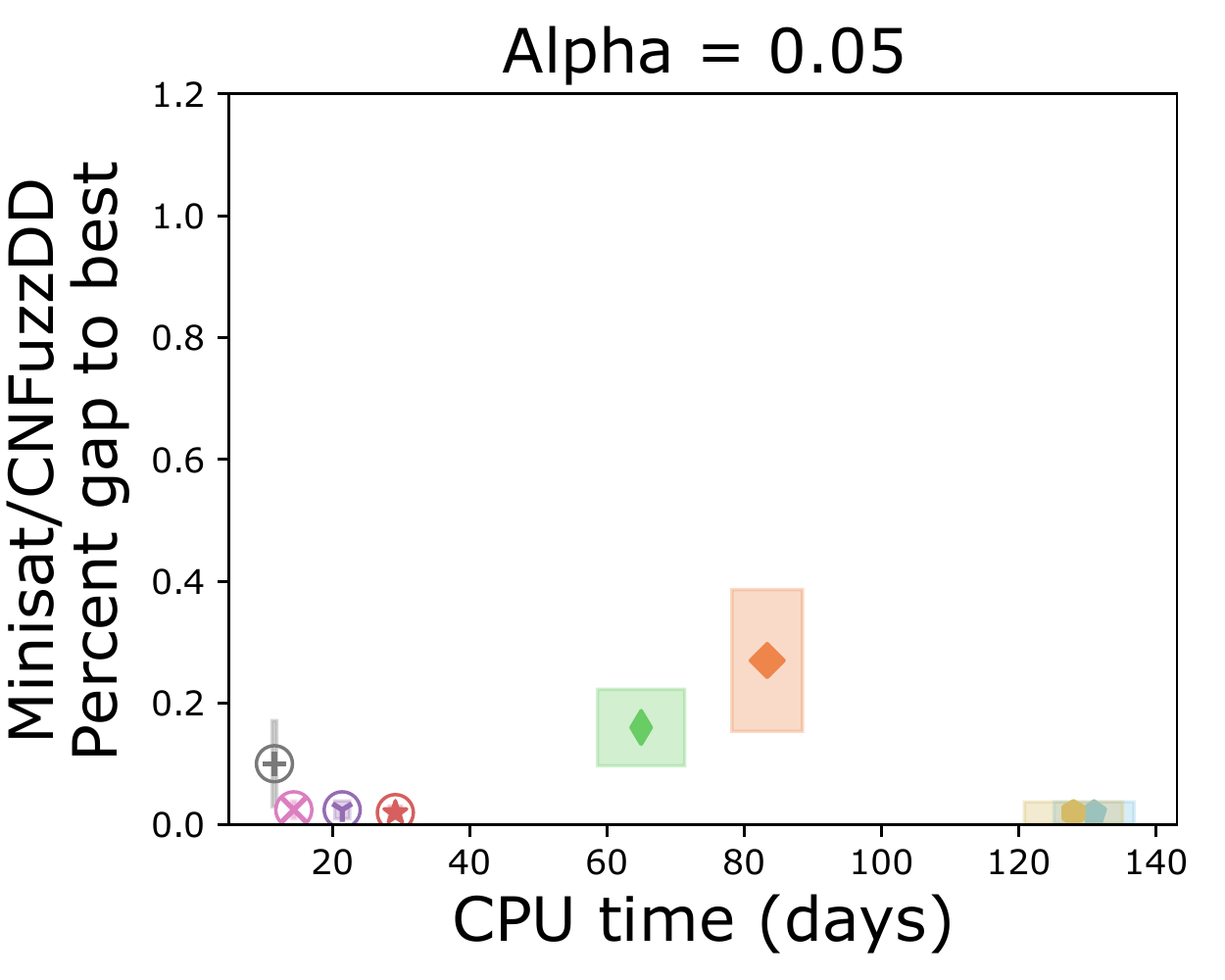}
  \end{subfigure}
  \begin{subfigure}[t]{0.3\textwidth}
    \centering
    \includegraphics[width=\textwidth]{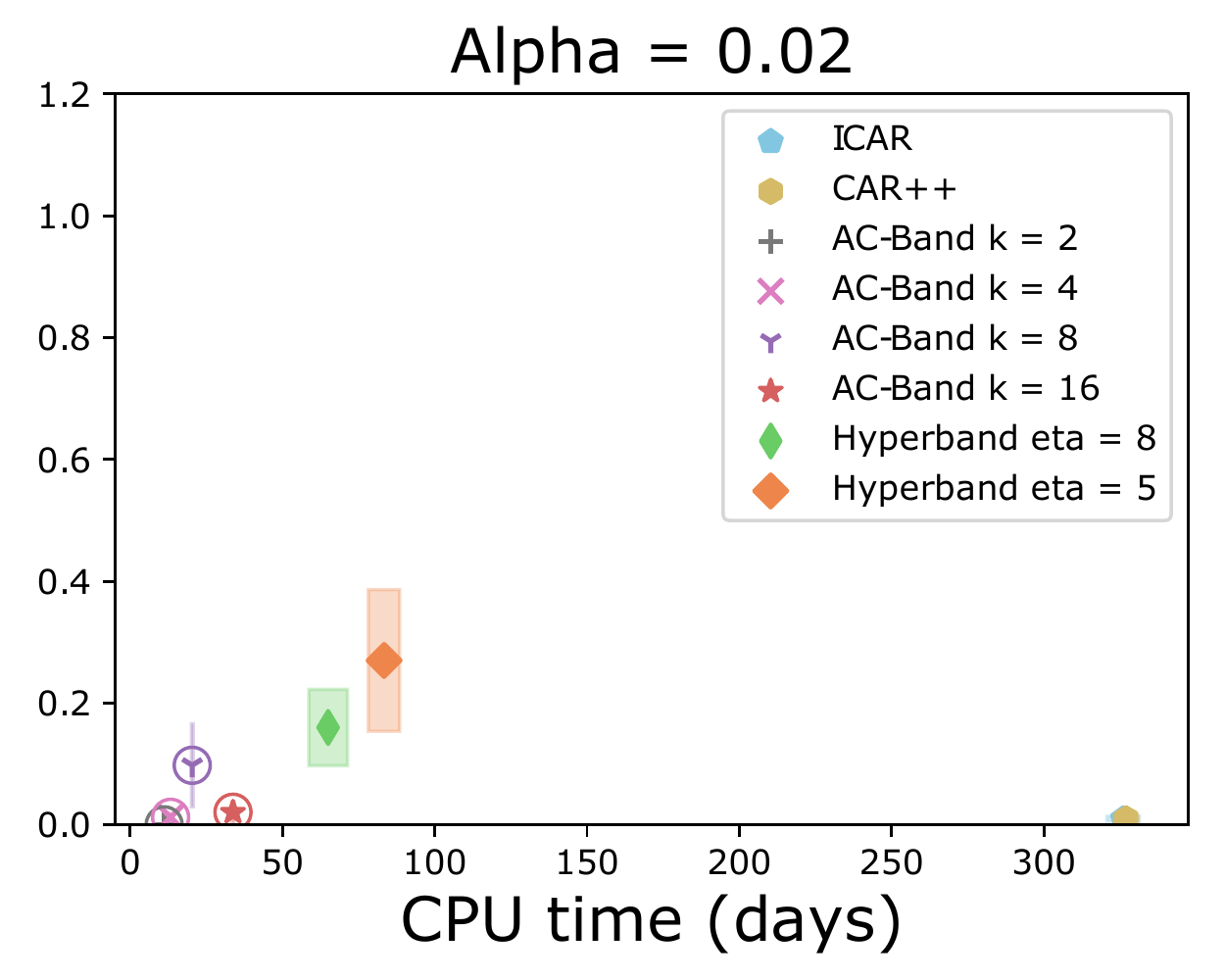}
  \end{subfigure}
  \begin{subfigure}[t]{0.3\textwidth}
    \centering
    \includegraphics[width=\textwidth]{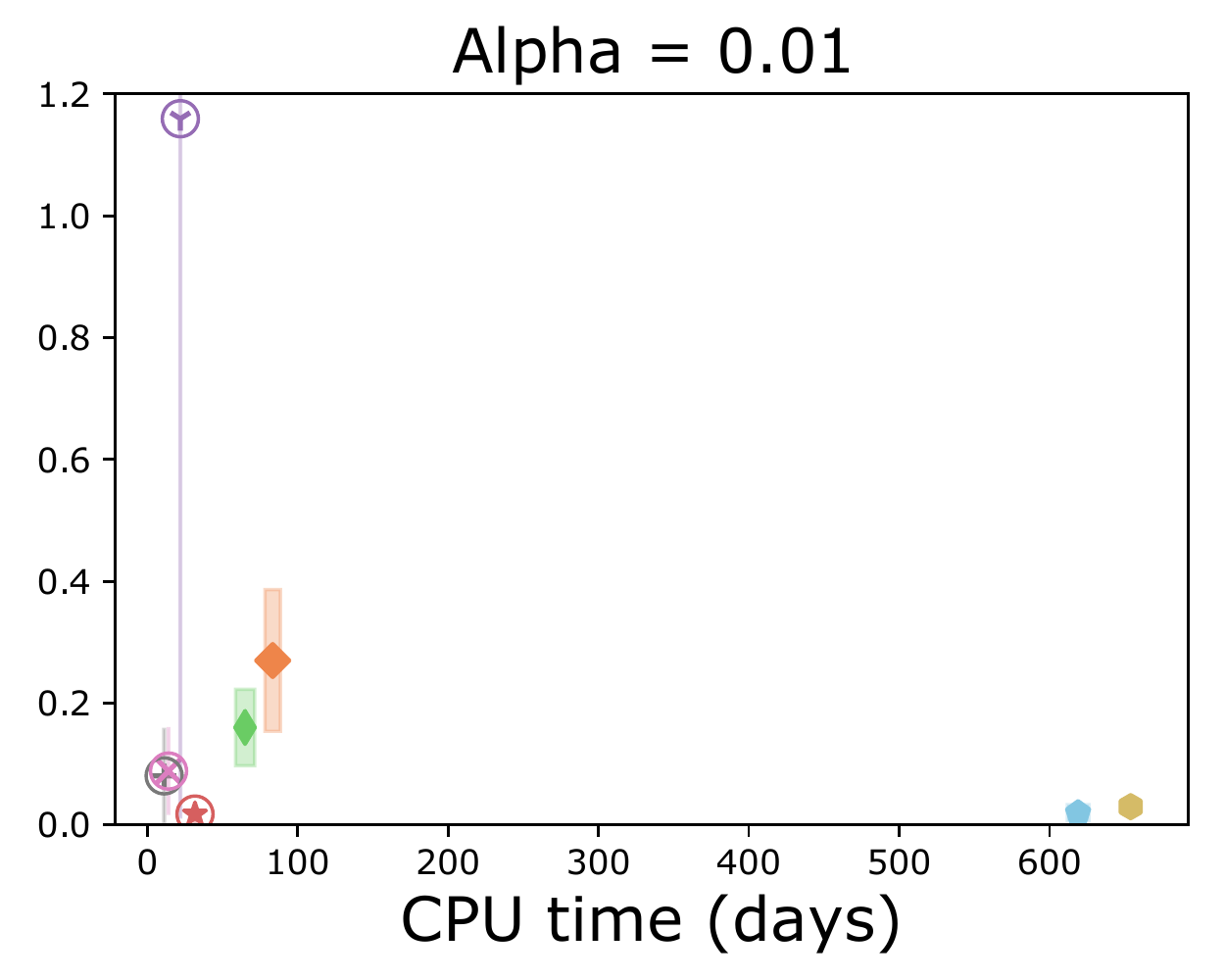}
  \end{subfigure}

  \begin{subfigure}[t]{0.3\textwidth}
    \centering
    \includegraphics[width=\textwidth]{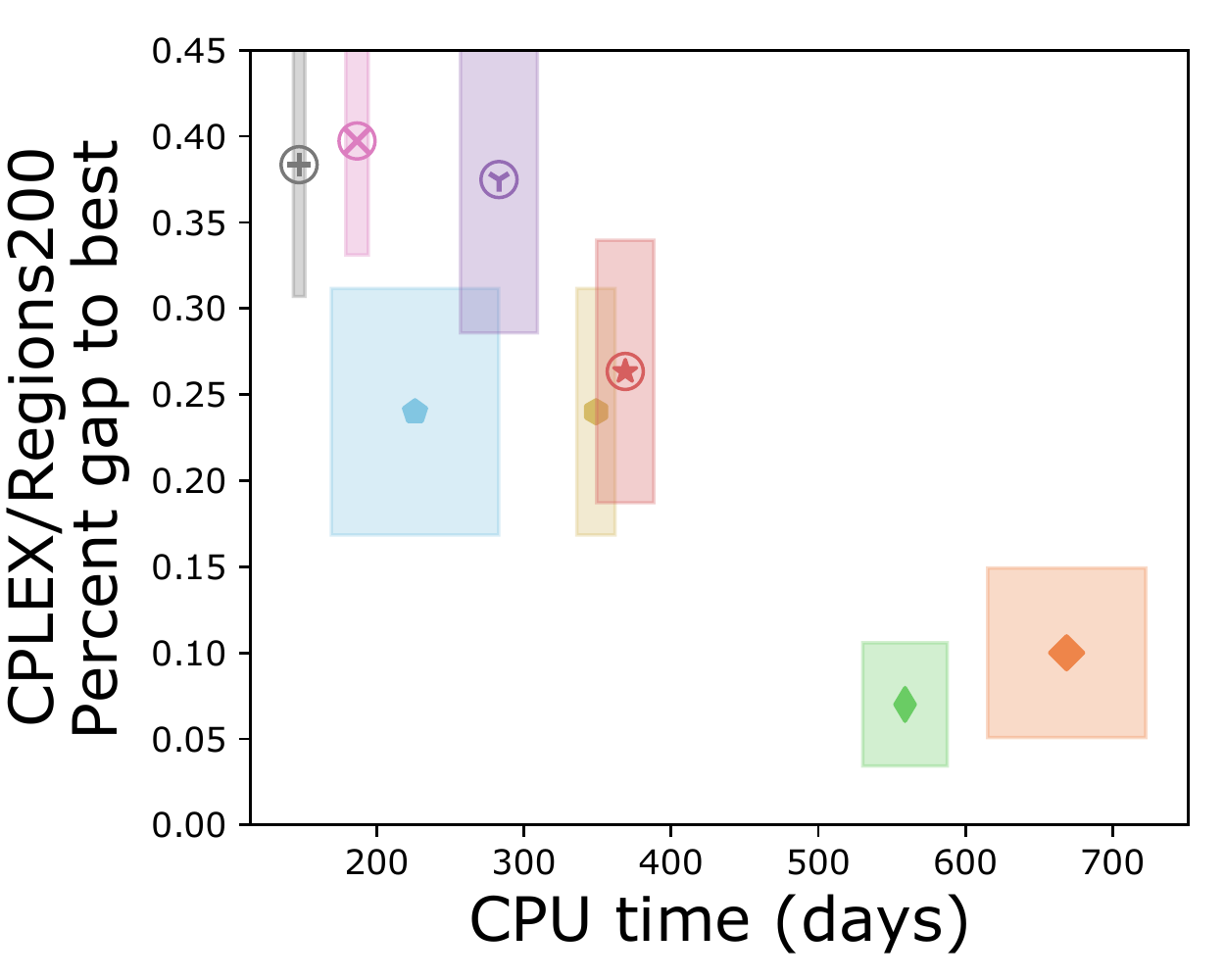}
  \end{subfigure}
  \begin{subfigure}[t]{0.3\textwidth}
    \centering
    \includegraphics[width=\textwidth]{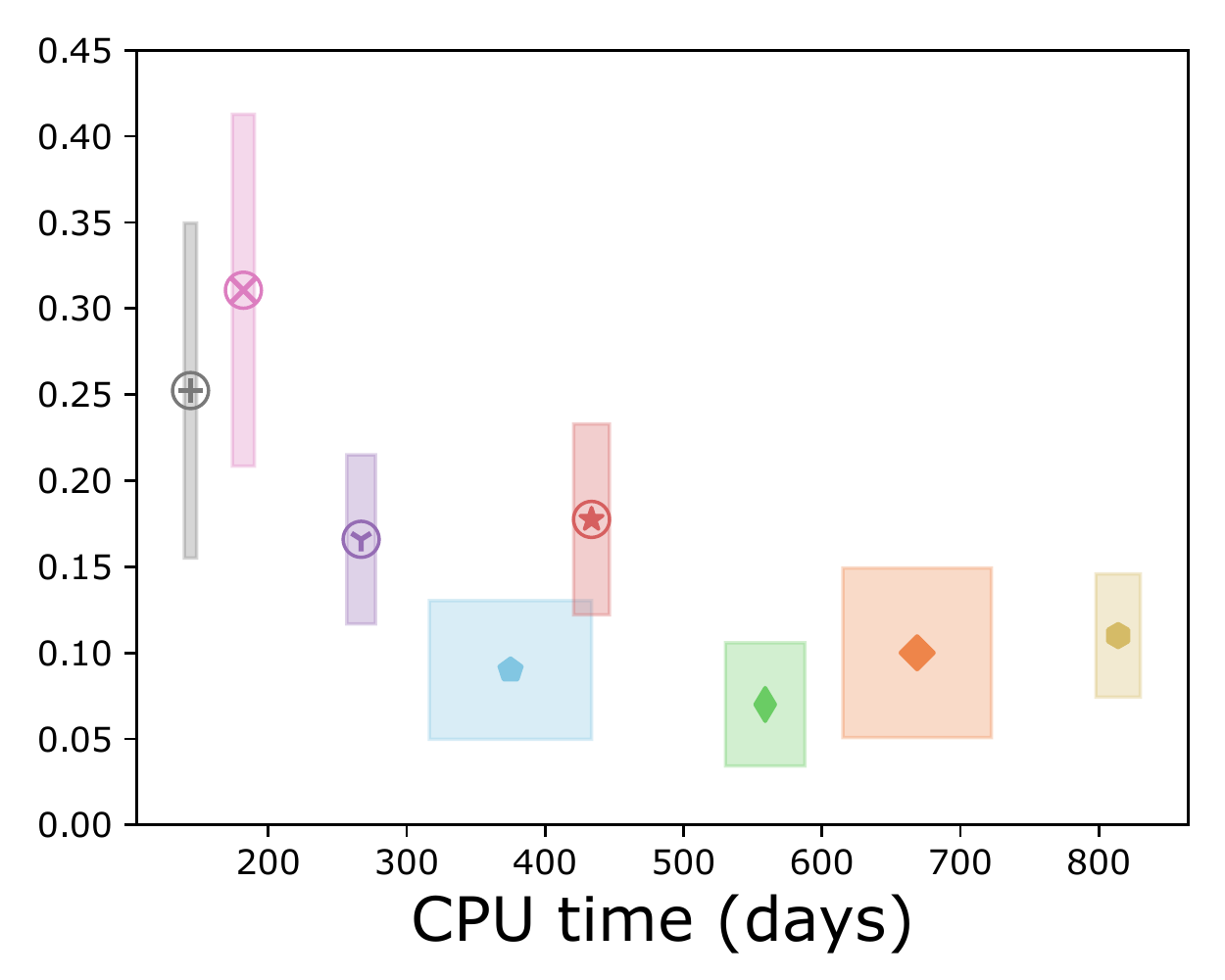}
  \end{subfigure}
  \begin{subfigure}[t]{0.3\textwidth}
    \centering
    \includegraphics[width=\textwidth]{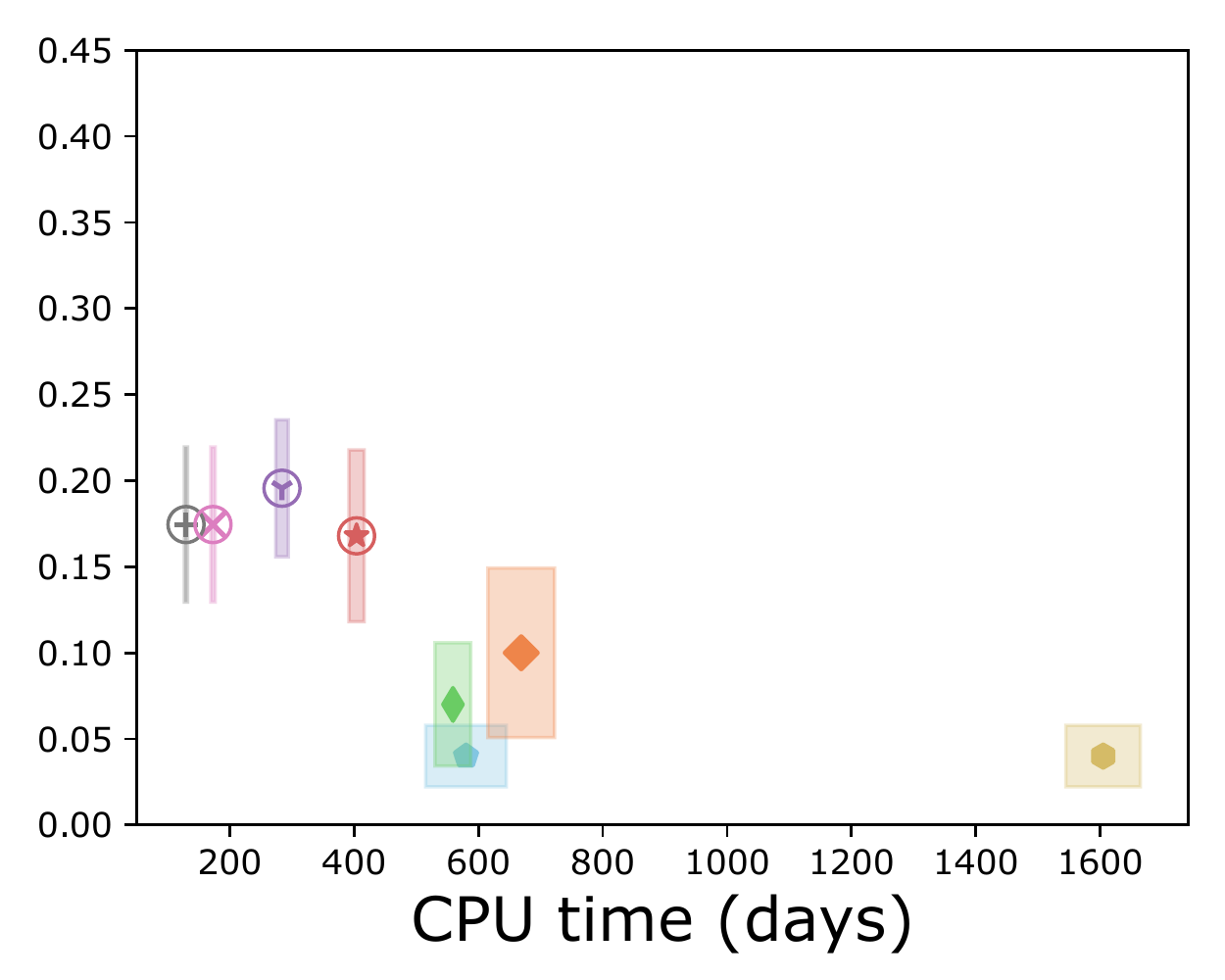}
  \end{subfigure}
  
    \begin{subfigure}[t]{0.3\textwidth}
    \centering
    \includegraphics[width=\textwidth]{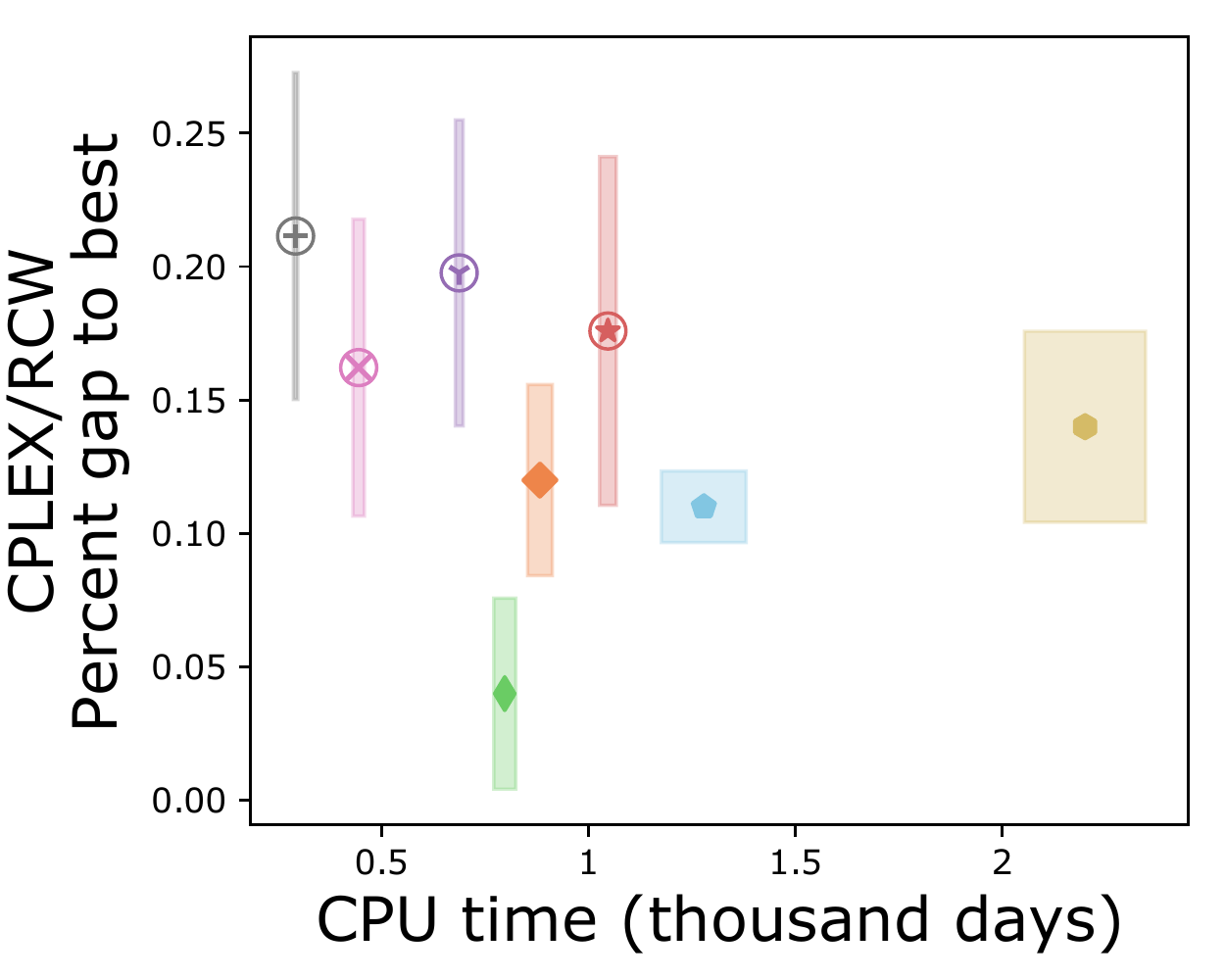}
  \end{subfigure}
  \begin{subfigure}[t]{0.3\textwidth}
    \centering
    \includegraphics[width=\textwidth]{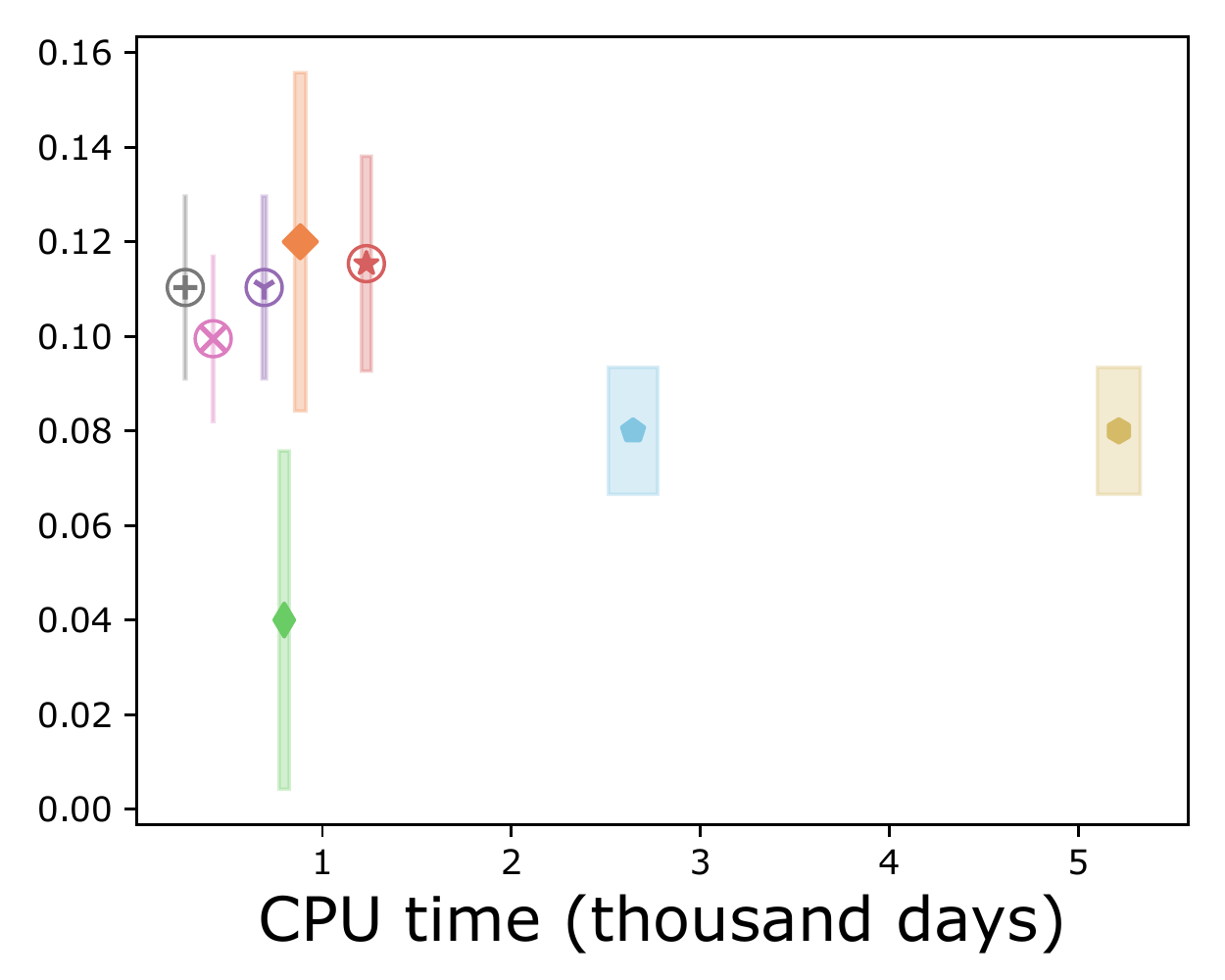}
  \end{subfigure}
  \begin{subfigure}[t]{0.3\textwidth}
    \centering
    \includegraphics[width=\textwidth]{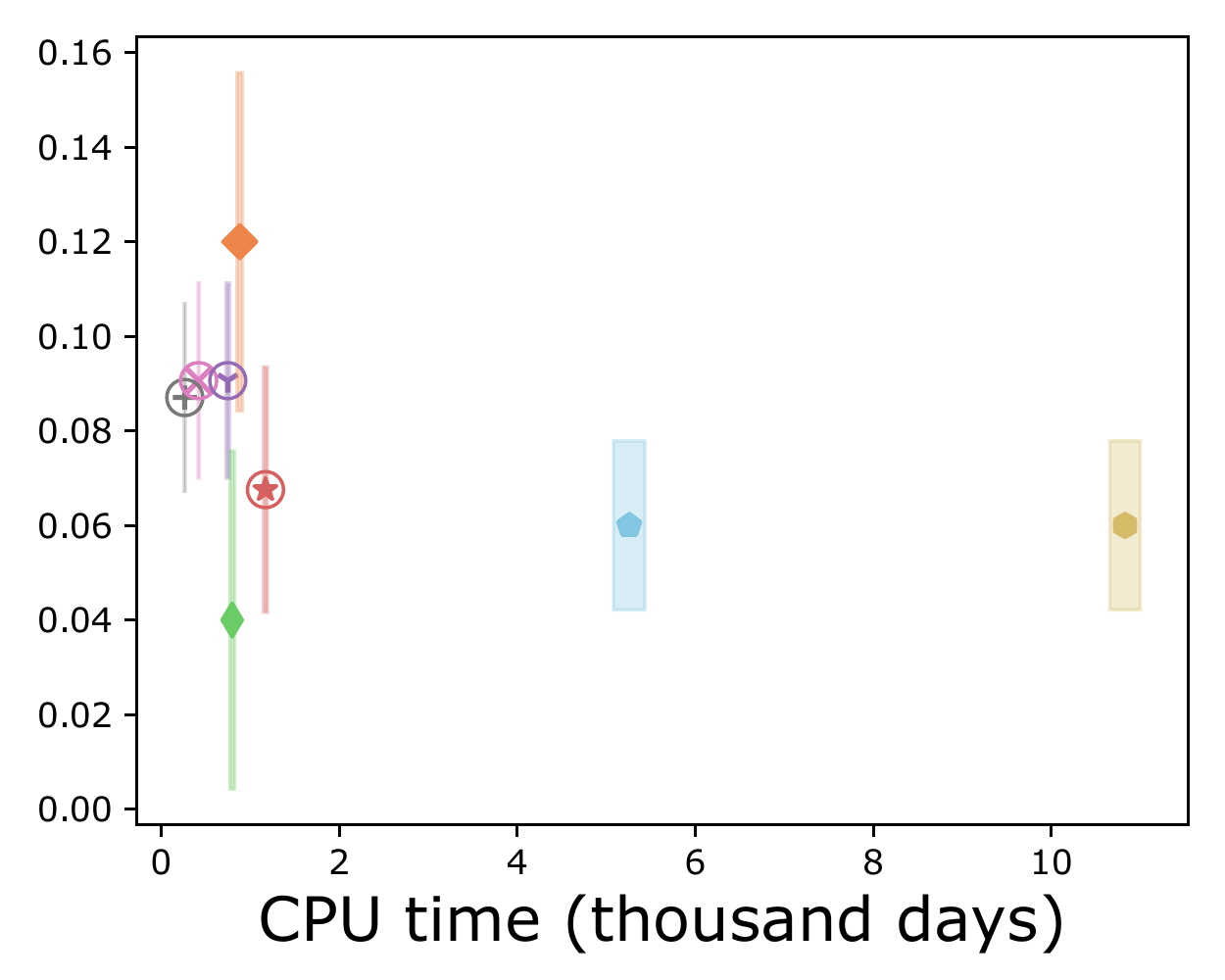}
  \end{subfigure}

  \caption{
  Mean CPU time and percent gap to best over $5$ seeds for $\delta=0.01$ and different $\alpha$ (columns) for AC-Band, ICAR, CAR++ and Hyperband on CNFuzzDD (top), Regions200 (middle) and RCW (bottom). Circles indicate variants of AC-Band. Rectangles represent the standard error over the seeds. The number of configurations tried for CAR++: $\{128, 325, 652\}$, ICAR: $\{166, 431, 884\},$ AC-Band: $\{93, 232, 462\}$, Hyperband($\eta=5$): $\{842\},$ Hyperband($\eta=8$): $\{ 618\}$.}\label{fig_d01}
\end{figure*}

Looking closer at the results obtained for the CNFuzzDD dataset for $\alpha = 0.01$, we note that the percent gap to best increases for $k = 8$. This increase is solely due to one seed for which a percent gap to best value of $5.74$ is obtained. AC-Band uses most of its sampling budget in the first rounds, where large sets of configurations are evaluated on large sets of instances. Over the course of AC-Band, both the configuration and instance sets become smaller. For the seed in question, AC-Band samples one new configuration in the last round that is able to beat the current incumbent on $18$ instances. Since the configuration wins, it is returned, even though the incumbent has seen more instances and proven worthy over the previous $8$ epochs. The possibility of this happening grows with the number of configurations to sample (decreasing $\alpha$) since the same amount of instances is distributed among more configurations.

\begin{table}[p]
\setcounter{table}{0}
\renewcommand{\thetable}{(\alph{table})}
\renewcommand{\tablename}{}
\tiny
\center
\begin{tabular}{@{\hspace{0.3\tabcolsep}}c@{\hspace{0.3\tabcolsep}}|@{\hspace{0.2\tabcolsep}}  r @{\hspace{0.6\tabcolsep}}  r @{\hspace{0.6\tabcolsep}} r @{\hspace{0.6\tabcolsep}} r @{\hspace{0.6\tabcolsep}} r @{\hspace{0.6\tabcolsep}} r@{\hspace{0.4\tabcolsep}}| @{\hspace{0.4\tabcolsep}} r @{\hspace{0.6\tabcolsep}}  r @{\hspace{0.6\tabcolsep}} r  @{\hspace{0.6\tabcolsep}} r @{\hspace{0.6\tabcolsep}} r @{\hspace{0.6\tabcolsep}} r@{\hspace{0.4\tabcolsep}}| @{\hspace{0.4\tabcolsep}} r @{\hspace{0.6\tabcolsep}} r @{\hspace{0.6\tabcolsep}} r  @{\hspace{0.6\tabcolsep}} r @{\hspace{0.6\tabcolsep}} r @{\hspace{0.6\tabcolsep}} r@{\hspace{0.4\tabcolsep}}| @{\hspace{0.4\tabcolsep}} r @{\hspace{0.6\tabcolsep}} r @{\hspace{0.6\tabcolsep}} r  @{\hspace{0.6\tabcolsep}} r @{\hspace{0.6\tabcolsep}} r @{\hspace{0.6\tabcolsep}} r @{\hspace{0.6\tabcolsep}}  r @{\hspace{0.6\tabcolsep}} r @{\hspace{0.6\tabcolsep}} r @{\hspace{0.6\tabcolsep}} r @{\hspace{0.6\tabcolsep}} r @{\hspace{0.6\tabcolsep}} r @{\hspace{0.6\tabcolsep}}}

 & \multicolumn{6}{c}{CPU Time (days)}  & \multicolumn{6}{c}{Percent gap to best}  & \multicolumn{6}{c}{Percent gap to subset-best}  & \multicolumn{6}{c}{$R^{\delta}$}\\
\hline
$\alpha$ & \multicolumn{2}{@{\hspace{0.5\tabcolsep}} c @{\hspace{0.5\tabcolsep}}}{$0.05$} & \multicolumn{2}{@{\hspace{0.5\tabcolsep}} c @{\hspace{0.5\tabcolsep}}}{$0.02$} & \multicolumn{2}{@{\hspace{0.5\tabcolsep}} c @{\hspace{0.5\tabcolsep}}}{$0.01$}& \multicolumn{2}{@{\hspace{0.5\tabcolsep}} c @{\hspace{0.5\tabcolsep}}}{$0.05$} & \multicolumn{2}{@{\hspace{0.5\tabcolsep}} c @{\hspace{0.5\tabcolsep}}}{$0.02$} & \multicolumn{2}{@{\hspace{0.5\tabcolsep}} c @{\hspace{0.5\tabcolsep}}}{$0.01$}& \multicolumn{2}{@{\hspace{0.5\tabcolsep}} c @{\hspace{0.5\tabcolsep}}}{$0.05$} & \multicolumn{2}{@{\hspace{0.5\tabcolsep}} c @{\hspace{0.5\tabcolsep}}}{$0.02$} & \multicolumn{2}{@{\hspace{0.5\tabcolsep}} c @{\hspace{0.5\tabcolsep}}}{$0.01$}& \multicolumn{2}{@{\hspace{0.5\tabcolsep}} c @{\hspace{0.5\tabcolsep}}}{$0.05$} & \multicolumn{2}{@{\hspace{0.5\tabcolsep}} c @{\hspace{0.5\tabcolsep}}}{$0.02$} & \multicolumn{2}{@{\hspace{0.5\tabcolsep}} c @{\hspace{0.5\tabcolsep}}}{$0.01$}\\
\hline
 Method & $\mu$ & $\sigma$ & $\mu$ & $\sigma$ & $\mu$ & $\sigma$ & $\mu$ & $\sigma$ & $\mu$ & $\sigma$ & $\mu$ & $\sigma$ & $\mu$ & $\sigma$ & $\mu$ & $\sigma$ & $\mu$ & $\sigma$ & $\mu$ & $\sigma$ & $\mu$ & $\sigma$ & $\mu$ & $\sigma$\\
\hline
ICAR & 130.57 & 13.14 & 325.61 & 11.80 & 619.16 & 15.58  & 0.02 & 0.04 & 0.01 & 0.01 & 0.02 & 0.03 & 0.01 & 0.02 & 0.01 & 0.01 & 0.02 & 0.03 & 5.0 & 0.2 & 4.9 & 0.1 & 4.9 & 0.1 \\
CAR++ & 128.14 & 15.59 & 327.18 & 9.35 & 654.46 & 12.68  & 0.02 & 0.04 & 0.01 & 0.01 & 0.03 & 0.02 & 0.01 & 0.02 & 0.01 & 0.01 & 0.03 & 0.02 &  5.0 & 0.2 & 4.9 & 0.1 & 4.9 & 0.1 \\
CAR & 220.33 & 29.17 & 554.95 & 19.47 & 1169.06 & 12.44  & 0.01 & 0.02 & 0.01 & 0.01 & 0.02 & 0.02 & 0.00 & 0.00 & 0.01 & 0.01 & 0.01 & 0.02 & 5.1 & 0.2 & 4.9 & 0.1 & 4.9 & 0.1 \\
\hline
$k = 2$ & 11.56 & 0.77 & 11.17 & 0.32 & 11.20 & 0.65 & 0.10 & 0.14 & 0.00 & 0.00 & 0.08 & 0.15 & 0.08 & 0.15 & 0.00 & 0.00 & 0.08 & 0.14 & 5.5 & 0.8 & 4.9 & 0.1 & 5.4 & 0.8 \\ 
$k = 4$  & 14.38 & 0.63 & 13.29 & 0.50 & 14.11 & 0.36 & 0.02 & 0.03 & 0.01 & 0.02 & 0.09 & 0.14 & 0.00 & 0.00 & 0.01 & 0.01 & 0.09 & 0.14 & 5.1 & 0.1 & 4.9 & 0.0 & 5.3 & 0.8 \\ 
$k = 8$ & 21.43 & 2.45 & 20.38 & 0.63 & 21.98 & 0.52 & 0.02 & 0.03 & 0.10 & 0.14 & 2.16 & 2.29 & 0.00 & 0.00 & 0.09 & 0.12 & 1.16 & 1.29 & 5.1 & 0.1 & 5.4   & 0.7 & 13.0 & 16.1 \\ 
$k = 16$ & 29.18 & 1.99 & 33.77 & 0.82 & 31.70 & 0.45 & 0.02 & 0.02 & 0.02 & 0.03 & 0.02 & 0.04 & 0.00 & 0.00 & 0.01 & 0.02 & 0.02 & 0.04 & 5.0 & 0.1 & 5.0 & 0.2 & 5.0 & 0.3 \\ 

\end{tabular} \caption{CNFuzzDD dataset}

\begin{tabular}{@{\hspace{0\tabcolsep}}c@{\hspace{0.3\tabcolsep}}|@{\hspace{0.2\tabcolsep}}  r @{\hspace{0.6\tabcolsep}}  r @{\hspace{0.6\tabcolsep}} r @{\hspace{0.6\tabcolsep}} r @{\hspace{0.6\tabcolsep}} r @{\hspace{0.6\tabcolsep}} r@{\hspace{0.4\tabcolsep}}| @{\hspace{0.4\tabcolsep}} r @{\hspace{0.6\tabcolsep}}  r @{\hspace{0.6\tabcolsep}} r  @{\hspace{0.6\tabcolsep}} r @{\hspace{0.6\tabcolsep}} r @{\hspace{0.6\tabcolsep}} r@{\hspace{0.4\tabcolsep}}| @{\hspace{0.4\tabcolsep}} r @{\hspace{0.6\tabcolsep}} r @{\hspace{0.6\tabcolsep}} r  @{\hspace{0.6\tabcolsep}} r @{\hspace{0.6\tabcolsep}} r @{\hspace{0.6\tabcolsep}} r@{\hspace{0.4\tabcolsep}}| @{\hspace{0.4\tabcolsep}} r @{\hspace{0.6\tabcolsep}} r @{\hspace{0.6\tabcolsep}} r  @{\hspace{0.6\tabcolsep}} r @{\hspace{0.6\tabcolsep}} r @{\hspace{0.6\tabcolsep}} r @{\hspace{0.6\tabcolsep}}  r @{\hspace{0.6\tabcolsep}} r @{\hspace{0.6\tabcolsep}} r @{\hspace{0.6\tabcolsep}} r @{\hspace{0.6\tabcolsep}} r @{\hspace{0.6\tabcolsep}} r @{\hspace{0.6\tabcolsep}}}

 & \multicolumn{6}{c}{CPU Time (days)}  & \multicolumn{6}{c}{Percent gap to best}  & \multicolumn{6}{c}{Percent gap to subset-best}  & \multicolumn{6}{c}{$R^{\delta}$}\\
\hline
$\alpha$ & \multicolumn{2}{@{\hspace{0.5\tabcolsep}} c @{\hspace{0.5\tabcolsep}}}{$0.05$} & \multicolumn{2}{@{\hspace{0.5\tabcolsep}} c @{\hspace{0.5\tabcolsep}}}{$0.02$} & \multicolumn{2}{@{\hspace{0.5\tabcolsep}} c @{\hspace{0.5\tabcolsep}}}{$0.01$}& \multicolumn{2}{@{\hspace{0.5\tabcolsep}} c @{\hspace{0.5\tabcolsep}}}{$0.05$} & \multicolumn{2}{@{\hspace{0.5\tabcolsep}} c @{\hspace{0.5\tabcolsep}}}{$0.02$} & \multicolumn{2}{@{\hspace{0.5\tabcolsep}} c @{\hspace{0.5\tabcolsep}}}{$0.01$}& \multicolumn{2}{@{\hspace{0.5\tabcolsep}} c @{\hspace{0.5\tabcolsep}}}{$0.05$} & \multicolumn{2}{@{\hspace{0.5\tabcolsep}} c @{\hspace{0.5\tabcolsep}}}{$0.02$} & \multicolumn{2}{@{\hspace{0.5\tabcolsep}} c @{\hspace{0.5\tabcolsep}}}{$0.01$}& \multicolumn{2}{@{\hspace{0.5\tabcolsep}} c @{\hspace{0.5\tabcolsep}}}{$0.05$} & \multicolumn{2}{@{\hspace{0.5\tabcolsep}} c @{\hspace{0.5\tabcolsep}}}{$0.02$} & \multicolumn{2}{@{\hspace{0.5\tabcolsep}} c @{\hspace{0.5\tabcolsep}}}{$0.01$}\\
\hline
 Method & $\mu$ & $\sigma$ & $\mu$ & $\sigma$ & $\mu$ & $\sigma$ & $\mu$ & $\sigma$ & $\mu$ & $\sigma$ & $\mu$ & $\sigma$ & $\mu$ & $\sigma$ & $\mu$ & $\sigma$ & $\mu$ & $\sigma$ & $\mu$ & $\sigma$ & $\mu$ & $\sigma$ & $\mu$ & $\sigma$\\
\hline
ICAR & 226.12 & 127.05 & 374.51 &  131.49 & 579.64 & 145.09  & 0.24 & 0.16 & 0.09 & 0.09 & 0.04 & 0.04 & 0.01 & 0.03 & 0.00 & 0.00 & 0.00 & 0.00 & 34.8 & 4.3 & 29.8 & 2.2 & 28.5 & 1.8 \\
CAR++ & 349.44 & 29.05 & 813.73 & 35.96 & 1604.52 & 133.59  & 0.24 & 0.16 & 0.11 & 0.08 & 0.04 & 0.04 & 0.00 & 0.00 & 0.00 & 0.00 & 0.00 & 0.00 &  34.8 & 4.3 & 30.6 & 2.2 & 28.5 & 1.8 \\
CAR & 798.6 & 74.21 & 1885.13 & 83.41 & 3717.16 & 264.71  & 0.24 & 0.16 & 0.11 & 0.08 & 0.04 & 0.04 & 0.00 & 0.00 & 0.00 & 0.00 & 0.00 & 0.00 & 34.8 & 4.3 & 30.6 & 1.6 & 28.5 & 1.8 \\
\hline
$k = 2$ & 147.24 & 8.72 & 143.81 & 9.80 & 129.57 & 5.53 & 0.38 & 0.17 & 0.25 & 0.22 & 0.17 & 0.10 & 0.00 & 0.00 & 0.02 & 0.03 & 0.05 & 0.07 & 38.9 & 4.9 & 34.4 & 5.2 & 31.6 & 3.5 \\ 
$k = 4$  & 186.61 & 17.15 & 182.01 & 17.80 & 172.98 & 7.71 & 0.40 & 0.15 & 0.31 & 0.23 & 0.17 & 0.10 & 0.00 & 0.00 & 0.03 & 0.03 & 0.01 & 0.02 & 39.6 & 3.5 & 36.0 & 5.6 & 31.6 & 3.5  \\ 
$k = 8$ & 283.07 &  58.47 & 267.06 &  23.24 & 283.98 & 22.41 & 0.37 & 0.20 & 0.17 & 0.11 & 0.20 & 0.09 & 0.02 & 0.04 & 0.02 & 0.02 & 0.03 & 0.02 & 38.1 & 4.9 & 31.5 & 3.3 & 32.0 &  3.0  \\ 
$k = 16$ & 368.87 &  43.25 & 433.53 &  28.73 & 403.93 & 27.27 & 0.26&  0.17 & 0.18 & 0.12 & 0.17  & 0.11 & 0.00 &  0.00 & 0.00 &  0.00 & 0.02 &  0.02 & 35.2 &  5.4 & 33.0 & 2.8 & 31.2 &  3.0  \\

\end{tabular}\caption{Regions200 dataset}

\begin{tabular}{@{\hspace{0\tabcolsep}}c@{\hspace{0.3\tabcolsep}}|@{\hspace{0.2\tabcolsep}}  r @{\hspace{0.6\tabcolsep}}  r @{\hspace{0.6\tabcolsep}} r @{\hspace{0.6\tabcolsep}} r @{\hspace{0.6\tabcolsep}} r @{\hspace{0.6\tabcolsep}} r@{\hspace{0.4\tabcolsep}}| @{\hspace{0.4\tabcolsep}} r @{\hspace{0.6\tabcolsep}}  r @{\hspace{0.6\tabcolsep}} r  @{\hspace{0.6\tabcolsep}} r @{\hspace{0.6\tabcolsep}} r @{\hspace{0.6\tabcolsep}} r@{\hspace{0.4\tabcolsep}}| @{\hspace{0.4\tabcolsep}} r @{\hspace{0.6\tabcolsep}} r @{\hspace{0.6\tabcolsep}} r  @{\hspace{0.6\tabcolsep}} r @{\hspace{0.6\tabcolsep}} r @{\hspace{0.6\tabcolsep}} r@{\hspace{0.4\tabcolsep}}| @{\hspace{0.4\tabcolsep}} r @{\hspace{0.6\tabcolsep}} r @{\hspace{0.6\tabcolsep}} r  @{\hspace{0.6\tabcolsep}} r @{\hspace{0.6\tabcolsep}} r @{\hspace{0.6\tabcolsep}} r @{\hspace{0.6\tabcolsep}}  r @{\hspace{0.6\tabcolsep}} r @{\hspace{0.6\tabcolsep}} r @{\hspace{0.6\tabcolsep}} r @{\hspace{0.6\tabcolsep}} r @{\hspace{0.6\tabcolsep}} r @{\hspace{0.6\tabcolsep}}}

 & \multicolumn{6}{c}{CPU Time (thousand days)}  & \multicolumn{6}{c}{Percent gap to best}  & \multicolumn{6}{c}{Percent gap to subset-best}  & \multicolumn{6}{c}{$R^{\delta}$}\\
\hline
$\alpha$ & \multicolumn{2}{@{\hspace{0.5\tabcolsep}} c @{\hspace{0.5\tabcolsep}}}{$0.05$} & \multicolumn{2}{@{\hspace{0.5\tabcolsep}} c @{\hspace{0.5\tabcolsep}}}{$0.02$} & \multicolumn{2}{@{\hspace{0.5\tabcolsep}} c @{\hspace{0.5\tabcolsep}}}{$0.01$}& \multicolumn{2}{@{\hspace{0.5\tabcolsep}} c @{\hspace{0.5\tabcolsep}}}{$0.05$} & \multicolumn{2}{@{\hspace{0.5\tabcolsep}} c @{\hspace{0.5\tabcolsep}}}{$0.02$} & \multicolumn{2}{@{\hspace{0.5\tabcolsep}} c @{\hspace{0.5\tabcolsep}}}{$0.01$}& \multicolumn{2}{@{\hspace{0.5\tabcolsep}} c @{\hspace{0.5\tabcolsep}}}{$0.05$} & \multicolumn{2}{@{\hspace{0.5\tabcolsep}} c @{\hspace{0.5\tabcolsep}}}{$0.02$} & \multicolumn{2}{@{\hspace{0.5\tabcolsep}} c @{\hspace{0.5\tabcolsep}}}{$0.01$}& \multicolumn{2}{@{\hspace{0.5\tabcolsep}} c @{\hspace{0.5\tabcolsep}}}{$0.05$} & \multicolumn{2}{@{\hspace{0.5\tabcolsep}} c @{\hspace{0.5\tabcolsep}}}{$0.02$} & \multicolumn{2}{@{\hspace{0.5\tabcolsep}} c @{\hspace{0.5\tabcolsep}}}{$0.01$}\\
\hline
 Method & $\mu$ & $\sigma$ & $\mu$ & $\sigma$ & $\mu$ & $\sigma$ & $\mu$ & $\sigma$ & $\mu$ & $\sigma$ & $\mu$ & $\sigma$ & $\mu$ & $\sigma$ & $\mu$ & $\sigma$ & $\mu$ & $\sigma$ & $\mu$ & $\sigma$ & $\mu$ & $\sigma$ & $\mu$ & $\sigma$\\
\hline
ICAR & 1.27 &  0.22 & 2.64 &  0.29 & 5.26 & 0.40   & 0.11 & 0.03 & 0.08 &  0.03 & 0.06 &  0.04 & 0.00 &  0.00 & 0.00 & 0.00 & 0.00 &  0.00 & 150.1 & 2.3 & 146.2 &  3.7 & 142.1 &  5.9  \\
CAR++ & 2.20 &  0.32  & 5.21 &  0.25  & 10.82 &  0.38 & 0.14 &  0.08 & 0.08 &  0.03 & 0.06 & 0.02 & 0.00 &  0.00 & 0.00 & 0.00 & 0.00 & 0.00 & 156.1 &  2.3 & 146.5 & 3.7 & 143.3 &  5.9 \\
CAR  & 4.44 &  0.48 & 10.79 & 0.26 & 22.60 &  0.61 & 0.14 &  0.08 & 0.08 &  0.03 & 0.06 & 0.02 & 0.00 &  0.00 & 0.00 & 0.00 & 0.00 &  0.00 & 156.1 & 11.9 & 146.5 & 4.1 & 143.3 &  4.9 \\

\hline
$k = 2$ & 0.29 & 0.01 & 0.27 &  0.09 & 0.26 &  0.01 & 0.21 &  0.14 & 0.11 &  0.04 & 0.09 &  0.04 & 0.07 &  0.08 & 0.04 &  0.04 & 0.03 &  0.03 & 168.7 &  25.9 & 145.1 & 3.6 & 141.9 & 4.7  \\ 
$k = 4$  & 0.45 &  0.03 & 0.42 &  0.01 & 0.42 &  0.02 & 0.16 & 0.12 & 0.10 &  0.04 & 0.09 & 0.05 & 0.05 &  0.09 & 0.04 & 0.04 & 0.04 & 0.04 & 156.9 &  19.6 & 144.9 & 3.3 & 141.8 &  4.6  \\ 
$k = 8$ & 0.69 &  0.02 & 0.70 & 0.03 & 0.74 &  0.05 & 0.20 &  0.13 & 0.11 & 0.04 & 0.09 & 0.05 & 0.05 &  0.09 & 0.04 &  0.04 & 0.04 &  0.04 & 164.2 & 21.7 & 145.1 &  3.6 & 141.8 &  4.6  \\ 
$k = 16$ & 1.05 &  0.04 & 1.23 &  0.05 & 1.17 &  0.05 & 0.18 &  0.15 & 0.12 &  0.05 & 0.07 & 0.06 & 0.06 & 0.09 & 0.04 &  0.04 & 0.02 &  0.02 & 163.7 & 27.2 & 145.6 &  3.8 & 139.4 &  5.0  \\ 

\end{tabular}\caption{RCW dataset}

\renewcommand{\tablename}{Table}
\renewcommand{\thetable}{2}
\caption{Mean ($\mu$) and standard derivation ($\sigma$) for CPU time, percent gap to best, percent gap to subset-best and mean $\delta$-capped runtime over $5$ seeds for $\delta=\textbf{0.01}$ and different values of $\alpha$ (columns) for AC-Band, ICAR, CAR++, CAR on the CNFuzzDD (top), Regions200 (middle) and RCW (bottom) datasets. The number of configurations tried for CAR(++): $\{128, 325, 652\}$, ICAR: $\{166, 431, 884\},$ AC-Band: $\{93, 232, 462\}$.}\label{table_d01_rcw}
\end{table}

\paragraph{Increasing the sampling budget of AC-Band}
AC-Band samples fewer configurations than Hyperband or (I)CAR(++) and leaves more of the configuration space unexplored. This explains, in part, why AC-Band usually has a worse percent gap to best than its competitors, but less runtime. To investigate this further, we let AC-Band sample the same number of configurations as CAR++ and Hyperband with $\eta = 8$ in Table \ref{table_car_budget} and Table \ref{table_hb}. In particular, we set $N$ to be equal to the number of samples of either method and set $n_{0} = N +1$. Note that we do not sample exactly the same amount of configurations due to the rounding operations within AC-Band.

On the CNFuzzDD and RCW dataset (see Table \ref{table_car_budget}), a larger sampling budget for AC-Band leads to a gap to best that is nearly as good as those obtained by (I)CAR(++), while needing significantly less CPU time. The same can be seen for the Regions200dataset, however with some exceptions where the gap to best increases. Specifically, for $\alpha = 0.05$ where only $101$ configurations out of $2000$ are examined. For a few seeds, no good configuration is contained in the $101$ samples, leading to these results.

Table \ref{table_hb} report the results obtained for Hyperband with different values of $\eta$ as well as the AC-Band results with a sampling budget of $618$ configurations. For all three datasets, AC-Band comes closer to the results of Hyperband in terms of gap to best, while needing significantly less CPU time. This is especially true for the Regions200 and RCW datasets. On the CNFuzzDD dataset, AC-Band's performance is weaker, which may be due to this dataset containing the smallest number of instances. %

\begin{table}[p]
\tiny
\center
\setcounter{table}{0}
\renewcommand{\thetable}{(\alph{table})}
\renewcommand{\tablename}{}
\begin{tabular}{@{\hspace{0.3\tabcolsep}}c@{\hspace{0.3\tabcolsep}}|@{\hspace{0.2\tabcolsep}}  r @{\hspace{0.6\tabcolsep}}  r @{\hspace{0.6\tabcolsep}} r @{\hspace{0.6\tabcolsep}} r @{\hspace{0.6\tabcolsep}} r @{\hspace{0.6\tabcolsep}} r@{\hspace{0.4\tabcolsep}}| @{\hspace{0.4\tabcolsep}} r @{\hspace{0.6\tabcolsep}}  r @{\hspace{0.6\tabcolsep}} r  @{\hspace{0.6\tabcolsep}} r @{\hspace{0.6\tabcolsep}} r @{\hspace{0.6\tabcolsep}} r@{\hspace{0.4\tabcolsep}}| @{\hspace{0.4\tabcolsep}} r @{\hspace{0.6\tabcolsep}} r @{\hspace{0.6\tabcolsep}} r  @{\hspace{0.6\tabcolsep}} r @{\hspace{0.6\tabcolsep}} r @{\hspace{0.6\tabcolsep}} r@{\hspace{0.4\tabcolsep}}| @{\hspace{0.4\tabcolsep}} r @{\hspace{0.6\tabcolsep}} r @{\hspace{0.6\tabcolsep}} r  @{\hspace{0.6\tabcolsep}} r @{\hspace{0.6\tabcolsep}} r @{\hspace{0.6\tabcolsep}} r @{\hspace{0.6\tabcolsep}}  r @{\hspace{0.6\tabcolsep}} r @{\hspace{0.6\tabcolsep}} r @{\hspace{0.6\tabcolsep}} r @{\hspace{0.6\tabcolsep}} r @{\hspace{0.6\tabcolsep}} r @{\hspace{0.6\tabcolsep}}}

 & \multicolumn{6}{c}{CPU Time (days)}  & \multicolumn{6}{c}{Percent gap to best}  & \multicolumn{6}{c}{Percent gap to subset-best}  & \multicolumn{6}{c}{$R^{\delta}$}\\
\hline
$\alpha$ & \multicolumn{2}{@{\hspace{0.5\tabcolsep}} c @{\hspace{0.5\tabcolsep}}}{$0.05$} & \multicolumn{2}{@{\hspace{0.5\tabcolsep}} c @{\hspace{0.5\tabcolsep}}}{$0.02$} & \multicolumn{2}{@{\hspace{0.5\tabcolsep}} c @{\hspace{0.5\tabcolsep}}}{$0.01$}& \multicolumn{2}{@{\hspace{0.5\tabcolsep}} c @{\hspace{0.5\tabcolsep}}}{$0.05$} & \multicolumn{2}{@{\hspace{0.5\tabcolsep}} c @{\hspace{0.5\tabcolsep}}}{$0.02$} & \multicolumn{2}{@{\hspace{0.5\tabcolsep}} c @{\hspace{0.5\tabcolsep}}}{$0.01$}& \multicolumn{2}{@{\hspace{0.5\tabcolsep}} c @{\hspace{0.5\tabcolsep}}}{$0.05$} & \multicolumn{2}{@{\hspace{0.5\tabcolsep}} c @{\hspace{0.5\tabcolsep}}}{$0.02$} & \multicolumn{2}{@{\hspace{0.5\tabcolsep}} c @{\hspace{0.5\tabcolsep}}}{$0.01$}& \multicolumn{2}{@{\hspace{0.5\tabcolsep}} c @{\hspace{0.5\tabcolsep}}}{$0.05$} & \multicolumn{2}{@{\hspace{0.5\tabcolsep}} c @{\hspace{0.5\tabcolsep}}}{$0.02$} & \multicolumn{2}{@{\hspace{0.5\tabcolsep}} c @{\hspace{0.5\tabcolsep}}}{$0.01$}\\
\hline
 Method & $\mu$ & $\sigma$ & $\mu$ & $\sigma$ & $\mu$ & $\sigma$ & $\mu$ & $\sigma$ & $\mu$ & $\sigma$ & $\mu$ & $\sigma$ & $\mu$ & $\sigma$ & $\mu$ & $\sigma$ & $\mu$ & $\sigma$ & $\mu$ & $\sigma$ & $\mu$ & $\sigma$ & $\mu$ & $\sigma$\\
\hline
ICAR & 100.64 & 12.74 & 242.74 & 14.96 & 467.32 & 25.08  & 0.02 & 0.04 &0.01 & 0.01 &0.02 & 0.03 &0.01 & 0.02 &0.01 & 0.01 &0.02 & 0.03 & 5.0 & 0.1 & 4.9 & 0.1 & 4.9 & 0.1 \\
CAR++ & 92.30 & 5.48 & 224.21 & 16.38 & 452.06 & 18.08  &0.05 & 0.04 &0.01 & 0.01 &0.01 & 0.01 &0.02 & 0.03 &0.00 & 0.01 & 0.01 & 0.01 &  5.2 & 0.2 & 4.9 & 0.1 & 4.9 & 0.1 \\
CAR & 157.76 & 18.07 & 367.86 & 7.24 & 771.45 & 21.72  & 0.04 & 0.03 & 0.01 & 0.01 & 0.01 & 0.01 & 0.01 & 0.02 & 0.00 & 0.01 & 0.01 & 0.01 & 5.2 & 0.2 & 4.9 & 0.1 & 4.9 & 0.1 \\
\hline
$k = 2$ & 11.66 & 0.87 & 11.50 & 0.25 & 10.83 & 0.25 & 0.03 & 0.02 & 0.02 & 0.02 & 0.09 & 0.13 & 0.00 & 0.00 & 0.01 & 0.02 & 0.09 & 0.14 & 5.1 & 0.1 & 5.0 & 0.2 & 5.3 & 0.7 \\ 
$k = 4$  & 14.18 & 0.49 & 14.53 & 0.42 & 13.51 & 0.35 & 0.03 & 0.03 & 0.10 & 0.16 & 0.03 & 0.04 & 0.00 & 0.00 & 0.09 & 0.14 & 0.03 & 0.05 & 5.1 & 0.1 & 5.4 & 0.8 & 5.0 & 0.2 \\ 
$k = 8$ & 22.39 & 0.82 & 21.45 & 0.59 & 20.17 & 0.58 & 0.02 & 0.03 & 0.02 & 0.01 & 0.00 & 0.01 & 0.00 & 0.00 & 0.01 & 0.01 & 0.00 & 0.01 & 5.1 & 0.2 & 5.0 & 0.1 & 4.9 & 0.1 \\ 
$k = 16$ & 29.56 & 1.61 & 30.00 & 1.18 & 32.28 & 0.43 & 0.04 & 0.03 & 0.01 & 0.01 & 0.02 & 0.03 & 0.00 & 0.00 & 0.00 & 0.00 & 0.02 & 0.03 & 5.2 & 0.2 & 5.0 & 0.1 & 5.0 & 0.1 \\ 

\end{tabular}\caption{CNFuzzDD dataset}

\begin{tabular}{@{\hspace{0.3\tabcolsep}}c@{\hspace{0.3\tabcolsep}}|@{\hspace{0.2\tabcolsep}}  r @{\hspace{0.6\tabcolsep}}  r @{\hspace{0.6\tabcolsep}} r @{\hspace{0.6\tabcolsep}} r @{\hspace{0.6\tabcolsep}} r @{\hspace{0.6\tabcolsep}} r@{\hspace{0.4\tabcolsep}}| @{\hspace{0.4\tabcolsep}} r @{\hspace{0.6\tabcolsep}}  r @{\hspace{0.6\tabcolsep}} r  @{\hspace{0.6\tabcolsep}} r @{\hspace{0.6\tabcolsep}} r @{\hspace{0.6\tabcolsep}} r@{\hspace{0.4\tabcolsep}}| @{\hspace{0.4\tabcolsep}} r @{\hspace{0.6\tabcolsep}} r @{\hspace{0.6\tabcolsep}} r  @{\hspace{0.6\tabcolsep}} r @{\hspace{0.6\tabcolsep}} r @{\hspace{0.6\tabcolsep}} r@{\hspace{0.4\tabcolsep}}| @{\hspace{0.4\tabcolsep}} r @{\hspace{0.6\tabcolsep}} r @{\hspace{0.6\tabcolsep}} r  @{\hspace{0.6\tabcolsep}} r @{\hspace{0.6\tabcolsep}} r @{\hspace{0.6\tabcolsep}} r @{\hspace{0.6\tabcolsep}}  r @{\hspace{0.6\tabcolsep}} r @{\hspace{0.6\tabcolsep}} r @{\hspace{0.6\tabcolsep}} r @{\hspace{0.6\tabcolsep}} r @{\hspace{0.6\tabcolsep}} r @{\hspace{0.6\tabcolsep}}}

 & \multicolumn{6}{c}{CPU Time (days)}  & \multicolumn{6}{c}{Percent gap to best}  & \multicolumn{6}{c}{Percent gap to subset-best}  & \multicolumn{6}{c}{$R^{\delta}$}\\
\hline
$\alpha$ & \multicolumn{2}{@{\hspace{0.5\tabcolsep}} c @{\hspace{0.5\tabcolsep}}}{$0.05$} & \multicolumn{2}{@{\hspace{0.5\tabcolsep}} c @{\hspace{0.5\tabcolsep}}}{$0.02$} & \multicolumn{2}{@{\hspace{0.5\tabcolsep}} c @{\hspace{0.5\tabcolsep}}}{$0.01$}& \multicolumn{2}{@{\hspace{0.5\tabcolsep}} c @{\hspace{0.5\tabcolsep}}}{$0.05$} & \multicolumn{2}{@{\hspace{0.5\tabcolsep}} c @{\hspace{0.5\tabcolsep}}}{$0.02$} & \multicolumn{2}{@{\hspace{0.5\tabcolsep}} c @{\hspace{0.5\tabcolsep}}}{$0.01$}& \multicolumn{2}{@{\hspace{0.5\tabcolsep}} c @{\hspace{0.5\tabcolsep}}}{$0.05$} & \multicolumn{2}{@{\hspace{0.5\tabcolsep}} c @{\hspace{0.5\tabcolsep}}}{$0.02$} & \multicolumn{2}{@{\hspace{0.5\tabcolsep}} c @{\hspace{0.5\tabcolsep}}}{$0.01$}& \multicolumn{2}{@{\hspace{0.5\tabcolsep}} c @{\hspace{0.5\tabcolsep}}}{$0.05$} & \multicolumn{2}{@{\hspace{0.5\tabcolsep}} c @{\hspace{0.5\tabcolsep}}}{$0.02$} & \multicolumn{2}{@{\hspace{0.5\tabcolsep}} c @{\hspace{0.5\tabcolsep}}}{$0.01$}\\
\hline
 Method & $\mu$ & $\sigma$ & $\mu$ & $\sigma$ & $\mu$ & $\sigma$ & $\mu$ & $\sigma$ & $\mu$ & $\sigma$ & $\mu$ & $\sigma$ & $\mu$ & $\sigma$ & $\mu$ & $\sigma$ & $\mu$ & $\sigma$ & $\mu$ & $\sigma$ & $\mu$ & $\sigma$ & $\mu$ & $\sigma$\\
\hline
ICAR & 164.30 & 91.05 & 274.84 & 100.72 & 420.15 & 103.24 & 0.24 & 0.16 & 0.09 & 0.09 & 0.04 & 0.04 & 0.00 & 0.00 & 0.00 & 0.00 & 0.00 & 0.00 & 34.8 & 4.3 & 29.8 & 2.2 & 28.5 & 1.8 \\
CAR++ & 229.29 & 19.93 & 566.99 & 28.21 & 1097.90 & 88.41 & 0.27 & 0.17 & 0.15 & 0.07 & 0.09 & 0.09 & 0.00 & 0.00 & 0.00 & 0.00 & 0.00 & 0.00 & 35.3 & 4.3 & 32.0 & 2.2 & 29.8 & 1.8 \\
CAR & 523.69 & 53.34 & 1294.87 & 64.11 & 2549.22 & 199.00 & 0.27 & 0.17 & 0.16 & 0.09 & 0.09 & 0.09 & 0.00 & 0.00 & 0.10 & 0.30 & 0.00 & 0.00 & 35.3 & 4.5 & 31.9 & 1.6 & 29.8 & 2.2 \\
\hline
$k = 2$ & 148.29 & 10.60 & 136.75 & 12.34 & 125.29 & 3.91 & 0.35 & 0.20 & 0.15 & 0.13 & 0.04 & 0.04 & 0.01 & 0.03 & 0.01 & 0.02 & 0.00 & 0.01 & 37.8 & 5.6 & 31.8 & 3.0 & 28.9 & 1.8  \\ 
$k = 4$  & 184.87 & 19.23 & 187.03 & 15.01 & 166.71 & 6.53 & 0.58 & 0.51 & 0.16 & 0.18 & 0.09 & 0.10 & 0.14 & 0.28 & 0.00 & 0.00 & 0.02 & 0.05 & 42.4 & 9.7 & 32.3 & 4.5 & 30.0 & 1.5  \\ 
$k = 8$ & 310.18 & 25.52 & 263.16 & 19.92 & 241.98 & 13.16 & 0.39 & 0.27 & 0.12 & 0.09 & 0.14 & 0.14 & 0.04 & 0.08 & 0.01 & 0.01 & 0.05 & 0.10 & 38.6 & 6.5 & 30.7 & 1.6 & 31.1 & 2.4  \\ 
$k = 16$ & 380.15 & 43.31 & 358.15 & 40.36 & 384.76 & 22.49 & 0.36 & 0.19 & 0.18 & 0.18 & 0.09 & 0.10 & 0.01 & 0.03 & 0.02 & 0.03 & 0.01 & 0.02 & 38.1 & 5.2 & 32.5 & 4.4 & 30.0 & 1.5  \\ 

\end{tabular}\caption{Regions200 dataset}
\begin{tabular}{@{\hspace{0.3\tabcolsep}}c@{\hspace{0.3\tabcolsep}}|@{\hspace{0.2\tabcolsep}}  r @{\hspace{0.6\tabcolsep}}  r @{\hspace{0.6\tabcolsep}} r @{\hspace{0.6\tabcolsep}} r @{\hspace{0.6\tabcolsep}} r @{\hspace{0.6\tabcolsep}} r@{\hspace{0.4\tabcolsep}}| @{\hspace{0.4\tabcolsep}} r @{\hspace{0.6\tabcolsep}}  r @{\hspace{0.6\tabcolsep}} r  @{\hspace{0.6\tabcolsep}} r @{\hspace{0.6\tabcolsep}} r @{\hspace{0.6\tabcolsep}} r@{\hspace{0.4\tabcolsep}}| @{\hspace{0.4\tabcolsep}} r @{\hspace{0.6\tabcolsep}} r @{\hspace{0.6\tabcolsep}} r  @{\hspace{0.6\tabcolsep}} r @{\hspace{0.6\tabcolsep}} r @{\hspace{0.6\tabcolsep}} r@{\hspace{0.4\tabcolsep}}| @{\hspace{0.4\tabcolsep}} r @{\hspace{0.6\tabcolsep}} r @{\hspace{0.6\tabcolsep}} r  @{\hspace{0.6\tabcolsep}} r @{\hspace{0.6\tabcolsep}} r @{\hspace{0.6\tabcolsep}} r @{\hspace{0.6\tabcolsep}}  r @{\hspace{0.6\tabcolsep}} r @{\hspace{0.6\tabcolsep}} r @{\hspace{0.6\tabcolsep}} r @{\hspace{0.6\tabcolsep}} r @{\hspace{0.6\tabcolsep}} r @{\hspace{0.6\tabcolsep}}}

 & \multicolumn{6}{c}{CPU Time (days)}  & \multicolumn{6}{c}{Percent gap to best}  & \multicolumn{6}{c}{Percent gap to subset-best}  & \multicolumn{6}{c}{$R^{\delta}$}\\
\hline
$\alpha$ & \multicolumn{2}{@{\hspace{0.5\tabcolsep}} c @{\hspace{0.5\tabcolsep}}}{$0.05$} & \multicolumn{2}{@{\hspace{0.5\tabcolsep}} c @{\hspace{0.5\tabcolsep}}}{$0.02$} & \multicolumn{2}{@{\hspace{0.5\tabcolsep}} c @{\hspace{0.5\tabcolsep}}}{$0.01$}& \multicolumn{2}{@{\hspace{0.5\tabcolsep}} c @{\hspace{0.5\tabcolsep}}}{$0.05$} & \multicolumn{2}{@{\hspace{0.5\tabcolsep}} c @{\hspace{0.5\tabcolsep}}}{$0.02$} & \multicolumn{2}{@{\hspace{0.5\tabcolsep}} c @{\hspace{0.5\tabcolsep}}}{$0.01$}& \multicolumn{2}{@{\hspace{0.5\tabcolsep}} c @{\hspace{0.5\tabcolsep}}}{$0.05$} & \multicolumn{2}{@{\hspace{0.5\tabcolsep}} c @{\hspace{0.5\tabcolsep}}}{$0.02$} & \multicolumn{2}{@{\hspace{0.5\tabcolsep}} c @{\hspace{0.5\tabcolsep}}}{$0.01$}& \multicolumn{2}{@{\hspace{0.5\tabcolsep}} c @{\hspace{0.5\tabcolsep}}}{$0.05$} & \multicolumn{2}{@{\hspace{0.5\tabcolsep}} c @{\hspace{0.5\tabcolsep}}}{$0.02$} & \multicolumn{2}{@{\hspace{0.5\tabcolsep}} c @{\hspace{0.5\tabcolsep}}}{$0.01$}\\
\hline
 Method & $\mu$ & $\sigma$ & $\mu$ & $\sigma$ & $\mu$ & $\sigma$ & $\mu$ & $\sigma$ & $\mu$ & $\sigma$ & $\mu$ & $\sigma$ & $\mu$ & $\sigma$ & $\mu$ & $\sigma$ & $\mu$ & $\sigma$ & $\mu$ & $\sigma$ & $\mu$ & $\sigma$ & $\mu$ & $\sigma$\\
\hline
ICAR & 1.28 & 0.40 & 2.03 & 0.30 & 4.07 & 0.24  & 0.14 & 0.08 & 0.08 & 0.03 & 0.06 & 0.02 & 0.00 & 0.01 & 0.00 & 0.00 & 0.00 & 0.00 & 156.1 & 11.9 & 146.5 & 4.1 & 143.3 & 4.9  \\
CAR++ & 1.72 & 0.37 & 3.64 & 0.18 & 7.52 & 0.13. & 0.17 & 0.09 & 0.10 & 0.03 & 0.06 & 0.02 & 0.00 & 0.01 & 0.00 & 0.00 & 0.00 & 0.00 & 162.1 & 11.9 & 149.1 & 4.1 & 143.3 & 4.9 \\
CAR  & 3.30 & 0.50 & 7.59. & 0.19 & 15.65 & 0.26 & 0.17 & 0.09 & 0.10 & 0.03 & 0.06 & 0.02 & 0.00 & 0.01 & 0.00 & 0.00 & 0.00 & 0.00 & 160.1 & 13.3 & 149.1 & 4.7 & 143.3 & 4.9 \\
\hline
$k = 2$ & 0.30 & 0.02 & 0.27 & 0.01 & 0.27 & 0.01 & 0.14 & 0.09 & 0.11 & 0.04 & 0.12 & 0.06 & 0.01 & 0.03 & 0.04 & 0.04 & 0.07 & 0.05 & 154.5 & 16.1 & 145.1 & 3.6 & 146.4 & 4.6  \\ 
$k = 4$  & 0.43 & 0.02 & 0.04 & 0.01 & 0.04 & 0.01 & 0.18 & 0.11 & 0.11 & 0.04 & 0.10 & 0.10 & 0.04 & 0.03 & 0.03 & 0.04 & 0.07 & 0.06 & 161.3 & 24.1 & 145.1 & 3.6 & 145.9 & 12.3  \\ 
$k = 8$ & 0.81 & 0.06 & 0.73 & 0.03 & 0.69 & 0.02 & 0.18 & 0.11 & 0.11 & 0.05 & 0.13 & 0.06 & 0.04 & 0.05 & 0.03 & 0.04 & 0.09 & 0.05 & 159.8 & 18.9 & 146.6 & 6.7 & 148.5 & 6.3  \\ 
$k = 16$ & 1.09 & 0.10 & 1.06 & 0.04 & 1.17 & 0.04 & 0.23 & 0.20 & 0.09 & 0.05 & 0.15 & 0.06 & 0.08 & 0.11 & 0.03 & 0.04 & 0.10 & 0.06 & 169.8 & 39.2 & 141.8 & 4.6 & 151.8 & 6.3  \\ 

\end{tabular}\caption{RCW dataset}
\renewcommand{\tablename}{Table}
\renewcommand{\thetable}{3}
\caption{Mean ($\mu$) and standard derivation ($\sigma$) for CPU time, percent gap to best, percent gap to subset-best and mean $\delta$-capped runtime over $5$ seeds for $\delta=\textbf{0.05}$ and different values of $\alpha$ (columns) for AC-Band, ICAR, CAR++, CAR on the CNFuzzDD (top), Regions200 (middle) and RCW (bottom) datasets. For AC-Band $N$ was set to the number of configurations sampled by CAR++ for the respective $\alpha$ value. The number of configurations tried for CAR(++): $\{97, 245, 492\}$, ICAR: $\{134, 351, 724\},$ AC-Band: $\{\textbf{101, 247, 492}\}$. Note that AC-Band samples slightly more configurations due to rounding operations.}\label{table_car_budget}
\end{table}

\begin{table}[p]
\tiny
\center
\setcounter{table}{0}
\renewcommand{\thetable}{(\alph{table})}
\renewcommand{\tablename}{}
\begin{tabular}{ c c |r r |r r |r r |r r }\\

 & & \multicolumn{2}{@{\hspace{0.7\tabcolsep}} c @{\hspace{0.7\tabcolsep}}}{CPU Time}   & \multicolumn{2}{@{\hspace{0.7\tabcolsep}} c @{\hspace{0.7\tabcolsep}}}{Percent gap}   & \multicolumn{2}{@{\hspace{0.7\tabcolsep}} c @{\hspace{0.7\tabcolsep}}}{Percent gap}  & \multicolumn{2}{@{\hspace{0.7\tabcolsep}} c @{\hspace{0.7\tabcolsep}}}{$R^{\delta}$} \\ 
 & & \multicolumn{2}{c}{(days)} & \multicolumn{2}{c}{to best} & \multicolumn{2}{c}{to subset-best} & \multicolumn{2}{c}{} \\
\hline
 & Method & \hspace*{0.5mm} $\mu$ \hspace*{0.5mm} & \hspace*{0.5mm} $\sigma$ \hspace*{0.5mm} & \hspace*{0.5mm} $\mu$ \hspace*{0.5mm} &\hspace*{0.5mm}  $\sigma$ \hspace*{0.5mm} &\hspace*{0.5mm}  $\mu$\hspace*{0.5mm} & \hspace*{0.5mm}$\sigma$ \hspace*{0.5mm} &\hspace*{0.5mm} $\mu$\hspace*{0.5mm} &\hspace*{0.5mm} $\sigma$ \hspace*{0.5mm}\\
\hline
\parbox[t]{2mm}{\multirow{4}{*}{\rotatebox[origin=c]{90}{Hyperband}}} & $\eta = 4$ & 80.05 & 18.36 & 0.14 & 0.17 & 0.13 & 0.16 & 5.7 & 0.0 \\ 
 & $\eta = 5$ & 83.37 & 11.50 & 0.27 & 0.26 & 0.27 & 0.26 & 6.3 & 0.1 \\ 
& $\eta = 6$ & 77.64 & 23.71 & 0.13 & 0.14 & 0.11 & 0.13 & 5.8 & 0.8 \\ 
& $\eta = 8$ & 65.43 & 14.21 & 0.16 & 0.14 & 0.16 & 0.14 & 5.7 & 0.8 \\ 
\hline
\parbox[t]{2mm}{\multirow{4}{*}{\rotatebox[origin=c]{90}{AC-Band}}}& $k = 2$ & 10.55 & 0.32 & 0.39 & 0.53 & 0.39 & 0.53 & 7.1 & 3.0 \\ 
& $k = 4$  & 13.66 & 0.26 & 0.06 & 0.08 & 0.06 & 0.08 & 5.2 & 0.5 \\ 
& $k = 8$ & 19.49 & 0.26 & 0.21 & 0.36 & 0.21 & 0.36 & 6.1 & 2.1 \\ 
& $k = 16$ & 33.35 & 0.57 & 0.01 & 0.02 & 0.01 & 0.02 & 5.0 & 0.1 \\ 
\end{tabular}\caption{CNFuzzDD dataset}
\begin{tabular}{ c c |r r |r r |r r |r r }\\

 & & \multicolumn{2}{@{\hspace{0.7\tabcolsep}} c @{\hspace{0.7\tabcolsep}}}{CPU Time}   & \multicolumn{2}{@{\hspace{0.7\tabcolsep}} c @{\hspace{0.7\tabcolsep}}}{Percent gap}   & \multicolumn{2}{@{\hspace{0.7\tabcolsep}} c @{\hspace{0.7\tabcolsep}}}{Percent gap}  & \multicolumn{2}{@{\hspace{0.7\tabcolsep}} c @{\hspace{0.7\tabcolsep}}}{$R^{\delta}$} \\ 
 & & \multicolumn{2}{c}{(days)} & \multicolumn{2}{c}{to best} & \multicolumn{2}{c}{to subset-best} & \multicolumn{2}{c}{} \\
\hline
 & Method & \hspace*{0.5mm} $\mu$ \hspace*{0.5mm} & \hspace*{0.5mm} $\sigma$ \hspace*{0.5mm} & \hspace*{0.5mm} $\mu$ \hspace*{0.5mm} &\hspace*{0.5mm}  $\sigma$ \hspace*{0.5mm} &\hspace*{0.5mm}  $\mu$\hspace*{0.5mm} & \hspace*{0.5mm}$\sigma$ \hspace*{0.5mm} &\hspace*{0.5mm} $\mu$\hspace*{0.5mm} &\hspace*{0.5mm} $\sigma$ \hspace*{0.5mm}\\
\hline
\parbox[t]{2mm}{\multirow{4}{*}{\rotatebox[origin=c]{90}{Hyperband}}} &  $\eta = 4$ & 798.58 & 64.32 & 0.06 & 0.06 & 0.06 & 0.05 & 29.0 & 2.3 \\ 
& $\eta = 5$ & 668.78 & 120.01 & 0.10 & 0.11 & 0.04 & 0.05 & 30.7 & 3.1 \\ 
& $\eta = 6$ & 662.96 & 94.83 & 0.06 & 0.07 & 0.06 & 0.06 & 28.6 & 2.4 \\ 
& $\eta = 8$ & 559.19 & 64.26 & 0.07 & 0.08 & 0.00 & 0.00 & 29.9 & 2.5 \\
\hline
\parbox[t]{2mm}{\multirow{4}{*}{\rotatebox[origin=c]{90}{AC-Band}}} & $k = 2$ & 117.68 & 6.73 & 0.09 & 0.09 & 0.05 & 0.07 & 29.6 & 1.92  \\ 
& $k = 4$  & 165.07 & 11.85 & 0.09 & 0.09 & 0.05 & 0.07 & 29.6 & 1.92   \\ 
& $k = 8$ & 230.52 & 25.96 & 0.09 & 0.09  & 0.02 & 0.02 & 29.6 & 1.92  \\ 
& $k = 16$ & 385.45 & 32.5  & 0.09 & 0.09  & 0.04 & 0.02 & 29.2 & 2.03  \\ 
\end{tabular}\caption{Regions200 dataset}
\begin{tabular}{ c c |r r |r r |r r |r r }\\

 & & \multicolumn{2}{@{\hspace{0.7\tabcolsep}} c @{\hspace{0.7\tabcolsep}}}{CPU Time}   & \multicolumn{2}{@{\hspace{0.7\tabcolsep}} c @{\hspace{0.7\tabcolsep}}}{Percent gap}   & \multicolumn{2}{@{\hspace{0.7\tabcolsep}} c @{\hspace{0.7\tabcolsep}}}{Percent gap}  & \multicolumn{2}{@{\hspace{0.7\tabcolsep}} c @{\hspace{0.7\tabcolsep}}}{$R^{\delta}$} \\ 
 & & \multicolumn{2}{c}{(days)} & \multicolumn{2}{c}{to best} & \multicolumn{2}{c}{to subset-best} & \multicolumn{2}{c}{} \\
\hline
 & Method & \hspace*{0.5mm} $\mu$ \hspace*{0.5mm} & \hspace*{0.5mm} $\sigma$ \hspace*{0.5mm} & \hspace*{0.5mm} $\mu$ \hspace*{0.5mm} &\hspace*{0.5mm}  $\sigma$ \hspace*{0.5mm} &\hspace*{0.5mm}  $\mu$\hspace*{0.5mm} & \hspace*{0.5mm}$\sigma$ \hspace*{0.5mm} &\hspace*{0.5mm} $\mu$\hspace*{0.5mm} &\hspace*{0.5mm} $\sigma$ \hspace*{0.5mm}\\
\hline
\parbox[t]{2mm}{\multirow{4}{*}{\rotatebox[origin=c]{90}{Hyperband}}} & $\eta = 4$ & 991.58 & 55.48 & 0.07 & 0.08 & 0.07 & 0.07 & 141.4 & 9.1 \\ 
& $\eta = 5$ & 883.04 & 67.03 & 0.12 & 0.08 & 0.07 & 0.07 & 147.6 & 12.2 \\ 
& $\eta = 6$ & 873.19 & 56.11 & 0.10 & 0.11 & 0.10 & 0.10 & 144.3 & 22.1 \\ 
& $\eta = 8$ & 796.60 & 59.49 & 0.04 & 0.08 & 0.00 & 0.00 & 138.0 & 9.3 \\ 
\hline
\parbox[t]{2mm}{\multirow{4}{*}{\rotatebox[origin=c]{90}{AC-Band}}} & $k = 2$ & 250.98 & 6.87 & 0.14 & 0.10 & 0.11 & 0.05 & 151.1 & 14.0  \\ 
& $k = 4$ & 411.85 & 11.33 & 0.12 & 0.12 & 0.09 & 0.07 & 148.5 & 15.1 \\ 
& $k = 8$ & 657.61 & 25.65 & 0.11 & 0.12 & 0.07 & 0.07 & 146.6 & 17.4  \\ 
& $k = 16$ & 1196.41 & 38.75 & 0.08 & 0.07 & 0.07 & 0.07 & 139.3 & 16.0  \\ 
\end{tabular}\caption{RCW dataset}
\renewcommand{\tablename}{Table}
\renewcommand{\thetable}{4}
\caption{Mean ($\mu$) and standard derivation ($\sigma$) for CPU time, percent gap to best, percent gap to subset-best and mean $\delta$-capped runtime over $5$ seeds for different $\eta$ for Hyperband and AC-Band on the CNFuzzDD (top), Regions200 (middle) and RCW (bottom) datasets. The number of configurations tried for Hyperband: $\{378, 842, 280, 618\}$, AC-Band: $\{618\}$. Note that we use a value of $\eta$ for which no more than the available number of configurations in a dataset would be sampled.}\label{table_hb}
\end{table}

\end{document}